%% file: main.tex
\input{setup/preamble.tex}

\newif\ifdraft
\draftfalse %

\title{Putting Transformer Theories to the Test} 

\author[1]{Awni Altabaa\thanks{Part of this work was done while interning at Microsoft Research.}}
\author[2]{Tobias Schnabel}
\author[3]{Kiran Tomlinson}
\affil[1]{Yale University\\\texttt{awni.altabaa@yale.edu}}
\affil[2]{Microsoft Research\\\texttt{tobias.schnabel@microsoft.com}}
\affil[3]{Microsoft Research\\\texttt{kitomlinson@microsoft.com}}

\ifdraft
  \date{\today \ \textcolor{red!75!black}{\small\bfseries \texttt{[DRAFT]}}}
\else
  \date{} %
\fi

\begin{document}

\maketitle

\input{sections/00-abstract.tex}

\input{sections/01-introduction.tex}

\input{sections/02-theories.tex}

\input{sections/06-related-work.tex}

\input{sections/03-methods.tex}

\input{sections/04-results.tex}

\input{sections/05-discussion.tex}

\FloatBarrier
\clearpage %

\input{sections/acknowledgments-ai-use.tex}
\PrintPaperBibliography

\clearpage
\appendix
\input{appendices/task-details/task-details.tex}
\FloatBarrier
\input{appendices/experimental-details/experimental-details.tex}

\FloatBarrier
\input{appendices/additional-results/additional-results.tex}
\clearpage
\input{appendices/theory-background/theory-background.tex}

\clearpage
\input{tables/theory-task-matrix/theory-task-prediction-notes.tex}

\FloatBarrier
\input{appendices/theory-task-results/theory-task-results.tex}

\FloatBarrier
\input{appendices/c-rasp-separation/c-rasp-separation.tex}

\FloatBarrier
\end{document}

%% file: setup/preamble.tex
\documentclass{paperstyle}

\usepackage[nopatch=footnote]{microtype} %

\input{setup/colors.tex} %
\input{setup/biblio.tex} %
\input{setup/packages.tex} %
\input{setup/theorems.tex} %

\input{setup/math.tex} %
\input{setup/comments.tex} %
\input{setup/theory-task-matrix.tex} %

%% file: setup/theorems.tex
\renewcommand{\qedsymbol}{\ensuremath{\blacksquare}}

\newtheorem{theorem}{Theorem}
\newtheorem*{theorem*}{Theorem}
\newtheorem{proposition}{Proposition}

\theoremstyle{definition}
\newtheorem{definition}{Definition}

\theoremstyle{plain}

\crefname{claim}{Claim}{Claims}
\Crefname{claim}{Claim}{Claims}

%% file: setup/math.tex
\input{setup/general_math_macros.tex}

\input{setup/paper_macros.tex}

%% file: setup/general_math_macros.tex
\newcommand{\bbR}{\mathbb{R}}

\newcommand{\reals}{\bbR}

\newcommand{\bmone}{\bm{1}}

%% file: setup/paper_macros.tex
\newcommand{\AC}{\textsf{AC}}
\newcommand{\TC}{\textsf{TC}}
\newcommand{\NC}{\textsf{NC}}
\newcommand{\Lclass}{\textsf{L}}
\newcommand{\NL}{\textsf{NL}}

\newcommand{\ACz}{\ensuremath{\AC^0}\xspace}
\newcommand{\TCz}{\ensuremath{\TC^0}\xspace}
\newcommand{\NCone}{\ensuremath{\NC^1}\xspace}

\newcommand{\FO}{\textsf{FO}}
\newcommand{\FOC}{\textsf{FOC}}
\newcommand{\FOM}{\textsf{FOM}}

\newcommand{\RASP}{\textsf{RASP}\xspace}
\newcommand{\CRASP}{\textsf{C-RASP}\xspace}
\newcommand{\BRASP}{\textsf{B-RASP}\xspace}

\newcommand{\MPC}{\textsf{MPC}\xspace}
\newcommand{\BAPO}{\textsf{BAPO}\xspace}
\newcommand{\Starfree}{\textsf{Star-free}\xspace}
\newcommand{\Sensitivity}{\textsf{Sensitivity}\xspace}

\newcommand{\WQ}{W_Q}
\newcommand{\WK}{W_K}

\newcommand{\Km}{\mathcal{K}_m}

\definecolor{TaskEasyColor}{HTML}{176B63}
\definecolor{TaskHardColor}{HTML}{963D2D}

\newcommand{\Easy}{\textcolor{TaskEasyColor}{\texttt{Easy}}\xspace}

\newcommand{\Hard}{\textcolor{TaskHardColor}{\texttt{Hard}}\xspace}

\newcommand{\taskname}[1]{\textup{\textsc{#1}}}

\newcommand{\TaskALL}{\taskname{All}\xspace}
\newcommand{\TaskANY}{\taskname{Any}\xspace}
\newcommand{\TaskMajority}{\taskname{Majority}\xspace}
\newcommand{\TaskParity}{\taskname{Parity}\xspace}

\newcommand{\TaskIndex}{\taskname{Index}\xspace}
\newcommand{\TaskIndexStar}{\taskname{Index-*}\xspace}
\newcommand{\TaskIndexAndOr}{\taskname{Index-All/Any}\xspace}
\newcommand{\TaskIndexAnd}{\taskname{Index-All}\xspace}
\newcommand{\TaskIndexOr}{\taskname{Index-Any}\xspace}
\newcommand{\TaskIndexXor}{\taskname{Index-Parity}\xspace}

\newcommand{\TaskSparseStar}{\taskname{Sparse-*}\xspace}
\newcommand{\TaskSparseAndOr}{\taskname{Sparse-All/Any}\xspace}
\newcommand{\TaskSparseAnd}{\taskname{Sparse-All}\xspace}
\newcommand{\TaskSparseOr}{\taskname{Sparse-Any}\xspace}
\newcommand{\TaskSparseXor}{\taskname{Sparse-Parity}\xspace}

\newcommand{\TaskBinaryAddition}{\taskname{BinaryAdd}\xspace}

\newcommand{\TaskDyck}{\taskname{Dyck}\xspace}
\newcommand{\TaskDyckOne}{\taskname{Dyck-}\ensuremath{1}\xspace}
\newcommand{\TaskDyckTwo}{\taskname{Dyck-}\ensuremath{2}\xspace}
\newcommand{\TaskKHop}{\taskname{\ensuremath{k}-Hop Induction}\xspace}
\newcommand{\TaskOneHop}{\taskname{\ensuremath{1}-Hop Induction}\xspace}

\newcommand{\TaskEdgeExistence}{\taskname{Edge Existence}\xspace}
\newcommand{\TaskDirectedReachability}{\taskname{Directed Reachability}\xspace}
\newcommand{\TaskUndirectedReachability}{\taskname{Undirected Reachability}\xspace}

\definecolor{relblue}{HTML}{26749C}
\definecolor{relred}{HTML}{BE512E}
\definecolor{relink}{HTML}{263744}
\ifx\RelationsColorArrows\undefined\def\RelationsColorArrows{1}\fi
\ifx\RelationsSlash\undefined\def\RelationsSlash{1}\fi
\newcommand{\RelationsSetColors}{%
  \ifnum\RelationsColorArrows=1
    \colorlet{relyes}{relblue}\colorlet{relno}{relred}%
  \else
    \colorlet{relyes}{relink}\colorlet{relno}{relink}%
  \fi}
\newcommand{\relnoarrowglyph}{\ifnum\RelationsSlash=1 \not\rightarrow\else\rightarrow\fi}
\DeclareRobustCommand{\ContainmentArrow}{\ensuremath{\mathrel{\RelationsSetColors\textcolor{relyes}{\rightarrow}}}}
\DeclareRobustCommand{\NonContainmentArrow}{\ensuremath{\mathrel{\RelationsSetColors\textcolor{relno}{\relnoarrowglyph}}}}
\DeclareRobustCommand{\ContainmentSymbol}{\ensuremath{\mathrel{\RelationsSetColors\textcolor{relyes}{\subseteq}}}}
\DeclareRobustCommand{\NonContainmentSymbol}{\ensuremath{\mathrel{\RelationsSetColors\textcolor{relno}{\nsubseteq}}}}
\DeclareRobustCommand{\LowerBoundArrow}{\ensuremath{\RelationsSetColors\textcolor{relyes}{\downarrow}}}
\DeclareRobustCommand{\UpperBoundArrow}{\ensuremath{\RelationsSetColors\textcolor{relyes}{\uparrow}}}

%% file: setup/comments.tex
\usepackage{fxannotations}

\DeclareAnnotationAuthor{aa}{Awni}{NavyBlue}
\DeclareAnnotationAuthor{kt}{Kiran}{RedViolet}
\DeclareAnnotationAuthor{ts}{Tobias}{RoyalPurple}

\markdownSetup{hybrid=true, fencedCode=true}

%% file: sections/00-abstract.tex
\begin{abstract}
    Several formal theories have been developed to characterize the capabilities and limitations of the transformer architecture, each employing different abstractions that capture different features of the architecture.
    Due to the differences in their underlying assumptions and conceptual frameworks, these theories can disagree in their predictions about what tasks transformers find easy or hard. 
    We map out seven theories of transformer capability (\Starfree, \ACz, \TCz, \CRASP, \MPC, \BAPO, and \Sensitivity) and translate their results into a common matrix of predictions across a suite of tasks covering different computational structures.
    We provide new theoretical results to complete this matrix.
    We then evaluate theoretical predictions through controlled experiments measuring in-distribution sample complexity, length generalization, and pretrained language model performance.
    We find that no single theory explains the observed pattern of successes and failures.
    When a theory does not capture the transformer behaviors relevant to a task, its predictions break down; for instance, we show that the $k$-hop induction task is predicted hard by \CRASP because it does not model content-dependent attention, but we train small transformers on this task that length-generalize. Similarly, theories assuming hard attention like \ACz and \BAPO predict binary majority to be hard, for which we also find success in practice.
    Together, our results point to several specific features that a theory of transformers must capture: soft rather than hard attention, dynamic content-dependent retrieval, bounded computation at each position, an account of learnability and inductive biases, and a combination of upper and lower bounds.
\end{abstract}

%% file: sections/01-introduction.tex
\section{Introduction}
\label{sec:intro}

Since its introduction by \citet{vaswani2017attention}, the transformer architecture has seen enormous practical success, becoming the workhorse of modern artificial intelligence.
In response to this rapid advance in empirical performance, a variety of different theoretical research has emerged to try to explain the capabilities and limitations of the transformer architecture~\citep[among many others]{strobl2024what,hahn2020theoretical,merrill2022saturated,weiss2021thinking,schnabel2025lost,sanford2024parallel}. With this improved understanding, research can be guided towards the most promising modified architectures and learning algorithms~\citep[e.g.,][]{merrill2026olmo}.
A theory is most valuable when it can explain and make predictions about real-world phenomena~\citep{popper1959logic,kuhn1962structure}, for instance, if a theory of transformers accurately predicts which tasks they will solve in practice.
However, given the complexity of real-world neural networks, theoretical analyses of transformers have used different abstractions and assumptions leading to different hardness predictions for different tasks. As a result, it is not always clear which theories or abstractions we should keep in mind when reasoning about the behavior of transformers.

In this work, we seek to shrink the gap between transformer theory and practice by comparing predictions from seven different theoretical frameworks against the empirical performance of small task-specific transformers as well as frontier language models.
For this comparison, we identify a collection of common tasks that are predicted easy or hard for transformers across theories, including Boolean problems like majority and parity, graph problems like reachability, and syntactic problems like Dyck string recognition.
Where existing analyses are lacking, we provide new theoretical results and constructions.

At a high level, our results illustrate that no theory fully captures the behavior of either small task-specific transformers or of large pretrained language models.
Across theories, predictions that a problem is hard tend to be much more accurate than predictions that a problem is easy.
This highlights a weakness of current theoretical practice: ``easy'' predictions tend to be a result of handcrafted constructions, but there is no guarantee that training can identify these constructions.
Moreover, certain theoretical frameworks only provide loose upper bounds, so their easy predictions only say that a task is not ruled out.
Conversely, a correct hard prediction does not imply a correct \emph{explanation} for the failure.
For example, transformers fail to fit even their training data on \TaskParity, which points to an optimization barrier rather than the expressivity limits behind several theories' hardness predictions.
Additionally, positional encodings strongly affect whether a transformer length-generalizes, and our experiments show that different encodings are needed for different tasks; current theory does not account for these inductive biases.
This motivates future theoretical focus on learnability and inductive biases, an area that has seen only limited progress relative to the work on expressivity~\citep{hahn2024sensitive,huang2024formal,golowich2025role}.
Our experiments also show that pretrained LLMs behave differently from small transformers trained from scratch, which suggests that a theory of LLMs must also account for their data and training, not only the architecture.

At a more granular level, each theory captures some features of transformers and omits others, and our experiments show that its predictions tend to fail when performance depends on a feature it omits.
For example, theories that assume hard attention~\citep{hao2022formal,yang2024masked,barcelo2024logical} cannot express global aggregation, so they predict that tasks such as \TaskMajority are hard, yet transformers learn these tasks and even length-generalize.
In contrast, the popular \CRASP abstraction~\citep{yang2024counting,yang2025kneedeep} captures global aggregation by modeling attention as uniform averaging, but cannot express content-dependent retrieval.
We prove that \CRASP requires exponential size to solve a retrieval task that softmax attention transformers solve efficiently, and we observe the same gap in trained transformers.
In the other direction, theories that rely on loose upper bounds from circuit complexity, or that model only communication and allow unbounded local computation, predict that transformers can solve many tasks that they fail to learn in practice. (Put another way, a theoretical analysis providing only upper bounds is an incomplete account of what tasks we should expect transformers to solve.)
Together, these results guide the search for more practical theories of transformers, built on compact abstractions whose primitives capture the essential features that explain what transformers can learn.

%% file: sections/02-theories.tex
\begin{figure}[t]
    \centering
    \input{figures/diagrams/theory-venn-diagram/horizontal/theory-venn-diagram-tikz.tex}
    \caption{
        \textbf{Relationships between theories.}
        Each region represents the tasks that a theory predicts are easy.
        Each dot is a task, drawn inside the regions of theories that predict it is easy and outside the regions of theories that predict it is hard.
        A filled dot means all of these placements are proved, while a hollow dot means at least one depends on a conjecture (e.g., $\TCz \neq \NCone$).
        \Sensitivity is not plotted here because it is a theory of learnability rather than expressivity.
    }
    \label{fig:theory-venn-diagram}
\end{figure}
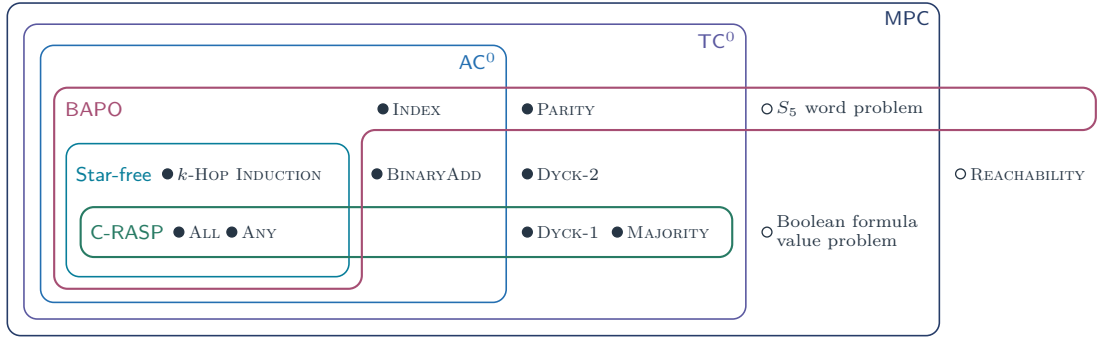

\section{Theories of Transformer Capabilities \& Limitations}\label{sec:theories}

In this work, we compare and evaluate seven theoretical frameworks that provide different accounts and predictions of the capabilities and limitations of transformers.
For the purposes of this work, a \emph{theory} is a coherent explanatory framework that makes predictions about a broad class of behavior.
Applied to transformers, a theory should predict whether a transformer can solve any given task.
We identify seven theories from existing literature that each provide a general predictive account of transformer capabilities: \Starfree, \ACz, \TCz, \CRASP, \MPC, \BAPO, and \Sensitivity.

\Cref{fig:theory-venn-diagram} highlights how the theories are related and some of the tasks that separate them.
Three of the theories form a chain of strict inclusions, $\Starfree \subsetneq \ACz \subsetneq \TCz$~\citep{mcnaughton1971counter,immerman1999descriptive,furst1984parity}.
\MPC contains this chain, because constant-round \MPC protocols can simulate $\NCone \supseteq \TCz$~\citep{roughgarden2018shuffles,frei2019circuit}, and the inclusion $\TCz \subseteq \MPC$ is strict if $\TCz \neq \NCone$.
\CRASP lies strictly inside \TCz~\citep[Prop.~31 and Remark~32]{huang2024formal}, but it is incomparable with \Starfree and \ACz.
\BAPO cuts across the whole hierarchy.
It contains every regular language, and hence all of \Starfree, but it excludes some tasks in \CRASP and in \ACz (\citealp{schnabel2025lost}; \Cref{res:addition-bapo}).

The seven theories differ both conceptually and formally, and model the key features of the transformer in different ways.
\Cref{tab:theory-features} highlights some of these differences, starting with the two most important: the direction of information flow and the type of attention.
\Starfree, \CRASP, and \BAPO assume causally masked information flow, while \ACz, \TCz, \MPC, and \Sensitivity assume bidirectional information flow.
Attention is modeled as hard in \Starfree, \ACz, and \BAPO, preventing input-wide aggregation (e.g., counting), while \TCz and \Sensitivity capture soft attention, allowing for global aggregation.
\CRASP considers a restricted form of soft attention that only allows uniform averaging, which precludes dynamic, query-dependent attention patterns.
Additionally, theories differ in how they model local computation, with \MPC and \BAPO allowing unbounded local computation (representing the role of the MLP in transformers), relying instead on communication bandwidth to limit expressive power.
Each theoretical abstraction is compared to a different class of transformers, and has a different formal relation (see~\Cref{fig:theory-transformer-relations}).
\Sensitivity is the only learnability theory, while the rest are expressivity theories.

\begin{table}[t]
    \centering
    \caption{
        \textbf{Key features \& assumptions of different theoretical frameworks.}
        Each row summarizes how a theory models, or whether it captures, a particular feature of transformers.
    }
    \label{tab:theory-features}
    \input{figures/diagrams/theory-features/theory-features-table.tex}
\end{table}

These differences lead the theories to make conflicting predictions about whether transformers are capable of performing different tasks.
In the following, we describe the core intuition and formal model behind each theory, and explain how each theory classifies tasks as \emph{easy} or \emph{hard}.

\paragraph{Out-of-scope work.} Notably, our definition of a \emph{theory} excludes task-specific theoretical analysis~\citep[e.g.,][]{sanford2023representational,chen2024theoretical}.
We also restrict our attention to a single forward pass of the transformer.
This excludes theories of chain-of-thought and scratchpad reasoning,\footnote{With intermediate tokens, a transformer can simulate any computation given enough steps~\citep[e.g.,][]{perez2019turing,merrill2024expressive}.} in which the model generates intermediate tokens.
We discuss these lines of work in \cref{sec:related_work}.

\paragraph{\texorpdfstring{\Starfree}{Star-free} languages: regular languages without modular counting.}

A common simplifying assumption in transformer theory is \emph{unique hard attention}.
Under this assumption, each attention head attends to exactly one position in the context~\citep{hahn2020theoretical,hao2022formal}.
This assumption makes the computation discrete, making it amenable to analysis with tools from logic and automata theory.
\citet{yang2024masked} study the expressive power of \emph{masked} (or \emph{causal}) hard-attention transformers, where each position can only attend to tokens in the past (or only future). %
They show that masked hard-attention transformers recognize exactly the \emph{star-free languages}, %
a subclass of the regular languages.
Informally, they can express the order of events in a string but not periodic counting.
The star-free languages can equivalently be characterized in terms of linear temporal logic~\citep{kamp1968tense}, first-order logic with order and counter-free automata~\citep{mcnaughton1971counter}, and aperiodic monoids~\citep{schutzenberger1965finite}.
This theory predicts that a task is \Easy for transformers if its language is \emph{star-free}, and \Hard otherwise.
For example, bounded-depth Dyck recognition is \Easy, while \TaskParity is \Hard.
Of the seven theories, \Starfree is the most restrictive: star-free languages form a strict subclass of $\ACz$, which is described next, and the theory classifies the fewest tasks in our suite as \Easy.

\paragraph{\texorpdfstring{\boldmath\ACz}{AC0}: constant-depth Boolean circuits.}

Another line of work studies unique hard attention without masking, in bidirectional transformers~\citep{hahn2020theoretical,hao2022formal}.
This class of transformers can be related to the $\ACz$ circuit complexity class.
$\ACz$ is the class of problems that can be solved by a family of Boolean circuits of constant depth and size polynomial in the input, using \texttt{AND}/\texttt{OR}/\texttt{NOT} gates with unbounded fan-in.\footnote{An important technical consideration for circuit families is a \emph{uniformity} requirement relating circuits across input sizes. The most recent transformer results for $\TCz$ use DLOGTIME-uniformity~\citep{chiang2025transformers}, while the main $\ACz$ result is nonuniform~\citep{hao2022formal}, but uniform $\ACz$ bounds are possible depending on the choice of positional encoding~\citep{yang2024masked}. Unless otherwise specified, we assume DLOGTIME-uniformity.}
An intuitive correspondence between transformers and Boolean circuits is that they are both models of parallel computation with bounded depth. 
\citet{hao2022formal} show that any bidirectional hard-attention transformer can be simulated by a (nonuniform) $\ACz$ circuit family.
This theory predicts that a task is \Easy for transformers if it is in $\ACz$, and \Hard otherwise.
The classic examples of hard tasks for $\ACz$ are \TaskParity and \TaskMajority~\citep{furst1984parity,ajtai1983sigma}.

\paragraph{\texorpdfstring{\boldmath\TCz}{TC0}: constant-depth threshold circuits.}

The previous two theories both assumed hard attention in transformers, which may overly restrict expressivity. For example, bidirectional hard-attention transformers are at most as expressive as nonuniform $\ACz$, precluding tasks like \TaskMajority.
By going from hard attention to soft attention, we can break this $\ACz$ ceiling.
\citet{merrill2022saturated} show that bidirectional transformers with saturated attention\footnote{\emph{Saturated attention}, also known as \emph{average-hard attention}, computes a uniform average of the joint maximizers of the attention scores, in contrast to unique hard attention, which selects a unique position to attend to, breaking ties by, e.g., choosing the rightmost or leftmost positions.} can compute \TaskMajority.
$\TCz$ is a circuit complexity class that extends $\ACz$ by adding threshold gates that output $1$ if a weighted sum of their inputs exceeds a threshold, while maintaining the constant-depth restriction.
\citet{merrill2022saturated} show that bidirectional saturated-attention transformers can be simulated by $\TCz$ circuits~\citep[see also][]{strobl2023average}, and later work extended this result to general softmax attention~\citep{merrill2023parallelism,merrill2023logic,chiang2025transformers}.
Intuitively, a soft attention layer computes a weighted average over all positions, and threshold circuits can compute the required sums and divisions in constant depth.
As with $\ACz$, this is an upper bound: $\TCz$ circuits can simulate transformers, but the converse fails for fixed-precision transformers, which cannot recognize some languages in $\TCz$~\citep{chiang2023tighter}.
This theory predicts that a task is \Easy for transformers if it is in $\TCz$, and \Hard otherwise.
For example, \TaskMajority and \TaskParity become \Easy under $\TCz$.

\paragraph{\texorpdfstring{\CRASP}{C-RASP}: programs with prefix counting.}

The circuit theories described above provide a ceiling on the computational power of transformers.
In contrast, \CRASP~\citep{yang2024counting,yang2025kneedeep} is a programming language that can be compiled directly into \emph{causal} transformers with softmax attention, providing a \emph{lower bound} on the computational power of transformers.
A \CRASP program has two value types: Boolean and integer counts.
Each position $i$ maintains and updates an intermediate representation of Booleans and counts, $\{B_\ell(i)\}_{\ell} \cup \{C_\ell(i)\}_{\ell}$.
The central operation, and the only mechanism for moving information across positions, is prefix-counting of Boolean values
\begin{equation}
    C(i) \coloneqq \#[j \leq i: B(j)],
\end{equation}
which returns the number of positions up to $i$ whose Boolean value $B$ is true.
These prefix counts are compiled to an attention head with \emph{uniform attention} by setting $\WQ = \WK = 0$ such that the head returns the fraction of positions satisfying $B$, which is enough to compare counts.
A \CRASP program solving the task provides an explicit construction for a causal transformer that solves the task, establishing a lower bound on the computational power of transformers.

\CRASP models information flow in transformers purely as prefix-counting, with the corresponding attention operation fixed to a uniform average over the prefix. 
As a result, \CRASP does not capture \emph{dynamic, content-dependent attention}, where the attention pattern depends on interactions between queries and keys.
We examine the cost of this restriction both theoretically and empirically in \Cref{sec:results,app:c-rasp-separation}.

\paragraph{\texorpdfstring{\MPC}{MPC}: massively parallel computation.}

The theories above compare transformers to circuits, logics, or programs.
\citet{sanford2024parallel,sanford2024graph} take a different view, and model transformers as distributed systems in which positions exchange messages through attention.
The Massively Parallel Computation (\MPC) model~\citep{karloff2010model} was introduced to study large-scale parallel systems such as \texttt{MapReduce}.
In \MPC, the input is split across many machines, each with memory much smaller than the input (typically, sublinear).
In each \emph{round}, every machine performs arbitrary local computation on its local data and then exchanges a bounded amount of data with other machines.
The complexity of a problem is measured in terms of the \emph{number of rounds} required.
\citet{sanford2024parallel} show that an \MPC protocol with $R$ rounds can be implemented by a bidirectional transformer of depth $O(R)$, and that an $L$-layer bidirectional transformer can be simulated by an \MPC protocol with $O(L)$ rounds.
Here, communication rounds correspond to attention operations in transformers, where each position retrieves the messages addressed to it.
This theory predicts that a task is \Easy for transformers if it can be solved by an \MPC protocol with sublinear local memory and a constant number of rounds, and \Hard otherwise.
\MPC is the most expressive of the seven theories we consider, as constant-round \MPC protocols are at least as powerful as $\NCone$~\citep{roughgarden2018shuffles,frei2019circuit}, which contains $\TCz$.

\paragraph{\texorpdfstring{\BAPO}{BAPO}: past-to-future communication complexity.}

\MPC can be interpreted as a theory of communication complexity in all-to-all parallel communication.
In a causal transformer, information can only flow from earlier positions to later ones.
\citet{schnabel2025lost} propose bounded attention prefix oracles (\BAPO) as a model that measures how much information a task requires to flow from prefix to suffix.
A task is said to be $(a,b)$-\BAPO if the input $(x_1, \ldots, x_n)$ can be split at any position $t < n$ and the output can be computed from the suffix $(x_{t+1}, \ldots, x_n)$ using only the following information about the prefix $(x_1, \ldots, x_t)$: (1) an $a$-bit summary, and (2) at most $b$ tokens selected by an attention function which sees the suffix.
A task is classified as \BAPO-easy if constant bandwidth $(a,b)$ suffices for every input length, and \BAPO-hard otherwise.
For example, \TaskIndex is a \BAPO-easy task because the suffix can directly retrieve the queried token.
Every regular language, including \TaskParity, is also \BAPO-easy because the prefix can summarize the state of a finite automaton.
In contrast, \TaskMajority is \BAPO-hard since it requires a count of $\Theta(\log n)$ bits to cross the split.
Unlike the previously mentioned theories, \BAPO has no formal simulation result relating it to transformers.
Instead, it is an abstract model of information flow that is postulated to capture bandwidth limitations of LLMs and is included to represent a different communication-based account of LLMs.

\paragraph{\texorpdfstring{\Sensitivity}{Sensitivity}: learnability via sensitivity and loss sharpness.}

Each theory mentioned so far studies the computational \emph{expressivity} of transformers: whether there exists some transformer that solves a given task.
A separate critical question is \emph{learnability}: whether training on finite data will find such a solution.
Learnability is comparatively less well understood than expressivity.
\citet{hahn2024sensitive} propose an account of one aspect of learnability via a simple property of the target function: how sensitive the output is to small changes in the input.
For a Boolean function $f$, the average sensitivity is defined as the expected number of input bits that would change the output if flipped,
\begin{equation}
    \mathrm{as}_n(f) \coloneqq \mathbb{E}_{x \sim \{\pm 1\}^n} \left[ \left| \{i: f(x) \neq f(x^{\oplus i})\}\right|\right],
\end{equation}
where $x^{\oplus i}$ is $x$ with the $i$-th bit flipped, and $x$ is drawn uniformly.
\citet{hahn2024sensitive} show that sensitivity of the target function in input space entails sensitivity in parameter space, which in turn causes any minima in the loss landscape to be isolated and sharp.
This theory predicts that a task is \Hard for transformers if its average sensitivity grows linearly with input size $n$.
For example, \TaskParity has average sensitivity $\Theta(n)$ and is therefore \Hard; in contrast, \TaskMajority has average sensitivity $\Theta(\sqrt{n})$, and $k$-sparse functions have average sensitivity at most $k$.
We label these tasks as \Easy under the sensitivity theory, although this is just a shorthand to say ``not \Hard.'' 
This theory stretches our definition of a ``general predictive account of transformers,'' since it is only applicable to Boolean functions with uniform distributions; we include it to represent an important learnability theory.

\def\RelationsShowSources{1} %
\def\RelationsKeySmall{1} %
\def\RelationsFooterMerge{1} %
\def\RelationsColorArrows{1} %
\def\RelationsSlash{0} %
\def\RelationsKeyInline{1} %
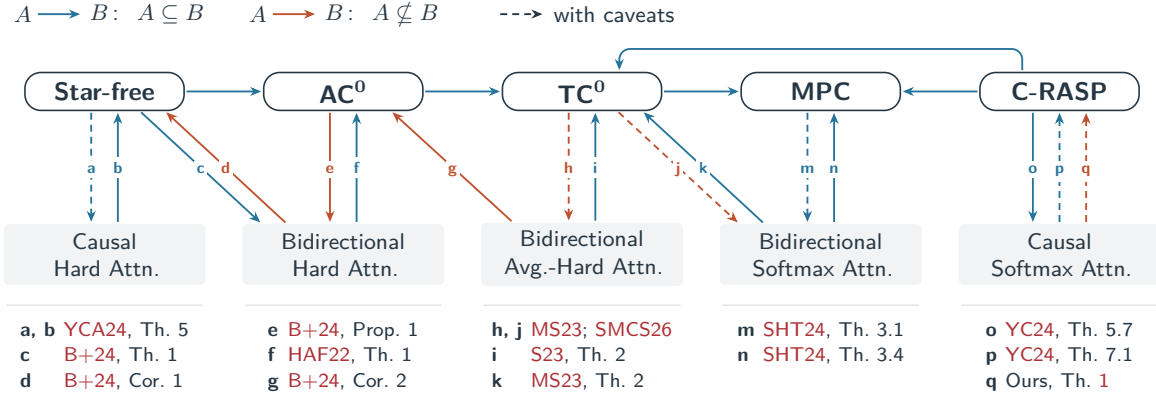
\begin{figure}[t]
    \centering
    \input{figures/diagrams/theory-transformer-relations/theory-transformer-relation-tikz.tex}
    \caption{
        \textbf{Formal expressivity relations between theories and transformer classes.}
        The top row shows the theories, and the bottom row shows the classes of transformers they are compared to.
        Each edge denotes a containment (blue) or non-containment (red) between classes.
        An blue arrow pointing down ($\LowerBoundArrow$) shows that the transformer class can express every task in the theory's class (a lower bound).
        An blue arrow pointing up ($\UpperBoundArrow$) shows that it cannot express any task outside that class (an upper bound).
        A red arrow indicates that the result in that direction does not hold.
        Some results have caveats or incompatible assumptions, so dashed arrows do not necessarily compose.
        See \Cref{app:theory-background} for discussion and details of each result.
        \BAPO is omitted because it has no known formal relation to transformers, and \Sensitivity because it is a learnability theory.
    }
    \label{fig:theory-transformer-relations}
\end{figure}

\paragraph{Overview of known expressivity results.}
\Cref{fig:theory-transformer-relations} shows a summary of known results relating the theories to different classes of transformers.
Each edge corresponds to a specific result in the literature (\Cref{tab:theory-transformer-edge-guide}).
Each theory is related to a different class of transformers, which differ in the type of attention (hard, average-hard, or softmax) and in the direction of information flow (causal or bidirectional).
The circuit complexity theories, \ACz and \TCz, only establish upper bounds on the expressive power of transformers.
\Starfree, \CRASP, and \MPC have results in both directions, but these come with significant caveats, such as exponential size or unbounded local computation.
We emphasize that there are several technical assumptions that underlie different analyses in the literature, and varying these assumptions can reverse the results they arrive at.
For example, analyses relating soft attention transformers to \TCz typically assume precision logarithmic in the input length, as assuming finite-precision would constrain softmax transformers to \ACz~\citep[Theorem 4.1]{london2025pause}.
Analogously, it is important to impose some constraint on the \emph{width} of the transformer to ensure that constructive results are informative, as the width budget controls the expressive power, and unbounded width can trivialize the result~\citep{sanford2023representational,sanford2024parallel,peng2024limitations}.\footnote{For example, \citet[footnote~2]{sanford2024parallel} observe that allowing width to grow with the input size can allow the transformer to simply collect the entire input at a single position and leave the rest of the computation to the MLP, which negates the advantage of depth.} 
These width and precision considerations are what separate the simulation of \citet[Thm.~7.1]{yang2024counting} (edge p) from our \cref{thm:c-rasp-retrieval-separation} (edge q).
See \cref{tab:theory-transformer-edge-guide} and \cref{app:c-rasp-separation} for more detailed discussion.

%% file: figures/diagrams/theory-venn-diagram/horizontal/theory-venn-diagram-tikz.tex
\begingroup%
\def\VennTheoryColour{2}%
\def\VennFills{0}
\input{figure-generation/diagrams/theory-venn-diagram/horizontal/figure-body.tex}%
\endgroup

%% file: figures/diagrams/theory-features/theory-features-table.tex
\begingroup
\def\FeaturesGlyphs{1}%
\def\FeaturesColor{1}%
\def\FeaturesGrowingDepth{0}%
\def\FeaturesSensitivityMark{0}%
\input{figures/diagrams/theory-features/theory-features-table-style.tex}

\input{figures/diagrams/theory-features/theory-features-table-body.tex}
\endgroup

%% file: figures/diagrams/theory-features/theory-features-table-style.tex
\ifx\FeaturesGlyphs\undefined\def\FeaturesGlyphs{1}\fi%
\ifx\FeaturesColor\undefined\def\FeaturesColor{1}\fi%
\ifx\FeaturesGrowingDepth\undefined\def\FeaturesGrowingDepth{0}\fi%
\ifx\FeaturesSensitivityMark\undefined\def\FeaturesSensitivityMark{0}\fi%

\newcommand{\tfLabelMid}{2pt}%
\newcommand{\tfCellPad}{1.2pt}%

\ifdefined\tfCellBox\else\newsavebox{\tfCellBox}\fi
\ifdefined\tfCellHalf\else\newlength{\tfCellHalf}\fi

\newcommand{\tfUseCenteredCellBox}{%
  \setlength{\tfCellHalf}{\dimexpr\ht\tfCellBox-\tfLabelMid\relax}%
  \ifdim\dimexpr\dp\tfCellBox+\tfLabelMid\relax>\tfCellHalf
    \setlength{\tfCellHalf}{\dimexpr\dp\tfCellBox+\tfLabelMid\relax}%
  \fi
  \ht\tfCellBox=\dimexpr\tfLabelMid+\tfCellHalf+\tfCellPad\relax
  \dp\tfCellBox=\dimexpr\tfCellHalf-\tfLabelMid+\tfCellPad\relax
  \usebox{\tfCellBox}}

\newcommand{\tfCell}[1]{%
  \sbox{\tfCellBox}{\strut#1}%
  \tfUseCenteredCellBox}

\newcommand{\tfCircleRadius}{2.2pt}
\newcommand{\tfYes}{\tfCell{\tikz[baseline=-\tfLabelMid]
  \fill (0,0) circle[radius=\tfCircleRadius];}}
\newcommand{\tfNo}{\tfCell{\tikz[baseline=-\tfLabelMid]
  \draw[line width=0.35pt] (0,0) circle[radius=\tfCircleRadius-0.175pt];}}

\newcommand{\tfFeatureIndent}{1em}
\newcommand{\tfFeature}[1]{%
  \sbox{\tfCellBox}{\hspace*{\tfFeatureIndent}%
    \parbox{\dimexpr\linewidth-\tfFeatureIndent\relax}{\raggedright\strut#1\strut}}%
  \tfUseCenteredCellBox}

\newcommand{\tfGroup}[1]{%
  \addlinespace[\tfGroupLabelSkip]%
  \multicolumn{8}{@{}l@{}}{\sbox{\tfCellBox}{\bfseries\strut#1}\tfUseCenteredCellBox}\\}

\newcommand{\tfSensitivityHeader}{{\Sensitivity}\ifnum\FeaturesSensitivityMark=1\rlap{$^{*}$}\fi}

\definecolor{tfMarkInk}{HTML}{6B7781}%
\ifnum\FeaturesColor=1
  \definecolor{tfSymbols}{HTML}{2F6FB3}\definecolor{tfBits}{HTML}{9B4F96}
  \definecolor{tfCausal}{HTML}{087E9A}\definecolor{tfBidir}{HTML}{9B4F96}
  \definecolor{tfHard}{HTML}{246FB1}\definecolor{tfSoft}{HTML}{A64E73}
  \colorlet{tfUnif}{tfSoft}\definecolor{tfMsg}{HTML}{8A5BA6}
  \definecolor{tfBounded}{HTML}{287B64}\definecolor{tfUnbounded}{HTML}{BE512E}
  \colorlet{tfMachines}{tfBounded}
\else
  \colorlet{tfSymbols}{tfMarkInk}\colorlet{tfBits}{tfMarkInk}\colorlet{tfMachines}{tfMarkInk}
  \colorlet{tfCausal}{tfMarkInk}\colorlet{tfBidir}{tfMarkInk}
  \colorlet{tfHard}{tfMarkInk}\colorlet{tfSoft}{tfMarkInk}\colorlet{tfUnif}{tfMarkInk}\colorlet{tfMsg}{tfMarkInk}
  \colorlet{tfBounded}{tfMarkInk}\colorlet{tfUnbounded}{tfMarkInk}
\fi

\newcommand{\tfGlyphSize}{5.8pt}
\newlength{\tfGlyphLen}\setlength{\tfGlyphLen}{\tfGlyphSize}
\def\tfUnit{\tfGlyphSize}%
\tikzset{
  tf line/.style={line width={#1*\tfUnit}}, tf line/.default=.09,
  tf tip/.style={-{Stealth[length={.32*\tfUnit},width={.34*\tfUnit}]}}}
\newcommand{\tfGlyphFont}[1]{\fontsize{#1\tfGlyphLen}{#1\tfGlyphLen}\selectfont}

\newcommand{\tfValue}[3]{%
  \ifnum\FeaturesGlyphs=1
    \tikz[baseline=0pt,x=\tfGlyphSize,y=\tfGlyphSize,line cap=round,line join=round,color=#1]{#2}\hspace{.28em}#3%
  \else\ifnum\FeaturesColor=1\textcolor{#1}{#3}\else#3\fi\fi}

\newcommand{\tfTokSymbols}{\tfValue{tfSymbols}{%
  \draw[tf line,rounded corners={.18*\tfUnit}] (0,0) rectangle (.95,.95);
  \node[inner sep=0pt,anchor=base,font=\tfGlyphFont{.92}\itshape] at (.475,.28) {a};}{symbols}}
\newcommand{\tfTokBits}{\tfValue{tfBits}{%
  \draw[tf line,rounded corners={.18*\tfUnit}] (0,0) rectangle (1.25,.95);
  \node[inner sep=0pt,anchor=base,font=\tfGlyphFont{.78}\ttfamily] at (.625,.2) {01};}{bits}}
\newcommand{\tfTokMachines}{\tfValue{tfMachines}{%
  \draw[tf line,rounded corners={.12*\tfUnit}] (0,0) rectangle (1.1,.92);
  \draw[tf line] (.18,.62) -- (.7,.62) (.18,.31) -- (.7,.31);
  \fill (.88,.62) circle (.07) (.88,.31) circle (.07);}{machines}}

\newcommand{\tfHarpoonCausal}{\draw[tf line=.17] (1.3,.4) -- (0,.4) -- (.46,.78);}
\newcommand{\tfHarpoonBidir}{\draw[tf line=.17] (1.3,.66) -- (0,.66) -- (.46,1.02);
  \draw[tf line=.17] (0,.22) -- (1.3,.22) -- (.84,-.14);}
\newcommand{\tfFlowCausal}{\tfValue{tfCausal}{\tfHarpoonCausal}{causal}}
\newcommand{\tfFlowBidir}{\tfValue{tfBidir}{\tfHarpoonBidir}{bidirect.}}

\newcommand{\tfHist}[4]{\fill (0,0) rectangle (.24,#1) (.32,0) rectangle (.56,#2)
  (.64,0) rectangle (.88,#3) (.96,0) rectangle (1.2,#4); \draw[tf line=.05] (-.04,0) -- (1.24,0);}
\newcommand{\tfAttnHard}{\tfValue{tfHard}{\tfHist{.06}{1}{.06}{.06}}{hard}}
\newcommand{\tfAttnSoft}{\tfValue{tfSoft}{\tfHist{.25}{1}{.5}{.15}}{softmax}}
\newcommand{\tfAttnUnif}{\tfValue{tfUnif}{\tfHist{.45}{.45}{.45}{.45}}{uniform}}
\newcommand{\tfAttnMsg}{\tfValue{tfMsg}{\fill (.1,.45) circle (.13) (1.15,.45) circle (.13);
  \draw[tf line,tf tip] (.28,.45) -- (.95,.45);}{msg.\ passing}}

\newcommand{\tfBounded}{\tfValue{tfBounded}{\draw[tf line=.11] (.42,.42) circle (.36);}{bounded}}
\newcommand{\tfUnbounded}{\tfValue{tfUnbounded}{\draw[tf line=.11]
  (.62,.42) .. controls (.35,.85) and (0,.65) .. (0,.42) .. controls (0,.19) and (.35,0) .. (.62,.42)
  .. controls (.89,.85) and (1.24,.65) .. (1.24,.42) .. controls (1.24,.19) and (.89,0) .. (.62,.42);}{unbounded}}

\newcommand{\tfSupGlyph}[1]{\mbox{\hspace{.7pt}\def\tfUnit{3.6pt}%
  \tikz[baseline=0pt,x=3.6pt,y=3.6pt,line cap=round,line join=round]{#1}}}
\newcommand{\tfT}[2]{T^{\tfSupGlyph{#1}}_{\mathrm{#2}}}

\newlength{\tfFeatureWidth}
\newlength{\tfWidthStarfree}
\newlength{\tfWidthACz}
\newlength{\tfWidthTCz}
\newlength{\tfWidthCRASP}
\newlength{\tfWidthMPC}
\newlength{\tfWidthBAPO}
\newlength{\tfWidthSensitivity}
\newlength{\tfMeasuredWidth}

\makeatletter
\newcommand{\tfFitWidth}[2]{%
  \setlength{#1}{0pt}%
  \@for\tf@label:=#2\do{%
    \settowidth{\tfMeasuredWidth}{\tf@label}%
    \ifdim\tfMeasuredWidth>#1\setlength{#1}{\tfMeasuredWidth}\fi}%
  \addtolength{#1}{1pt}}
\makeatother

\newcommand{\tfSetColumnWidths}{%
  \settowidth{\tfFeatureWidth}{Communication}%
  \addtolength{\tfFeatureWidth}{\dimexpr12pt+\tfFeatureIndent\relax}%
  \tfFitWidth{\tfWidthStarfree}{\Starfree,\tfTokSymbols,\tfFlowCausal,\tfAttnHard,\tfBounded,\tfRelSF}%
  \tfFitWidth{\tfWidthACz}{\ACz,\tfTokBits,\tfFlowBidir,\tfAttnHard,\tfBounded,\tfRelAC}%
  \tfFitWidth{\tfWidthTCz}{\TCz,\tfTokBits,\tfFlowBidir,\tfAttnSoft,\tfBounded,\tfRelTC}%
  \tfFitWidth{\tfWidthCRASP}{\CRASP,\tfTokSymbols,\tfFlowCausal,\tfAttnUnif,\tfBounded,\tfRelCR}%
  \tfFitWidth{\tfWidthMPC}{\MPC,\tfTokMachines,\tfFlowBidir,\tfAttnMsg,\tfUnbounded,\tfRelMPC}%
  \tfFitWidth{\tfWidthBAPO}{\BAPO,\tfTokSymbols,\tfFlowCausal,\tfAttnHard,\tfUnbounded,none}%
  \tfFitWidth{\tfWidthSensitivity}{\tfSensitivityHeader,\tfTokBits,\tfFlowBidir,\tfAttnSoft,\tfBounded,learnability}}

\newcolumntype{F}{>{\raggedright\arraybackslash}p{\tfFeatureWidth}}%
\newcolumntype{T}[1]{>{\centering\arraybackslash}p{#1}}%

\setlength{\aboverulesep}{0pt}
\setlength{\belowrulesep}{0pt}

\newcommand{\tfGroupSkip}{3pt}
\newcommand{\tfGroupLabelSkip}{1.5pt}
\newcommand{\tfGroupRule}{%
  \addlinespace[\tfGroupSkip]%
  \arrayrulecolor{PredictionFamilyRule}%
  \specialrule{0.25pt}{0pt}{0pt}%
  \arrayrulecolor{black}}

\newenvironment{tfTable}{%
  \fontsize{8}{9.2}\selectfont
  \tfSetColumnWidths
  \setlength{\tabcolsep}{0pt}%
  \renewcommand{\arraystretch}{0}%
  \begin{tabular*}{\textwidth}{@{\extracolsep{\fill}}F T{\tfWidthStarfree} T{\tfWidthACz}
      T{\tfWidthTCz} T{\tfWidthCRASP} T{\tfWidthMPC} T{\tfWidthBAPO}
      T{\tfWidthSensitivity}@{}}%
}{%
  \end{tabular*}}

%% file: figures/diagrams/theory-features/theory-features-table-body.tex
\newcommand{\tfRelSF}{$=\tfT{\tfHarpoonCausal}{hard}$}
\newcommand{\tfRelAC}{$\supset\tfT{\tfHarpoonBidir}{hard}$}
\newcommand{\tfRelTC}{$\supset\tfT{\tfHarpoonBidir}{soft}$}
\newcommand{\tfRelCR}{$\subset\tfT{\tfHarpoonCausal}{soft}$}
\newcommand{\tfRelMPC}{$\supset\tfT{\tfHarpoonBidir}{soft}$}

\begin{tfTable}
\toprule
  & \tfCell{\Starfree}
  & \tfCell{\ACz}
  & \tfCell{\TCz}
  & \tfCell{\CRASP}
  & \tfCell{\MPC}
  & \tfCell{\BAPO}
  & \tfCell{\tfSensitivityHeader} \\
\midrule
\tfGroup{Key Assumptions}
\tfFeature{Token primitives}
  & \tfCell{\tfTokSymbols}
  & \tfCell{\tfTokBits}
  & \tfCell{\tfTokBits}
  & \tfCell{\tfTokSymbols}
  & \tfCell{\tfTokSymbols} %
  & \tfCell{\tfTokSymbols}
  & \tfCell{\tfTokBits} \\
\tfFeature{Information flow}
  & \tfCell{\tfFlowCausal}
  & \tfCell{\tfFlowBidir}
  & \tfCell{\tfFlowBidir}
  & \tfCell{\tfFlowCausal}
  & \tfCell{\tfFlowBidir}
  & \tfCell{\tfFlowCausal}
  & \tfCell{\tfFlowBidir} \\
\tfFeature{Attention type}
  & \tfCell{\tfAttnHard}
  & \tfCell{\tfAttnHard}
  & \tfCell{\tfAttnSoft}
  & \tfCell{\tfAttnUnif}
  & \tfCell{\tfAttnMsg}
  & \tfCell{\tfAttnHard}
  & \tfCell{\tfAttnSoft} \\
\tfFeature{Local computation}
  & \tfCell{\tfBounded}
  & \tfCell{\tfBounded}
  & \tfCell{\tfBounded}
  & \tfCell{\tfBounded}
  & \tfCell{\tfUnbounded}
  & \tfCell{\tfUnbounded}
  & \tfCell{\tfBounded} \\
\tfFeature{Formal relation to transformers}
  & \tfCell{\tfRelSF}
  & \tfCell{\tfRelAC}
  & \tfCell{\tfRelTC}
  & \tfCell{\tfRelCR}
  & \tfCell{\tfRelMPC}
  & \tfCell{none}
  & \tfCell{learnability} \\
\tfGroupRule
\tfGroup{Captured Features}
\tfFeature{Dynamic attention}
  & \tfYes & \tfYes & \tfYes & \tfNo & \tfYes & \tfYes & \tfYes \\
\tfFeature{Depth}
  & \tfYes & \tfYes & \tfYes & \tfYes & \tfYes & \tfNo & \tfYes \\
\ifnum\FeaturesGrowingDepth=1
\tfFeature{Growing depth}
  & \tfNo & \tfNo & \tfNo & \tfNo & \tfYes & \tfNo & \tfNo \\
\fi
\tfFeature{Communication bandwidth}
  & \tfNo & \tfNo & \tfNo & \tfNo & \tfYes & \tfYes & \tfNo \\
\tfFeature{Learnability}
  & \tfNo & \tfNo & \tfNo & \tfNo & \tfNo & \tfNo & \tfYes \\
\bottomrule
\end{tfTable}

%% file: figures/diagrams/theory-transformer-relations/theory-transformer-relation-tikz.tex
\begingroup
\input{figures/diagrams/theory-transformer-relations/theory-transformer-relation-style.tex}
\defcitealias{yang2024masked}{YCA24}
\defcitealias{barcelo2024logical}{B+24}
\defcitealias{hao2022formal}{HAF22}
\defcitealias{strobl2023average}{S23}
\defcitealias{merrill2023logic}{MS23}
\defcitealias{svete2026padded}{SMCS26}
\defcitealias{sanford2024parallel}{SHT24}
\defcitealias{yang2024counting}{YC24}
\renewcommand{\relref}[3]{%
  \if\relax\detokenize{#3}\relax\citetalias{#2}\else\citetalias[#3]{#2}\fi}
\renewcommand{\relownref}{Ours, Th.~\ref{thm:c-rasp-retrieval-separation}}
\resizebox{\linewidth}{!}{%
  \input{figures/diagrams/theory-transformer-relations/theory-transformer-relation-body.tex}%
}%
\endgroup

%% file: sections/06-related-work.tex
\subsection{Additional Related Work}
\label{sec:related_work}

\paragraph{Other expressivity results.} 
Beyond the seven theories we describe above, there have been many other theoretical advances in the study of transformers. 
\citet{strobl2024what} provide a comprehensive survey of this area. Early results include that sufficiently large transformers can approximate arbitrary sequence-to-sequence functions~\citep{Yun2020Are} and analyses of transformer expressivity for problems like \TaskDyck and \TaskParity~\citep{bhattamishra2020ability,hahn2020theoretical}.
The introduction of the \RASP programming language~\citep{weiss2021thinking} provided the basis for the line of theoretical work on \textsf{B-} and \CRASP.
Other notable work includes \citet{liu2023transformers}, who show how shallow transformers can simulate automata using shortcuts. Other work outside our comparison provides task-specific analyses rather than general predictive frameworks~\citep{sanford2023representational,chen2024theoretical,yao2021self}. 

\paragraph{Learnability and length generalization.}
Much of the literature has focused on expressivity rather than learnability, but there are some exceptions.
\citet{edelman2022inductive} use covering number bounds to argue that transformers can sample-efficiently learn sparse functions. Recent work has also started to explore the sample complexity of learning \CRASP programs~\citep{rizvimartel2026sample}.
\citet{abbe2024globality} conjecture that transformers can weakly learn functions if and only if they have constant globality, a distributional property roughly capturing the number of relevant input tokens.
Beyond expressivity, there has also been some theoretical work on length generalization. \citet{huang2024formal} show that an idealized inference procedure, which selects the lowest-complexity transformer that fits all inputs up to a given length, length-generalizes on problems expressible in variants of \CRASP matched to the positional encoding, such as $\CRASP[\varnothing]$ for NoPE.
This provides a formal account of length generalization related to the \RASP{}\textsf{-L} conjecture~\citep{zhou2024what}.
\citet{golowich2025role} show that an abstraction of attention heads can learn to length generalize on sparse functions. 

\paragraph{Chain of thought.}
Another line of research concerns the power of transformers with chain of thought (CoT).
\citet{perez2019turing} show that CoT makes transformers Turing-complete. Subsequent work has provided finer-grained characterizations of how different CoT budgets affect expressive power~\citep{merrill2024expressive,li2024chain}, shown CoT step lower bounds for specific problems~\citep{bavandpour2025lower,tomlinson2026reasoning}, and extended Turing-completeness results to probabilistic Turing machines~\citep{nowak2024representational}.
Some CoT work has also examined the effect it has on learnability~\citep{altabaa2025cot,huang2025multistep,huang2025length,joshi2025theory}.
Related to CoT results, there has also been work on the expressivity of transformers with scaling depth~\citep{merrill2025little} and padding tokens~\citep{merrill2026exact}. 

\paragraph{Empirical work informing theory.}
There is a large existing literature on experiments with synthetic tasks to better understand the capabilities of transformers. 
\citet{bhattamishra2023simplicity} study sensitivity-based inductive biases in transformers and LSTMs, relating differences in these biases to performance on sparse Boolean functions. 
\citet{deletang2023neural} test whether the Chomsky hierarchy predicts length generalization across neural architectures, including transformers and RNNs. \citet{kazemnejad2023impact} show that the choice of positional encoding substantially affects length generalization on reasoning and mathematical tasks. 
While much of this empirical work has focused on small purpose-trained transformers, \citet{jobanputra2025born} explore how LLM pretraining affects transformer capabilities on retrieval and copying tasks. 
Our contribution is to compare competing theoretical predictions on a common set of tasks, evaluating their explanatory power across in-distribution sample complexity, length generalization, and pretrained-model performance.

%% file: sections/03-methods.tex
\section{Methodology}
\label{sec:methodology}

We compare the theories of transformers by their ability to correctly predict empirical outcomes, especially those that were not used to construct them~\citep{popper1959logic}.
As we will see, these seven theories of transformers produce conflicting predictions on our suite of tasks, as they rest on distinct conceptual foundations and assumptions. We rely on empirical experiments with both small task-specific transformers as well as frontier LLMs to reveal which abstractions capture the features of transformers that matter in practice.
Our goal is to use these tests to map out each theory's domain of validity, and to draw a synthesis that guides future theoretical work.

\subsection{Suite of Tasks}
\label{ssec:tasks}

We evaluate the theories on a suite of tasks from seven families, summarized below.
We designed the suite with two goals: (1) to cover a range of algorithmic and computational structures, and (2) to target tasks on which theories disagree.
Each task tests whether transformers are able to learn and represent a particular computational structure.
For example, \TaskMajority tests whether transformers can perform global aggregation, \TaskDyck tests whether transformers can track stack-like nesting structures, and \TaskKHop tests whether transformers can compose content-dependent retrievals. The full suite of tasks was as follows:

\begin{itemize}
    \item \textbf{Boolean functions:} \TaskALL, \TaskANY, \TaskMajority, \TaskParity.
    \item \textbf{Indexed retrieval:} \TaskIndex, \TaskIndexAndOr, \TaskIndexXor. Given a list and one or more indices $i$, retrieve values $\text{list}[i]$ and (in the case of \TaskIndexStar) combine them.
    \item \textbf{Sparse functions:} \TaskSparseAndOr, \TaskSparseXor. Each applies a Boolean function to $k$ fixed positions learned from data.
    \item \textbf{Arithmetic:} \TaskBinaryAddition.
    \item \textbf{Dyck languages:} \TaskDyck, with bounded and unbounded depth.
    \item \textbf{Iterated retrieval:} \TaskKHop for $k \in \{1,2,3\}$, which composes $k$ content-based lookup operations.
    \item \textbf{Graph tasks:} \TaskEdgeExistence, \TaskUndirectedReachability, \TaskDirectedReachability.
\end{itemize}

Details of each task, including formal definitions, input distributions, output formats, and variants, are described in~\cref{app:task_details}.

Across all tasks, the size or complexity of the input instance is parameterized by $n$.
Formally, a task is a family of problems $f = (f_n)_{n \geq 1}$, one for each input size $n$, with a single finite description that applies to every size $n$.
For example, ``$\TaskALL(x_1, \ldots, x_n) = \bmone\{x_i = 1, \, \forall i \in [n]\}$'' is one description that defines the task for every $n$.
The theories we consider make claims about whether a single transformer, whose size is fixed in $n$,\footnote{Some theories allow the numerical precision to grow logarithmically in $n$.} can solve a task for all $n$.
Such claims describe how well the transformer architecture captures a given computational \emph{structure}.
For a finite $n$, a task may be solvable by simply memorizing a lookup table using a large-enough model.
However, if a task is solvable by a single model of finite size \emph{for all input sizes $n$}, then the transformer architecture represents that computational structure natively.
In our experimental design, described in~\cref{ssec:experiment_methodology}, we assess this \emph{structural} alignment by evaluating the performance of transformers on these tasks \emph{as the instance size $n$ grows}.

\subsection{From Theories to Predictions}
\label{ssec:translating}

\input{tables/theory-task-matrix/theory-task-matrix-empirical.tex}

To test the theories on our task suite, we first apply each theory to each task to derive a prediction.
The outcome of this process is a matrix of theory--task predictions, shown in \cref{tab:theory-task-predictions}.
The theories make narrow formal claims of different kinds: expressivity upper or lower bounds, learnability claims, or bandwidth requirements.
For each theory, we classify each task as either \Easy or \Hard according to the theory's abstract model, regardless of the formal claim that links the model to transformers. Our goal is to understand whether the abstract model itself makes predictions that are aligned with the empirical behavior of trained transformers.
We provide an annotated version with references for each cell in \cref{app:prediction-notes}, \cref{tab:theory-task-predictions-annotated}.

We fill each cell in the matrix via one of four possible sources: (i) direct analysis in the theory's own paper(s); (ii) known results about the abstract model; (iii) new results proved in this paper; (iv) conditional results that depend on standard complexity conjectures.
When theoretical analysis is missing for a given theory--task pair, we apply the theory's framework to the task ourselves and prove the resulting prediction: constructions for \Easy predictions and lower bounds or impossibility results for \Hard predictions.
These new results are marked in~\cref{tab:theory-task-predictions} and the proofs are provided in \cref{app:theory-task-results}.

\subsection{Experimental Design}
\label{ssec:experiment_methodology}

In our experimental design, we evaluate three notions of empirical success and test each with its own line of experiments.
The first line of experiments evaluates whether transformers can learn a task and generalize in-distribution as the input size $n$ grows.
The second line of experiments evaluates whether models trained on small instances can generalize \emph{out-of-distribution} to larger instances.
The third line of experiments evaluates whether pretrained large language models can solve each task, without task-specific training.
In each setting, we measure how performance scales with the input size $n$. 
Since different tasks have different \emph{chance-level accuracy} baselines, we measure performance via the \emph{normalized accuracy} metric $(a - b) / (1 - b)$, where $a$ is the accuracy and $b$ is the chance-level baseline for the given task and input size $n$ (see \cref{app:normalized-accuracy}).\footnote{For tasks whose answer is a sequence (\TaskBinaryAddition and \TaskDyck), we report full-sequence accuracy without normalization ($b = 0$).}
This maps chance-level accuracy to $0 \%$ and perfect accuracy to $100 \%$.
We use the normalized accuracy threshold $\tau = 90\%$ to classify whether a task is \emph{solved} by a given model at a given size $n$.
Each experiment then grades a task as \Easy or \Hard based on whether the model solves the task as $n$ grows.

\paragraph{In-distribution learning.}
This experiment evaluates whether a transformer of fixed size can learn the task from examples, and how much data it needs as the input size $n$ grows.
For this experiment, we use a single causal transformer architecture (4 layers, width 256, 4 heads, rotary positional encoding).
For each task, we train a transformer model with this architecture across a range of input sizes $n \in \{8, 16, \ldots, 256\}$ and training set sizes $N \in \{2^8, 2^{10}, \ldots, 2^{16} \}$. 
Every run uses the same number of optimizer steps to match the training compute budget across dataset sizes, so that a smaller $N$ means more passes over fewer distinct training examples.
For each configuration of (task, $n$, $N$), we repeat the experiment with 5 random seeds.
A task is graded as \emph{solved} at size $n$ if, for some training set size $N$, at least one seed reaches the $\tau = 90 \%$ threshold.

\paragraph{Length generalization.}
Many expressivity theories make claims of the form ``there exists a transformer $T_\theta$ that solves $\mathtt{\langle task \rangle}$ for all sizes $n$''.
This suggests length generalization as a testable empirical analogue: can transformers \emph{discover} such a length-generalizing solution?
This stronger form of generalization, to larger out-of-distribution inputs, requires robustly learning the underlying computational structure of the task.
To evaluate this, we train on $n \in \{8, \ldots, 32\}$ and evaluate the same model on $n$ up to $256$, eight times the longest training input.
Length generalization is known to depend strongly on the architecture, and especially on the positional encoding~\citep{anil2022exploring,kazemnejad2023impact,zhou2024what}.
To control and evaluate the effect of the architecture, we train every task with 3 model sizes and 9 positional encodings, including NoPE, RoPE, absolute, and relative schemes.
Details of model sizes and positional encodings are in \cref{app:length-generalization-details}.
We repeat the experiment with 5 random seeds for each configuration.
A task is solved at size $n$ if a single trained model (one seed at a given architecture) exceeds the normalized accuracy threshold $\tau = 90\%$ at every test size up to $n$.

\paragraph{Pretrained large language models.}
The two lines of experiments described above evaluate small transformers trained from scratch on each task.
Large language models (LLMs) are believed to use the same transformer architecture, but are much larger and are trained on broad data with large-scale pretraining and post-training pipelines.
This experiment evaluates whether the predictions of transformer theories apply to such models and their complex training recipes.
We evaluate five pretrained models from several model providers on the same set of tasks, varying input size $n \in \{8, \ldots, 128\}$.
Each prompt describes the task in natural language and gives three solved examples, followed by the test input.
With enough intermediate reasoning steps, a transformer can theoretically simulate any computation~\citep{perez2019turing,merrill2024expressive}.
This falls outside the single-pass setting of the seven theories we evaluate.
Thus, we disable reasoning in our main evaluation; we report results with reasoning enabled in \cref{app:llm-reasoning}.
We omit the \TaskSparseStar tasks, since their hidden positions must be learned at training time.
We grade tasks similarly to the previous experiments: each model is graded on its own based on whether it solves the task at the largest tested size ($n = 128$), and the task receives the best grade reached by any single model.

%% file: tables/theory-task-matrix/theory-task-matrix-empirical.tex
\begin{table*}[t!]
\centering
\caption{\textbf{Theory--task predictions and experimental outcomes.}
Left: for each theory, we apply its hardness criterion, adapting tasks when necessary to fit within the theory's assumptions about inputs.
See \cref{tab:theory-task-predictions-annotated} for an annotated version with explanations and references for each claim.
Right: the outcome of each experiment in \cref{sec:results}: in-distribution learning (ID), length generalization (LG), and pretrained LLMs (LLM), graded as in \cref{ssec:experiment_methodology} on the task settings of \cref{fig:results}.
\predneg{}: \Hard, \predpos{}: \Easy (for predictions, in some cases, not \Hard), \textsuperscript{\prednewproof}: non-trivial result proved here, $^*$: assuming a complexity theory conjecture, ---: inapplicable.
The bottom rows count, for each theory, the predictions that match each experiment, among cells with both a prediction and an outcome.
}
\label{tab:theory-task-predictions}
\def\AgreementBars{1}%
\def\AgreementFadeSensitivity{0}%
\def\AgreementSensitivityDenominator{0}%
\def\AgreementRowDenominator{1}%
\input{tables/theory-task-matrix/theory-task-matrix-empirical-content}
\end{table*}

%% file: tables/theory-task-matrix/theory-task-matrix-empirical-content.tex
\begingroup
\fontsize{8}{10}\selectfont
\setlength{\tabcolsep}{2.25pt}
\renewcommand{\arraystretch}{1.15}
\begin{tabular*}{\textwidth}{@{\extracolsep{\fill}}l c c c c c c c @{\hspace{4pt}{\color{PredictionFamilyRule}\vrule width 0.4pt}\hspace{4pt}} c c c@{}}
\toprule
 & \multicolumn{7}{c}{\textbf{Theory predictions}} & \multicolumn{3}{c}{\textbf{Empirical}} \\
\cmidrule(lr){2-8}\cmidrule(l){9-11}
\textbf{Task} & \textbf{\Starfree} & \textbf{\boldmath\ACz} & \textbf{\boldmath\TCz} & \textbf{\CRASP} & \textbf{\MPC} & \textbf{\BAPO} & \textbf{\Sensitivity} & \textbf{ID} & \textbf{LG} & \textbf{LLM} \\
\midrule
\multicolumn{8}{@{}l}{\textbf{Boolean}} & & & \\[1pt]
\quad \TaskALL/\TaskANY & \predpos{} & \predpos{} & \predpos{} & \predpos{} & \predpos{} & \predpos{} & \predpos{} & \predpos{} & \predpos{} & \predpos{} \\
\quad \TaskMajority & \predneg{} & \predneg{} & \predpos{} & \predpos{} & \predpos{} & \predneg{} & \predpos{} & \predpos{} & \predpos{} & \predneg{} \\
\quad \TaskParity & \predneg{} & \predneg{} & \predpos{} & \predneg{} & \predpos{} & \predpos{} & \predneg{} & \predneg{} & \predneg{} & \predneg{} \\
\predfamilyrule
\multicolumn{8}{@{}l}{\textbf{\TaskIndex}} & & & \\[1pt]
\quad \TaskIndex & \predneg{} & \predpos{} & \predpos{} & \predneg[\prednewproof]{} & \predpos{} & \predpos{} & \predpos{} & \predpos{} & \prednone & \predneg{} \\
\quad \TaskIndexAndOr & \predneg{} & \predpos{} & \predpos{} & \predneg[\prednewproof]{} & \predpos{} & \predpos[\prednewproof]{} & \predpos{} & \predneg{} & \prednone & \predneg{} \\
\quad \TaskIndexXor & \predneg{} & \predpos{} & \predpos{} & \predneg[\prednewproof]{} & \predpos{} & \predpos[\prednewproof]{} & \predpos{} & \predneg{} & \prednone & \predneg{} \\
\predfamilyrule
\multicolumn{8}{@{}l}{\textbf{Sparse}} & & & \\[1pt]
\quad \TaskSparseAndOr & \predpos{} & \predpos{} & \predpos{} & \predpos{} & \predpos{} & \predpos{} & \predpos{} & \predpos{} & \predpos{} & \prednone \\
\quad \TaskSparseXor & \predpos{} & \predpos{} & \predpos{} & \predpos{} & \predpos{} & \predpos{} & \predpos{} & \predpos{} & \predpos{} & \prednone \\
\predfamilyrule
\textbf{\TaskBinaryAddition} & \predneg{} & \predpos{} & \predpos{} & \predneg{} & \predpos{} & \predneg[\prednewproof]{} & \predpos{} & \predneg{} & \predneg{} & \predneg{} \\
\predfamilyrule
\multicolumn{8}{@{}l}{\textbf{\TaskDyck}} & & & \\[1pt]
\quad \TaskDyckOne (bounded) & \predpos{} & \predpos{} & \predpos{} & \predpos{} & \predpos{} & \predpos{} & \prednone & \predpos{} & \predneg{} & \predneg{} \\
\quad \TaskDyckOne (unbounded) & \predneg{} & \predneg{} & \predpos{} & \predpos{} & \predpos{} & \predneg[\prednewproof]{} & \prednone & \predneg{} & \predpos{} & \predneg{} \\
\quad \TaskDyckTwo (bounded) & \predpos{} & \predpos{} & \predpos{} & \predpos[\prednewproof]{} & \predpos{} & \predpos{} & \prednone & \predpos{} & \predneg{} & \predneg{} \\
\quad \TaskDyckTwo (unbounded) & \predneg{} & \predneg{} & \predpos{} & \predneg{} & \predpos{} & \predneg[\prednewproof]{} & \prednone & \predneg{} & \predneg{} & \predneg{} \\
\predfamilyrule
\multicolumn{8}{@{}l}{\textbf{\TaskKHop}} & & & \\[1pt]
\quad $k = 1$ & \predpos[\prednewproof]{} & \predpos[\prednewproof]{} & \predpos[\prednewproof]{} & \predneg[\prednewproof]{} & \predpos{} & \predpos[\prednewproof]{} & \prednone & \predpos{} & \predpos{} & \predneg{} \\
\quad $k = 2$ & \predpos[\prednewproof]{} & \predpos[\prednewproof]{} & \predpos[\prednewproof]{} & \predneg[\prednewproof]{} & \predpos{} & \predpos[\prednewproof]{} & \prednone & \predneg{} & \predneg{} & \predneg{} \\
\quad $k = 3$ & \predpos[\prednewproof]{} & \predpos[\prednewproof]{} & \predpos[\prednewproof]{} & \predneg[\prednewproof]{} & \predpos{} & \predpos[\prednewproof]{} & \prednone & \predneg{} & \predneg{} & \predneg{} \\
\predfamilyrule
\multicolumn{8}{@{}l}{\textbf{Graphs}} & & & \\[1pt]
\quad \TaskEdgeExistence & \predneg{} & \predpos{} & \predpos{} & \predneg[\prednewproof]{} & \predpos{} & \predpos{} & \prednone & \predneg{} & \predneg{} & \predpos{} \\
\quad \TaskDirectedReachability & \predneg{} & \predneg{} & \prednegc{} & \prednegc{} & \prednegc{} & \predneg{} & \prednone & \predneg{} & \predneg{} & \predneg{} \\
\quad \TaskUndirectedReachability & \predneg{} & \predneg{} & \prednegc{} & \prednegc{} & \prednegc{} & \predneg{} & \prednone & \predneg{} & \predneg{} & \predneg{} \\
\midrule
\multicolumn{8}{@{}l}{\textbf{Agreement with experiments}} & & & \\[1pt]
\predagreementlabel{ID}{19} & \predagreementcell{6}{8}{9}{11}{0}{15}{19}{} & \predagreementcell{7}{13}{5}{6}{0}{12}{19}{} & \predagreementcell{8}{17}{2}{2}{0}{10}{19}{} & \predagreementcell{6}{7}{10}{12}{0}{16}{19}{} & \predagreementcell{8}{17}{2}{2}{0}{10}{19}{} & \predagreementcell{7}{13}{5}{6}{0}{12}{19}{} & \predagreementcell{5}{8}{1}{1}{1}{6}{9}{(of 9)} & & & \\[\predagreementrowskip]
\predagreementlabel{LG}{16} & \predagreementcell{4}{8}{6}{8}{0}{15}{19}{} & \predagreementcell{4}{10}{4}{6}{0}{12}{19}{} & \predagreementcell{6}{14}{2}{2}{0}{10}{19}{} & \predagreementcell{5}{7}{8}{9}{0}{16}{19}{} & \predagreementcell{6}{14}{2}{2}{0}{10}{19}{} & \predagreementcell{4}{10}{4}{6}{0}{12}{19}{} & \predagreementcell{4}{5}{1}{1}{1}{6}{9}{(of 9)} & \multicolumn{3}{l@{}}{\predagreementlegend} \\[\predagreementrowskip]
\predagreementlabel{LLM}{17} & \predagreementcell{1}{6}{10}{11}{0}{15}{19}{} & \predagreementcell{2}{11}{6}{6}{0}{12}{19}{} & \predagreementcell{2}{15}{2}{2}{0}{10}{19}{} & \predagreementcell{1}{5}{11}{12}{0}{16}{19}{} & \predagreementcell{2}{15}{2}{2}{0}{10}{19}{} & \predagreementcell{2}{11}{6}{6}{0}{12}{19}{} & \predagreementcell{1}{6}{1}{1}{1}{6}{9}{(of 9)} & & & \\[\predagreementrowskip]
\bottomrule
\end{tabular*}
\endgroup

%% file: sections/04-results.tex
\section{Results}
\label{sec:results}

We now compare the predictions of the seven theories against the outcomes of our three lines of experiments.
We first describe where the theories disagree and provide an overview of our theoretical contributions towards this analysis. We then report the outcomes of our experiments and what they tell us about the validity of different theoretical constructs.

\subsection{Theory}
The task predictions for each theory are shown in \Cref{tab:theory-task-predictions}, along with a coarse summary of the experiment results for easy comparison.

\paragraph{Where the theories agree and disagree.}
On \TaskALL/\TaskANY, \TaskSparseStar, and bounded \TaskDyck, all theories that apply predict these tasks to be \Easy.
On the remaining tasks, the seven theories make conflicting predictions.
For example, \TaskMajority requires global aggregation, which \Starfree, \ACz, and \BAPO lack.
\TaskParity requires modular counting, which \Starfree, \ACz, and \CRASP cannot express.
For retrieval tasks whose range grows with $n$ (\TaskIndexStar and \TaskEdgeExistence), \Starfree and \CRASP predict \Hard.
\Starfree describes only regular languages, which cannot track a position that grows with $n$; \CRASP predicts \Hard because it lacks content-dependent attention.\footnote{When the index range is bounded, \CRASP can simulate content-dependent retrieval by retrieving all possible queries and selecting between them in the positionwise computation (\Cref{res:index-fixed-crasp,res:index-bool-fixed-crasp}). }
This lack also makes \CRASP the only theory to predict \Hard for \TaskKHop (even when $k=1$, a new result we show in \Cref{res:k-hop-crasp}), an important fact given the central role induction heads are thought to perform in practice~\citep{elhage2021mathematical,olsson2022incontextlearninginductionheads}. 
Unbounded \TaskDyck is captured by \TCz and \MPC, but not the other theories, with the exception of \CRASP which can represent \TaskDyckOne but not \TaskDyckTwo.
Among the seven theories, \Starfree is the most restrictive, predicting \Easy on only 8 of 19 tasks, while \TCz and \MPC are the most powerful, predicting \Easy on 17 of 19.

\paragraph{New theoretical analyses.}
A few notable new results we provide to fill in \Cref{tab:theory-task-predictions} include \TaskDyck-$(k, D)$ in \CRASP (\Cref{res:dyck-crasp}), \TaskKHop not in \CRASP (even for $k=1$, \Cref{res:k-hop-crasp}), and \BAPO-hardness for unbounded \TaskDyck-$k$ (\Cref{res:dyck-bapo}) and \TaskBinaryAddition (\Cref{res:addition-bapo}). Additionally, we show a new exponential width separation between \CRASP and softmax attention transformers---despite the depth-for-depth equivalence between \CRASP and softmax transformers shown by~\citet{yang2024counting,yang2025kneedeep} via uniform attention. This is due to the lack of content-dependent attention in \CRASP, which we will see manifests in our experiments. Informally, we show the following result.
\begin{theorem*}[Informal]
    There exists a family of query-key retrieval tasks parameterized by $m$ such that: (i) every \CRASP program that solves it needs program size $\Omega(2^m / m)$, (ii) any uniform-attention transformer with depth $D$ and $p$-bit precision needs width $\Omega(2^m / (D (p +m)))$, and (iii) one softmax attention head solves the task with width $O(m)$.
\end{theorem*}
The formal result and the proof are provided in~\cref{app:c-rasp-separation}, along with experiments showing that the separation also appears when training softmax and uniform-attention transformers on these tasks.

\subsection{Experiments}
We first focus on the from-scratch transformer experiment results, grouped by task category, then summarize the LLM results. \Cref{fig:results} summarizes the results of the in-distribution, length generalization, and pretrained LLM experiments.

\begin{figure}[h!]
    \centering
    \includegraphics[width=\linewidth]{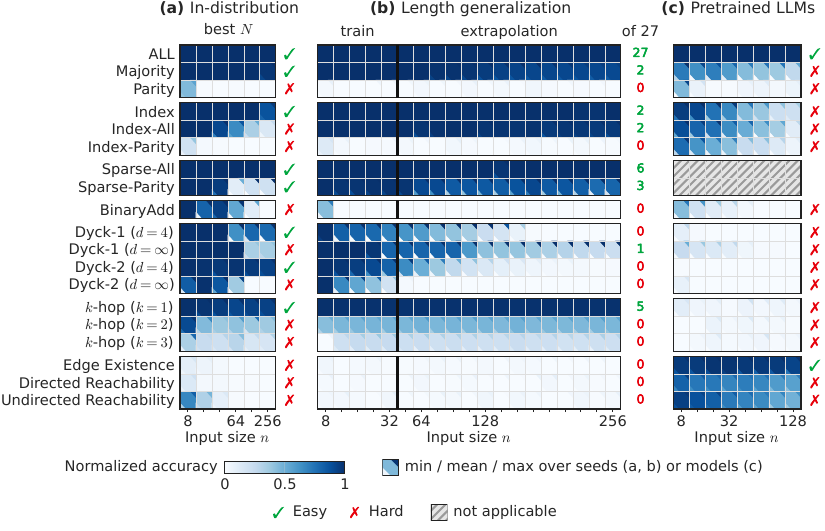}
    \caption{\textbf{Results of the three experiments.}
    Rows are task settings, and heatmap color shows mean normalized accuracy.
    Corner tabs, when present, show the minimum and maximum over five seeds in (a,b) or available models in (c).
    \textbf{(a)} In-distribution learning: each size $n$ shows the training-set size $N$ with the highest-scoring seed; the full $n \times N$ grid is in \cref{fig:app-id-landscape}.
    \textbf{(b)} Length generalization: models train through $n=32$ (left of the black line).
    Each row uses the configuration that generalizes furthest among 27; the adjacent number counts configurations with a seed succeeding at every tested size through $n=256$; \cref{app:results-lg} provides more details.
    \textbf{(c)} Pretrained LLMs: cells average the models available at each size; hatched rows are not applicable.
    In (a,c), \predpos{} marks \Easy and \predneg{} marks \Hard, using the $90\%$ threshold.
    A setting is \Easy if it meets the threshold at the largest tested $n$: in (a), with any seed at any $N$; in (c), with any single model.
    }
    \label{fig:results}
\end{figure}

\paragraph{Boolean functions.}
\TaskALL/\TaskANY, which measure whether transformers can detect the occurrence of a symbol, is predicted \Easy by all theories.
This is supported by the experiments, which find this task to be learnable and easy to length-generalize.
On the other hand, \TaskMajority separates theories by probing their ability to perform global aggregation and thresholding.
We find that transformers successfully learn and length-generalize on this task, which agrees with \TCz, \CRASP, \MPC, and \Sensitivity, and contradicts \Starfree, \ACz, and \BAPO.
This highlights that transformers do indeed possess global aggregation capabilities, which hard-attention theories fail to capture.
Another task that separates theories is \TaskParity, which tests for the ability to learn and perform a simple type of binary state-tracking.
We find that transformers fail to learn this task, even in-distribution at modest input sizes.
This agrees with the \Hard predictions of \Starfree, \ACz, \CRASP, and \Sensitivity, while contradicting \TCz, \MPC, and \BAPO.
However, among these theories, we suspect that \Sensitivity is closest to capturing the right \emph{mechanism} for this failure.
We observe that even \emph{training} accuracy stays low when the training set is large (e.g., at $n=8$, $N=65{,}536$, training-sample accuracy is only $49.7$--$59.1\%$ across the five seeds).
This suggests that the central barrier is learning and optimization, rather than merely expressivity, as previously argued by \citet{hahn2024sensitive}.
While certain more restrictive theories correctly predicted hardness, they do not seem to have captured the real underlying mechanism for this hardness.

\paragraph{Sparse functions.}
The \TaskSparseStar tasks measure the ability of transformers to learn functions of a hidden sparse subset of the inputs.
All theories predict that such sparse functions are \Easy for transformers.
Indeed, we find that \TaskSparseAndOr is efficiently learnable, and moreover length-generalizes for some architectures.
\TaskSparseXor is also \Easy in distribution, as all theories predict.
However, this success is fragile at longer inputs: for $n \geq 64$, at most one of 25 runs succeeds at each size.

\paragraph{Retrieval.}
The \TaskIndexStar tasks assess the ability of transformers to retrieve a sparse, query-dependent set of inputs from a growing context.
\Starfree and \CRASP predict \Hard for these tasks when the index range or key space grows with the input size. 
Other theories predict \Easy, as they provide mechanisms for efficient query-dependent retrieval through dynamic attention.
Empirically, we find that small transformers learn \TaskIndex at every input size.
\TaskIndexAndOr, which requires composing a retrieval with a simple computation (\TaskALL/\TaskANY), is harder: within our budget, it is solved only up to $n = 16$ to $64$, but it appears learnable with increased data.
Interestingly, however, we find \TaskIndexXor to be unlearnable, despite having the same modest sparsity level as \TaskIndexAndOr ($k=2$ or $k=3$) and in-distribution success on the similar \TaskSparseStar task.
In our length generalization experiments, we bound the maximum index range to make the length generalization problem well-defined.
This makes it a different task from the in-distribution one, so we do not compare it with the theory predictions (\Cref{tab:theory-task-predictions}).
Transformers do length-generalize on \TaskIndex and \TaskIndexAnd for certain architectural choices.

\paragraph{Iterated retrieval.}
The \TaskKHop tasks involve composing content-dependent information retrieval operations.
A single hop is the computation of an induction head, a two-layer attention circuit that mechanistic interpretability work has found in pretrained language models of many sizes~\citep{elhage2021mathematical,olsson2022incontextlearninginductionheads}.
Induction heads are thought to play an important role in in-context learning abilities of large language models~\citep{olsson2022incontextlearninginductionheads,crosbie2025induction}.
\TaskKHop composes $k$ such hops, and so tests whether transformers can chain this basic retrieval operation.
Here, unlike the \TaskIndexStar tasks, we take the alphabet to be fixed in $n$, and instead focus on the iterated structure.
All theories except \CRASP predict that \TaskKHop is \Easy at fixed $k$.
Empirically, we find that \TaskOneHop is learnable for small transformers trained from scratch, and even length-generalizes on 5 of 27 architectural configurations.
However, learning compositions is substantially harder.
In the in-distribution experiments, \TaskKHop is only solved at $n=8$ for $k=2$, and is never solved for $k=3$.

\paragraph{Stack-like nesting with bounded and unbounded depth.}
The \TaskDyck tasks assess the ability of transformers to track stack-like states and nesting structure.
We consider both bounded-depth and unbounded-depth \TaskDyck.
When depth is bounded, all theories that apply predict \Easy.
In-distribution, transformers trained from scratch solve bounded-depth ($d=4$) \TaskDyckOne and \TaskDyckTwo at every size.
However, they exhibit limited length generalization.
With unbounded depth, the theories disagree.
The most expressive theories, \TCz and \MPC, predict \Easy for all \TaskDyck-$k$, \CRASP predicts \Easy for \TaskDyckOne only, while \Starfree, \ACz, and \BAPO predict \Hard.
In-distribution, both \TaskDyckOne and \TaskDyckTwo are solved only up to $n = 64$, in line with the \Hard predictions.
Interestingly, however, with the benefit of an architecture search over different positional encodings, one of five seeds on one of 27 architectural configurations (medium size, NoPE with a BOS token) \emph{length-generalizes} on \TaskDyckOne to $n=256$.
This has an interesting connection to \CRASP theory, which provides a simple program that compiles to a uniform attention transformer that solves \TaskDyckOne for all input sizes by counting brackets~\citep{yang2024counting}.

\paragraph{Multi-step computation.}
The \TaskBinaryAddition task assesses the ability of transformers to perform multi-step computations associated with arithmetic.
The two circuit class theories, \ACz and \TCz, together with \MPC predict \Easy for this task, while the rest predict \Hard.
Empirically, we find that transformers succeed in learning the task only for small inputs, suggesting that circuit complexity classes, even the weaker \ACz, underestimate the difficulty of this task for transformers.
At larger sizes, accuracy still improves with more training data and varies widely across seeds, which points to a learning barrier rather than an expressivity limit.
Consistent with this, prior work has found success on addition by augmenting position embeddings~\citep{mcleish2024transformers,cho2024position,cho2025arithmetic}.
Another pair of tasks that involve multi-step computation is \TaskDirectedReachability and \TaskUndirectedReachability, which involve complex graph processing.
These tasks are the most difficult in our task suite, as they are predicted \Hard by all theories.\footnote{Some of these predictions are conditional on standard complexity conjectures.}
The experiments agree with this classification, as small transformers fail to learn the tasks, even for the smallest input sizes ($n=8$-node graphs).
Another graph task that was surprisingly difficult for transformers to learn in our experiments is the simpler \TaskEdgeExistence task.
This task only requires checking if a queried edge exists in the graph.
In our empirical setup, we found that small transformers failed to learn this task.
We provide a deeper analysis of this in the appendix, showing that the task becomes learnable with denser supervision on every edge rather than only on the final answer (\cref{app:edge-existence-supervision}). This suggests learning is once again the culprit.

\paragraph{Pretrained LLM experiments.}
The LLM evaluation experiments test whether the predictions of transformer theories carry over to frontier language models.
On our task suite, LLMs succeed robustly only on \TaskALL and \TaskEdgeExistence, which require only identifying the presence of a single item in a list. They show some success at small input lengths for \TaskMajority and \TaskIndexStar. LLMs struggle with any more complex task as $n$ grows. With reasoning, however, GPT-5 solves most task settings through $n = 128$, with exceptions for \TaskDyck, \TaskBinaryAddition, and \TaskIndexAnd with $k = 3$ (\Cref{app:llm-reasoning}, \Cref{fig:llm-reasoning}; failures on \TaskDyck are likely due to our use of per-position next-token-set prediction, see \Cref{app:task-dyck}).

%% file: sections/05-discussion.tex
\section{Discussion}\label{sec:discussion}

\Cref{sec:results} showed that no single theory explains the pattern of empirical observations on our suite of tasks.
When a prediction fails, it points to a feature of transformers that the theory's abstract model fails to accurately capture.
We discuss the most important of these features and highlight concrete directions for future theoretical work.

\paragraph{Models of attention: soft \emph{vs} hard, and dynamic retrieval.}
The \emph{attention} operation is the defining feature of transformers, and each theory models it differently.
We identify two properties of how attention is modeled that explain many of our results: (i) whether attention is soft or hard, and (ii) whether it supports dynamic, content-dependent retrieval.
Real transformers with soft attention can aggregate information over many positions, while hard attention is limited to selecting a fixed number of positions.
As a result, theories that build on hard attention (\Starfree, \ACz, and \BAPO) cannot perform global aggregation, so they predict the \TaskMajority task to be \Hard.
Yet real transformers can use their soft attention to solve this task robustly across input sizes.
On the other hand, while \CRASP theory uses a soft form of attention that enables global aggregation (in the form of prefix-counting), its restriction to uniform averaging prevents it from capturing the \emph{dynamic}, content-dependent information retrieval capabilities of real transformers, as \cref{sec:results} shows.
Therefore, we argue that a faithful model of attention should be both soft and content-dependent.

\paragraph{Unbounded computation and loose upper bounds.}
\TCz and \MPC theories model both global aggregation and content-dependent retrieval.
Yet, our experiments show that they are overpowered as they predict \Easy for many tasks that transformers fail to learn.
These prediction errors can be partially explained by computation resources that exceed real transformers, as well as loose upper bounds.
For example, while $\TCz$ circuits can simulate soft attention transformers, it is known that transformers do not solve every task in \TCz~\citep{chiang2023tighter}.
\MPC and \BAPO measure the communication complexity of attention, but allow for unbounded local computation.
This leads to false \Easy predictions on tasks which have modest communication requirements but are nonetheless computationally hard.
As an example, \MPC and \BAPO predict that every regular language is \Easy, including the word problem for the permutation group $S_5$, which is $\NCone$-complete~\citep{barrington1989bounded} and therefore lies outside the reach of constant-depth transformers assuming $\TCz \neq \NCone$~\citep{merrill2023parallelism}. 

\paragraph{Failures of learning, not expressivity.}
Most theoretical work addresses questions around expressivity, without addressing learnability.
Among the seven theories we investigate in this work, only \Sensitivity makes claims about what transformers can \emph{learn}.
Several of our negative results reflect failures of learning: transformers fail to fit even their training data on \TaskParity, accuracy on \TaskBinaryAddition varies widely across seeds and improves with more data, and \TaskEdgeExistence becomes learnable with denser supervision.
Expressivity theories cannot explain these failures.
Some correctly predict \TaskParity to be \Hard, but for the wrong reason, since its difficulty comes from optimization rather than the expressivity limits they identify.
However, existing learnability theory, such as \Sensitivity, tends to be narrower in scope than expressivity theory.
Broader theoretical work on learnability, where our understanding lags behind, would be particularly valuable.

\paragraph{Positional encoding \& inductive biases.}
Current theory does not provide an account of the inductive bias effects of positional encoding during learning, and how they lead to success or failure in length generalization.
Our experiments show that the ability to learn a length-generalizing solution depends strongly on the positional encoding, and that no single encoding scheme succeeds across tasks.
For example, sinusoidal encodings generalize on bounded-range \TaskIndex, relative positional encodings succeed on \TaskOneHop, and only NoPE succeeds on unbounded \TaskDyckOne.
\citet{huang2024formal} take a first step toward such an account by relating a notion of length generalization to variants of \CRASP with different types of positional information, but consider an idealized asymptotic analysis.
Our results point to an impactful direction for future theory: an account of the inductive biases that make an encoding generalize on a task could guide the design of new encodings that generalize across many tasks at once.

\paragraph{LLMs are more complex objects.}
Our experiments show that small transformer models trained from scratch exhibit markedly different behavior from pretrained large language models.
For example, LLMs fail on tasks like \TaskMajority and \TaskIndex as $n$ grows, while small transformers trained from scratch successfully learn them.
This suggests that there are different important factors at play, beyond what current expressivity theory captures, that dictate what capabilities LLMs acquire, such as tokenization, training data, the pretraining and post-training recipe, and the dynamics of simultaneously training on different tasks.
Thus, a theory of LLMs must seek to develop abstractions that capture these facets in addition to analysis of the architecture itself. 

\section{Conclusion}\label{sec:conclusion}
In this paper, we compare and contrast seven different theoretical frameworks that are aimed at explaining the capabilities and limitations of transformers.
We do so by curating a diverse suite of tasks studied in the literature, including tasks such as \TaskParity, \TaskIndex, \TaskDyckOne, and \TaskBinaryAddition, on which the theories disagree.
We fill in missing predictions with new theoretical analyses.
Importantly, our experiments show that transformer failures are not just due to one, but multiple mechanisms, including limits of expressivity, failures of learning, and the inductive biases of positional encodings.
For example, in addition to the well-studied \TaskParity task, \TaskKHop for $k \geq 2$ and unbounded \TaskDyckOne appear to be hard to learn, even though transformers can express solutions to them.
We also see that length generalization often depends on architectural choices such as positional encodings.
Finally, we find that pretrained LLMs and small transformers trained from scratch behave differently, with LLMs succeeding on tasks that small transformers fail on, and vice versa.
This suggests that the abilities learned by large models depend jointly on their architecture, data, and training recipe.
No single theory accounts for all of these mechanisms and their interactions, so none explains the full pattern of successes and failures.

\paragraph{Towards a unified theory of transformers.}
Each of the seven theories captures some features of transformers and misses others.
Theories built on hard attention capture content-dependent retrieval, but lack global aggregation; \CRASP captures global aggregation but not content-dependent retrieval.
\TCz and \MPC can represent content-dependent soft attention, but are overpowered, predicting \Easy for many tasks that transformers fail to learn.
\MPC and \BAPO formalize the bandwidth limits of attention, but assume unbounded local computation, so they classify any task with low communication complexity as \Easy, regardless of how much computation each position must perform.
Only \Sensitivity addresses learning, but it gives just one criterion for hardness and is limited to Boolean functions under the uniform distribution.
Our results identify concrete guidelines for future theories of transformers: model attention as both soft and content-dependent, bound the computation at each position, explain which expressible solutions training finds, and account for the added complexities of large language models.
Our task suite and prediction matrix provide a testbed for new theories.

%% file: sections/acknowledgments-ai-use.tex
\section*{Acknowledgments and AI Use Disclosure}

We thank the AI Interaction and Learning (AIIL) team at Microsoft Research for helpful discussions during the development of this project.

We used generative AI tools in several parts of this work.
For the experiments, we used AI coding tools to implement the model components, the training code, and the synthetic task generators.
For the theory, we used AI tools to assist in developing proofs, both by formalizing proof sketches and by suggesting proof ideas.
The final proofs in the paper were written and reviewed by the human authors using the AI-generated proofs as starting points when appropriate.
We also used AI tools to search for and summarize relevant prior work.
In preparing the paper, AI tools were used to edit the text for readability, format the LaTeX source, write the code that produces our results figures, and draw the conceptual diagrams in TikZ.
We take full responsibility for all content in this paper.

%% file: appendices/task-details/task-details.tex
\section{Task Suite Details}
\label{app:task_details}

We separate the mathematical functions and languages that describe tasks (\cref{app:task-definitions}) from the experimental sampling, tokenization, and supervision choices (\cref{app:task-experiments}).

\subsection{Mathematical task definitions}
\label{app:task-definitions}

We write $[n]=\{1,\ldots,n\}$ and use one-based positions in the mathematical definitions.
Unless stated otherwise, $n\geq1$, data are binary, and task parameters are fixed independently of $n$.

\begin{definition}[Boolean tasks]
\label{def:task-boolean}
For $x=(x_1,\ldots,x_n)\in\{0,1\}^n$,
\begin{align*}
    \TaskALL(x)&=\bigwedge_{i=1}^n x_i,
    &\TaskANY(x)&=\bigvee_{i=1}^n x_i,\\
    \TaskMajority(x)&=\bmone\left\{\sum_{i=1}^n x_i\geq n/2\right\},
    &\TaskParity(x)&=\left(\sum_{i=1}^n x_i\right)\bmod 2.
\end{align*}
\end{definition}

\begin{definition}[\TaskIndex and \TaskIndexStar]
\label{def:task-index}
For data \(x\in\{0, 1\}^n\) and query \(i\in[n]\),
\[
    \text{\TaskIndex}(x, i)
    =x_i.
\]
Given \(k\geq1\) and \(F:\{0,1\}^k\to\{0,1\}\), with queries \(i_1, \dots, i_k \in [n]\), 
\[
    \text{\TaskIndexStar}_{F, k}(x, i_1, \dots, i_k)
    =F(x_{i_1},\ldots,x_{i_k}).
\]
Taking \(F\) to be conjunction, disjunction, or parity gives \TaskIndexAnd, \TaskIndexOr, or \TaskIndexXor, respectively.

We occasionally consider \emph{fixed-range} \TaskIndex$_L$, where we fix an index bound $L \ge 1$ and have $i \le L$ for each queried index $i$. When needed for clarity, we refer to the unbounded-index version as \emph{growing-range} index. 
Similarly, \(\text{\TaskIndexStar}_{L,F,k}\) denotes the fixed-range version of \(\text{\TaskIndexStar}_{F,k}\), with every query bounded by \(L\).
\end{definition}

\begin{definition}[\TaskSparseStar]
\label{def:task-sparse}
Fix a tuple of positive indices $I=(i_1,\ldots,i_k)$ and a Boolean function $F:\{0,1\}^k\to\{0,1\}$, with $k\geq1$, independently of the input length $n$. Repeated indices are allowed.
For $x\in\{0,1\}^n$,
\[
    \text{\TaskSparseStar}_{I,F}(x)=
    \begin{cases}
        F(x_{i_1},\ldots,x_{i_k}) & n\geq\max_j i_j,\\
        0 & n<\max_j i_j.
    \end{cases}
\]
Conjunction, disjunction, and parity give \TaskSparseAnd, \TaskSparseOr, and \TaskSparseXor.
\end{definition}

\begin{definition}[\TaskBinaryAddition]
\label{def:task-addition}
For a binary string $x=x_1\cdots x_m$ written most-significant bit first, let
$\operatorname{val}_2(x)=\sum_{i=1}^m x_i2^{m-i}$.
The decision version of \TaskBinaryAddition checks whether two $n$-bit numbers sum to a proposed $(n+1)$-bit result.
For $a,b\in\{0,1\}^n$ and $c\in\{0,1\}^{n+1}$, 
\[
\TaskBinaryAddition(a,b,c)
=\mathbf{1}\{\operatorname{val}_2(a)+\operatorname{val}_2(b)=\operatorname{val}_2(c)\}.
\]
\end{definition}

\begin{definition}[\TaskDyck]
\label{def:task-dyck}
Fix $k\geq1$ bracket types, with alphabet
$\Sigma_k=\{\mathsf{open}_t,\mathsf{close}_t:t\in[k]\}$.
The \TaskDyck-$k$ language is generated by the context-free grammar with start symbol $S$ and productions
\[
    S\ \to\ \epsilon \mid SS \mid \mathsf{open}_t\,S\,\mathsf{close}_t
    \qquad (t\in[k]).
\]
For $w\in\text{\TaskDyck-$k$}$, its nesting depth $\operatorname{depth}(w)$ is the maximum, over all prefixes of $w$, of the number of opening brackets minus the number of closing brackets.
For a fixed depth bound $D\geq1$, define
\[
    \TaskDyck\text{-}(k,D)
    =\{w\in\TaskDyck\text{-}k:
      \operatorname{depth}(w)\leq D\}.
\]
\end{definition}

\begin{definition}[\TaskKHop]
\label{def:task-k-hop}
Fix a finite alphabet $\Sigma$ and $k\geq1$.
For $x=x_1\cdots x_n\in\Sigma^n$, define
\[
    h(i)=1+\max\{j<i:x_j=x_i\},
\]
leaving $h(i)$ undefined when there is no previous occurrence.
The output of \TaskKHop is $x_{h^k(n)}$ if all $k$ hops are defined, and a null symbol $\bot\notin\Sigma$ otherwise.
Each hop moves to the position immediately after the most recent earlier occurrence of the current symbol.
Adjacent repeats are allowed and can give $h(i)=i$.
\end{definition}

\begin{definition}[Graph problems]
\label{def:task-graphs}
An instance consists of a directed graph $G=([n],E)$ and query vertices $s,t\in[n]$.
\TaskEdgeExistence returns $\bmone\{(s,t)\in E\}$.
\TaskDirectedReachability returns $1$ exactly when there is a directed path from $s$ to $t$.
\TaskUndirectedReachability returns $1$ exactly when there is a path from $s$ to $t$ after edge directions are ignored.
\end{definition}

\subsection{Experimental tasks}
\label{app:task-experiments}

This appendix describes how we sample and encode tasks for the experiments.
\Cref{tab:task-settings} lists every task setting and the experiments that use it.
The text format used for pretrained LLMs is described in \cref{app:llm-eval-details}.

Throughout, $n$ is the size of an instance: the number of input bits or symbols, the number of bits per operand for \TaskBinaryAddition, the number of bracket pairs for \TaskDyck, and the number of nodes for the graph tasks.
Positions are numbered from $0$.
All tasks are supervised with a cross-entropy loss at the scored positions only.
Unless stated otherwise, a single position is scored: the last token of the input, where the model predicts the answer.

\begin{table}[h]
\centering
\small
\caption{\textbf{Task settings.} Each row is one or more task settings. ``Tokens'' is the input length for models trained from scratch at size $n$. A check mark means the setting is used in the in-distribution (ID), length-generalization (LG), or pretrained-LLM (LLM) experiment. In total, ID uses 34 settings, LG uses 34, and LLM uses 27.}
\label{tab:task-settings}
\renewcommand{\arraystretch}{1.15}
\setlength{\tabcolsep}{4pt}
\begin{tabularx}{\linewidth}{@{}>{\raggedright\arraybackslash}p{0.36\linewidth} >{\raggedright\arraybackslash}X l c c c@{}}
    \toprule
    \textbf{Task} & \textbf{Settings} & \textbf{Tokens} & \textbf{ID} & \textbf{LG} & \textbf{LLM} \\
    \midrule
    \TaskALL, \TaskANY, \TaskMajority, \TaskParity & one each & $n$ & \checkmark & \checkmark & \checkmark \\
    \TaskIndex & 8 symbols & $n + 1$ & \checkmark & \checkmark & \checkmark \\
    \TaskIndexAnd, \TaskIndexOr, \TaskIndexXor & $k \in \{2, 3\}$ & $n + k + 1$ & \checkmark & \checkmark & \checkmark \\
    \TaskSparseAnd, \TaskSparseOr, \TaskSparseXor & $k \in \{2, 3\}$ & $n$ & \checkmark & \checkmark & -- \\
    \TaskBinaryAddition & blocked or interleaved input; MSB- or LSB-first output & $3n + 1$ & \checkmark & \checkmark & \checkmark \\
    \TaskDyckOne, \TaskDyckTwo & depth bound $D \in \{4, 8, \infty\}$ & $2n + 2$ & \checkmark & \checkmark & \checkmark \\
    \TaskKHop & $k \in \{1, 2, 3\}$ & $n$ & \checkmark & \checkmark & \checkmark \\
    \TaskEdgeExistence & sparse supervision (final answer only) & $3n + 4$ & \checkmark & \checkmark & \checkmark \\
    \TaskEdgeExistence & dense supervision (every edge) & $3n + 4$ & \checkmark & \checkmark & -- \\
    \TaskDirectedReachability, \TaskUndirectedReachability & one each & $3n + 4$ & \checkmark & \checkmark & \checkmark \\
    \bottomrule
\end{tabularx}
\end{table}

\subsubsection{Boolean tasks}
\label{app:task-boolean}

For the experiments,
bits are drawn i.i.d.\ with a probability chosen so that both labels are equally likely.
For \TaskALL, $P(x_i = 1) = 2^{-1/n}$; for \TaskANY, $P(x_i = 1) = 1 - 2^{-1/n}$; for \TaskParity, $P(x_i = 1) = 1/2$.
For \TaskMajority, bits are uniform and the label $1$ is slightly more likely due to ties (\cref{tab:chance-levels}).

\subsubsection{Indexed retrieval}
\label{app:task-index}

\paragraph{\TaskIndex.}
The input is a sequence of $n$ symbols $s_0, \ldots, s_{n-1}$, drawn i.i.d.\ and uniformly from an alphabet of $8$ symbols, followed by one address token $\langle i \rangle$.
The answer is $s_i$.
The address $i$ is uniform over $\{0, \ldots, \min(n, R) - 1\}$, where $R$ is the index range.
Each address is an atomic token, so the vocabulary has one token for every allowed address.

\paragraph{\TaskIndexAnd, \TaskIndexOr, \TaskIndexXor.}
The input is a bit string $x \in \{0,1\}^n$, a separator token, and $k$ address tokens $\langle i_1 \rangle, \ldots, \langle i_k \rangle$.
The addresses are distinct and drawn uniformly without replacement from $\{0, \ldots, \min(n, R) - 1\}$.
The answer is the AND, OR, or XOR of $x_{i_1}, \ldots, x_{i_k}$.
Bits are drawn i.i.d.\ so that both labels are equally likely given $k$: $P(x_i = 1) = 2^{-1/k}$ for AND, $1 - 2^{-1/k}$ for OR, and $1/2$ for XOR.

\paragraph{Index range.}
In the in-distribution experiment the index range grows with the input, $R = n$, so every position can be queried and the address vocabulary grows with $n$.
In the length-generalization experiment the range is fixed at $R = 20$, so test inputs longer than the training inputs introduce no new address tokens (see \cref{app:prediction-conventions}).

\subsubsection{Sparse Boolean functions}
\label{app:task-sparse}

The input is a bit string $x \in \{0,1\}^n$ with no other tokens.
A set of $k$ hidden positions $i_1 < \cdots < i_k$ is fixed for each task setting and never appears in the input.
The answer of \TaskSparseAnd, \TaskSparseOr, and \TaskSparseXor is the AND, OR, or XOR of $x_{i_1}, \ldots, x_{i_k}$, and the model must infer the hidden positions from the training data.
Bits are drawn as for the indexed functions.
In the in-distribution experiment the hidden positions are drawn once, uniformly at random, from $\{0, \ldots, n-1\}$, separately for each $n$ (for example, $(6, 11)$ at $n = 16$ and $(58, 251)$ at $n = 256$ for $k = 2$).
In the length-generalization experiment they are drawn from $\{0, \ldots, 19\}$ and shared by all lengths: $(6, 11)$ for $k = 2$ and $(6, 11, 17)$ for $k = 3$.
The same hidden positions are used for AND, OR, and XOR.

\subsubsection{Binary addition}
\label{app:task-binary-addition}

The input is two $n$-bit unsigned operands $a$ and $b$, drawn independently and uniformly from $\{0, \ldots, 2^n - 1\}$ with leading zeros kept.
The answer is the $(n+1)$-bit sum $a + b$.
Operand bits use role-specific tokens ($\texttt{A0}, \texttt{A1}, \texttt{B0}, \texttt{B1}$) and are written most-significant bit first.
In the \emph{blocked} layout all bits of $a$ precede all bits of $b$; in the \emph{interleaved} layout the bits alternate, $a_{n-1} b_{n-1} a_{n-2} b_{n-2} \cdots$.
An answer-marker token follows the operands.
The model then predicts the $n+1$ result bits autoregressively, either most-significant bit first or least-significant bit first.
During training the result bits are teacher-forced; the input therefore has $3n+1$ tokens and $n+1$ scored positions.
At evaluation the model generates all $n+1$ bits greedily, and an example is correct only if every bit is correct.
The main-text figures show the blocked layout with MSB-first output.

\subsubsection{Dyck languages}
\label{app:task-dyck}

\TaskDyckOne and \TaskDyckTwo use $K = 1$ and $K = 2$ bracket types, and a depth bound $D \in \{4, 8, \infty\}$, where $D = \infty$ is unbounded \TaskDyck.
An instance of size $n$ is a correctly nested string of $n$ bracket pairs, with nesting depth at most $D$, framed by $\texttt{START}$ and $\texttt{END}$ tokens ($2n+2$ tokens).
Strings are sampled uniformly among all such strings, and each pair receives an independent uniform type.
Rather than the theoretical membership problem of \cref{def:task-dyck}, we run experiments on the next-token formulation of \citet{yao2021self}.
Membership as a classification task requires a distribution of negative examples, and it is difficult to design one that rules out shortcuts and truly tests whether the model has learned the nesting structure.
In the next-token formulation, only a positive distribution is needed, and the task is to predict, after $\texttt{START}$ and after every bracket, the set of tokens that can legally come next.
A model that predicts every legal set correctly also decides membership, since a string is in the language exactly when each of its tokens is legal at its position.
The legal set contains every opening bracket if the current depth is below $D$, the closing bracket that matches the most recent unclosed opening bracket if the depth is positive, and $\texttt{END}$ if the depth is zero.
The target at each of the $2n+1$ scored positions is a binary vector over the $2K$ brackets and $\texttt{END}$, trained with binary cross-entropy.
An example is correct only if the predicted set is exactly right at every scored position.

\subsubsection{Iterated retrieval: \texorpdfstring{\TaskKHop}{k-Hop Induction}}
\label{app:task-k-hop}

The input is a sequence $s_1, \ldots, s_n$ over an alphabet of $4$ symbols.
The first symbol is uniform, and each later symbol is uniform over the $3$ symbols that differ from its predecessor, following \citet{sanford2024parallel}.
One hop from position $p$ finds the most recent earlier occurrence of the current symbol, $i = \max\{\, j < p : s_j = s_p \,\}$, and moves to $p = i + 1$.
Starting at $p = n$, \TaskKHop applies $k$ hops and returns $s_p$.
If some hop finds no earlier occurrence, the answer is a special null symbol, so the output has $5$ classes.
We use $k \in \{1, 2, 3\}$ with the alphabet fixed for all $n$; $k = 1$ is \TaskOneHop.

\subsubsection{Graph tasks}
\label{app:task-graphs}

The three graph tasks use the same graph distribution and input format and differ only in the question.
An instance of size $n$ is a simple directed graph with $n$ nodes and $n$ edges, drawn uniformly.
Node names are atomic tokens, drawn uniformly without replacement from a pool of $256$ names, so larger graphs introduce no new tokens.
Each edge $(u, v)$ is written as the three tokens $\texttt{<EDGE>}\ u\ v$.
The edge list is followed by $\texttt{<QUERY>}\ s\ t$ and an $\texttt{<ANSWER>}$ token, which gives $3n+4$ tokens.
The answer is binary:
\TaskEdgeExistence asks whether $(s, t)$ is a listed edge;
\TaskDirectedReachability asks whether there is a directed path from $s$ to $t$;
\TaskUndirectedReachability asks whether $s$ and $t$ are connected when edge directions are ignored.
From each graph we draw one positive and one negative query: positives uniformly among ordered pairs $(s, t)$, $s \neq t$, with a positive answer, and negatives uniformly among pairs with a negative answer.
Graphs that do not admit both a positive and a negative query are rejected and redrawn.
For \TaskEdgeExistence, negative queries include reversed edges, so the model must read edge direction.

\paragraph{Sparse and dense supervision for \TaskEdgeExistence.}
By default only the final answer is labeled; we call this \emph{sparse} supervision.
In the \emph{dense} variant the query is placed before the edge list ($\texttt{<QUERY>}\ s\ t$, then the edges, then $\texttt{<ANSWER>}$), so that a causal model sees the query while reading each edge.
In addition to the final answer, the last token of every edge is labeled with whether that edge equals the query.
We report accuracy on the final answer only, so the metric is the same as for final-answer supervision.
This variant changes both the supervision and the position of the query; \cref{app:edge-existence-supervision} discusses its results.

%% file: appendices/experimental-details/experimental-details.tex
\section{Experimental Details}
\label{app:experimental_details}

This appendix describes the models, training, and evaluation of the three experiments in \cref{ssec:experiment_methodology}.
\Cref{app:normalized-accuracy} defines normalized accuracy and the chance level of each task.

\subsection{Model architecture and training}
\label{app:architecture-training}

All models trained from scratch are decoder-style transformers with pre-layer normalization, softmax attention, and a two-layer feedforward block of width $4d$ with GELU activation, where $d$ is the model width.
Attention and feedforward layers have biases, and there is no dropout.
Input and output embeddings are tied, except for the graph tasks and \TaskBinaryAddition, which use a separate two-class output layer.
All models are trained with AdamW, a constant learning rate of $10^{-3}$ with no warmup, no weight decay, gradient-norm clipping at $1.0$, batch size $256$, and bfloat16 mixed precision.
\Cref{tab:training-hyperparameters} lists the settings that differ between the in-distribution and length-generalization experiments.

\begin{table}[h]
\centering
\small
\caption{\textbf{Training settings of the two from-scratch experiments.} Settings shared by both experiments are given in the text.}
\label{tab:training-hyperparameters}
\renewcommand{\arraystretch}{1.15}
\begin{tabularx}{\linewidth}{@{}l >{\raggedright\arraybackslash}X >{\raggedright\arraybackslash}X@{}}
    \toprule
     & \textbf{In-distribution} & \textbf{Length generalization} \\
    \midrule
    Model sizes (layers, width, heads) & $(4, 256, 4)$ & S $(2, 64, 2)$, M $(4, 128, 4)$, L $(8, 256, 8)$ \\
    Positional encoding & RoPE & 9 encodings (\cref{tab:positional-encodings}), plus uniform attention \\
    Attention mask & causal & causal; also bidirectional where the task allows \\
    Training sizes $n$ & one of $8, 16, 32, 64, 128, 256$ & $8, 16, 20, 24, 32$ jointly \\
    Evaluation sizes $n$ & the training size & $8, 16, 20, 24, 32$, and $48$ to $256$ in steps of $16$ \\
    Training examples $N$ & $2^8, 2^{10}, 2^{12}, 2^{14}, 2^{16}$ & $8{,}192$, split evenly over training sizes \\
    Optimizer steps & $2{,}048$ (epochs $= 2^{19} / N$) & $1{,}120$ (32 epochs) \\
    Evaluation examples & $8{,}192$ & $1{,}024$ per evaluation size \\
    Seeds & 5 & 5 \\
    Runs & $5{,}100$ & $8{,}550$ \\
    \bottomrule
\end{tabularx}
\end{table}

\input{appendices/experimental-details/in-distribution.tex}
\input{appendices/experimental-details/length-generalization.tex}
\input{appendices/experimental-details/llm-eval.tex}
\input{appendices/experimental-details/normalized-accuracy.tex}
\input{appendices/experimental-details/llm-prompts.tex}

%% file: appendices/experimental-details/in-distribution.tex
\subsection{In-distribution learning}
\label{app:in-distribution-details}

\paragraph{Design.}
For each task setting and input size $n \in \{8, 16, 32, 64, 128, 256\}$, we train on examples of size $n$ and evaluate on fresh examples of the same size.
The model is a causal transformer with 4 layers, width 256, 4 heads, and rotary positional encoding (base $10^4$).
We vary the number of distinct training examples $N \in \{2^8, 2^{10}, 2^{12}, 2^{14}, 2^{16}\}$ and repeat each configuration with 5 seeds.
Every run takes $2{,}048$ optimizer steps with batch size 256, that is, $2^{19}$ example presentations.
A run with $N$ training examples therefore makes $2^{19}/N$ passes over its data, from 8 epochs at $N = 2^{16}$ to $2{,}048$ epochs at $N = 2^8$.
This holds the training compute fixed, so differences across $N$ reflect the number of distinct examples rather than the number of gradient steps.
The experiment covers all 34 task settings of \cref{tab:task-settings}.

\paragraph{Data.}
All runs at a given $(\text{setting}, n)$ are evaluated on the same $8{,}192$ held-out examples.
Each seed has its own training set of $2^{16}$ examples, disjoint from those of other seeds, and the smaller training sets are nested subsets of it.
Comparisons across $N$ within a seed are therefore paired, and seeds use independent training data.
The seed also sets the model initialization and the order of training examples.
Held-out examples are distinct draws, but for small $n$ the input space is small, so the same input can appear in both the training and held-out sets (for example, there are only $2^8$ bit strings at $n = 8$).

\paragraph{Evaluation and grading.}
At the end of training we evaluate on the $8{,}192$ held-out examples and on $1{,}024$ of the training examples.
The metric is full-sequence accuracy, normalized by the chance level of the task (\cref{app:normalized-accuracy}).
A task setting is solved at size $n$ if, for some $N$, at least one of the five seeds reaches $90\%$ normalized accuracy on the held-out examples.
The setting is graded \Easy if it is solved at the largest size, $n = 256$, and \Hard otherwise.

%% file: appendices/experimental-details/length-generalization.tex
\subsection{Length generalization}
\label{app:length-generalization-details}

\paragraph{Design.}
Each run trains one model jointly on sizes $n \in \{8, 16, 20, 24, 32\}$ and evaluates it at 19 sizes: the five training sizes and every multiple of 16 from 48 to 256.
The training set has $8{,}192$ examples, split evenly across the five sizes, and each batch contains a single size.
Training takes 32 epochs, or $1{,}120$ optimizer steps.
Each evaluation size uses $1{,}024$ fresh examples.
The seed sets the model initialization.
For a given task setting and seed, every architecture uses the same training data, data order, and evaluation data, so comparisons across architectures are paired.

\paragraph{Architectures.}
We cross three model sizes (\cref{tab:training-hyperparameters}) with the nine positional encodings of \cref{tab:positional-encodings}.
We also train an attention ablation with uniform attention, where queries and keys are set to zero so that every head attends uniformly over its allowed positions, with a BOS token prepended.
This matches the attention model of \CRASP, and we use it only for the comparison in \cref{fig:app-lg-uniform}.
Every task setting is trained with a causal mask.
Task settings with a single answer at the last position are also trained with a bidirectional mask.
\TaskDyck, \TaskBinaryAddition, and \TaskEdgeExistence with dense supervision score positions before the end of the input, so they use the causal mask only.
This gives 57 (task setting, mask) pairs, each trained with three model sizes and ten attention variants (the nine positional encodings and uniform attention).

\begin{table}[h]
\centering
\small
\caption{\textbf{Positional encodings in the length-generalization experiment.}}
\label{tab:positional-encodings}
\renewcommand{\arraystretch}{1.15}
\begin{tabularx}{\linewidth}{@{}l >{\raggedright\arraybackslash}X@{}}
    \toprule
    \textbf{Encoding} & \textbf{Position information} \\
    \midrule
    NoPE & None~\citep{kazemnejad2023impact} \\
    NoPE + BOS & None; a learned beginning-of-sequence token is prepended \\
    Sinusoidal & Fixed sinusoidal encoding added to the input embeddings~\citep{vaswani2017attention} \\
    RoPE & Rotary encoding of queries and keys, base $10^4$~\citep{su2024roformer} \\
    YaRN & RoPE with YaRN dynamic scaling; the training length is the longest training input~\citep{peng2024yarn} \\
    T5 bias & Learned relative attention bias, 32 buckets, maximum distance 128~\citep{raffel2020exploring} \\
    Shaw & Learned relative position embeddings, distances clipped at 16~\citep{shaw2018self} \\
    DeBERTa & Disentangled relative attention (content-to-position and position-to-content), 256 buckets~\citep{he2021deberta} \\
    ALiBi & Fixed linear attention bias~\citep{press2022train} \\
    \bottomrule
\end{tabularx}
\end{table}

\paragraph{Index range and hidden positions.}
For the indexed tasks, query addresses are restricted to positions $0$ to $19$, and for the sparse tasks the hidden positions lie in the same range (\cref{app:task-index,app:task-sparse}).
All address tokens and hidden positions are therefore seen during training.
Longer test inputs add unqueried positions but no new addresses, so the experiment tests generalization to longer inputs rather than to unseen address tokens.
Because the index range does not grow with $n$, the indexed settings are a different task from the in-distribution ones, and we do not compare them with theory predictions in \cref{tab:theory-task-predictions}.

\paragraph{Grading.}
The metric is normalized full-sequence accuracy, except for \TaskEdgeExistence, where we score the final answer only (\cref{app:task-graphs}).
A seed succeeds at a size if it reaches $90\%$ normalized accuracy there.
A seed is credited only if it succeeds at all five training sizes.
Its horizon is then the largest evaluation size up to which it succeeds at every size, starting from 48.
A configuration (model size, positional encoding, mask) takes the largest horizon among its five seeds, so the credited horizon is always reached by one trained model.
The main search covers the 27 causal configurations with the nine positional encodings.
A task setting is solved at size $n$ if some configuration in the main search has horizon at least $n$, and it is graded \Easy if it is solved at $n = 256$.
For each task, panel (b) of \cref{fig:results} shows the configuration with the largest horizon, with ties broken by the median out-of-distribution accuracy across seeds, then the mean, then in-distribution accuracy.

%% file: appendices/experimental-details/llm-eval.tex
\subsection{Pretrained language models}
\label{app:llm-eval-details}

\paragraph{Models.}
\Cref{tab:llm-models} lists the evaluated models.
We disable reasoning in the main evaluation, using each provider's control for this.
GPT-4o mini does not reason.
For GPT-5, we additionally evaluate every task setting with medium reasoning effort (\cref{app:llm-reasoning}).
We use each provider's default sampling settings and request a structured JSON answer.

\begin{table}[h]
\centering
\small
\caption{\textbf{Evaluated language models.} Reasoning is disabled for all models in the main evaluation.}
\label{tab:llm-models}
\renewcommand{\arraystretch}{1.15}
\begin{tabular}{@{}l l l@{}}
    \toprule
    \textbf{Model} & \textbf{Model ID} & \textbf{Reasoning control} \\
    \midrule
    GPT-5 & \texttt{gpt-5-2025-08-07} & effort \texttt{none}; also \texttt{medium} \\
    GPT-5.4 mini & \texttt{gpt-5.4-mini-2026-03-17} & effort \texttt{none} \\
    GPT-4o mini & \texttt{gpt-4o-mini-2024-07-18} & not applicable \\
    DeepSeek-V3.2 & \texttt{DeepSeek-V3.2} & effort \texttt{none} \\
    Gemini 2.5 Flash & \texttt{gemini-2.5-flash} & thinking budget 0 \\
    \bottomrule
\end{tabular}
\end{table}

\paragraph{Tasks and sizes.}
We evaluate the 27 task settings of \cref{tab:task-settings} marked for LLMs.
The sparse tasks are excluded because the hidden positions must be inferred from many labeled examples, which does not fit a few-shot prompt.
Each setting is evaluated at $n \in \{8, 16, 24, 32, 48, 64, 96, 128\}$ with 128 examples per size.
All models see the same prompts.
For \TaskBinaryAddition, the model returns all $n+1$ result bits at once, in the requested order.
For the graph tasks, the edge list precedes the query.

\paragraph{Prompts.}
The system message is:
``You solve precisely specified formal tasks over symbol sequences. Follow the task definition exactly and return only the requested structured answer.''
The user message contains a natural-language definition of the task, a sentence describing the answer format, three solved examples of size 12, and the test input.
The solved examples are the same for every test size.
Inputs are written as text: bits and symbols separated by spaces, symbols as letters, Dyck brackets as \texttt{(\_1} and \texttt{)\_1}, graph nodes as \texttt{v}$i$ with one edge per line, and binary operands as bit strings.
\Cref{app:llm-prompts} shows one prompt for every task.

\paragraph{Scoring and failures.}
An answer is correct only if it matches the target exactly: the single label for most tasks, all $n+1$ bits for \TaskBinaryAddition, and the legal-next-token set at every prefix for \TaskDyck.
Refusals and responses that are not valid JSON of the requested form are counted as incorrect.
Responses that exhaust the completion-token budget are counted as incorrect.
The completion-token budget is at least $4{,}096$ tokens and several times the length of a correct answer, with additional tokens when reasoning is enabled.
Without reasoning, truncation occurs only on long answers (\TaskDyck and \TaskBinaryAddition) at sizes where accuracy is already near zero.
With reasoning, occurs mainly on \TaskBinaryAddition (\cref{app:llm-reasoning}).
Whether truncations are treated as incorrect predictions or filtered out does not affect any \Easy/\Hard grade.

Accuracy is normalized with the chance levels of \cref{tab:chance-levels}.

\paragraph{Grading.}
Each model is graded on its own.
A model solves a task setting at size $n$ if its normalized accuracy is at least $90\%$, and a setting is graded \Easy for that model if it is solved at the largest evaluated size, $n = 128$.
A setting receives the best grade reached by any single model.
Gemini 2.5 Flash is not evaluated on \TaskBinaryAddition or on bounded \TaskDyckTwo ($D = 4$ and $D = 8$).
For the two of these settings shown in panel (c) of \cref{fig:results}, the panel averages the other four models; all other rows average five models.

%% file: appendices/experimental-details/normalized-accuracy.tex
\subsection{Normalized accuracy}
\label{app:normalized-accuracy}

Tasks differ in how often a constant answer is correct, so raw accuracies are not comparable across tasks.
We therefore report normalized accuracy $\tilde{a} = (a - b)/(1 - b)$, where $a$ is accuracy and $b$ is the chance level of the task, so that chance maps to $0$ and perfect accuracy maps to $1$.
The chance level $b(\text{task}, n)$ is the accuracy of the best constant answer at input size $n$: the probability of the most likely answer under the input distribution of the task.
A model that ignores its input therefore scores at most $0$.
Here $a$ is full-sequence accuracy: an example counts as correct only if every scored output is correct, which for single-answer tasks is the accuracy of that answer.
\Cref{tab:chance-levels} lists the chance level of each task.
For most tasks it is a constant, because the input distribution makes all answers equally likely.
For \TaskMajority and \TaskKHop it depends on $n$.
A setting counts as solved at a given input size when its normalized accuracy reaches $90\%$, and each experiment combines this threshold with its own rules over seeds and lengths to assign a grade.
For \TaskKHop, the input distribution does not change when the symbols are relabeled, so all symbol answers are equally likely, and the best constant answer is either the null answer or any single symbol.
We estimate the probability of the null answer from $200{,}000$ samples per $(k, n)$.
For tasks whose answer is a sequence, \TaskDyck and \TaskBinaryAddition, we report full-sequence accuracy without normalization ($b = 0$).
Here the best constant answer is a single fixed answer sequence, and models that fail generally score near $0$ rather than near its accuracy.
One exception is unbounded \TaskDyckOne in the in-distribution experiment: the best constant answer (the depth returns to $0$ only at the end of the string) has accuracy $\tfrac{n+1}{2(2n-1)} \approx 0.25$, and models that fail at $n \geq 128$ score $0.2$ to $0.33$, close to it.

\begin{table}[t]
\centering
\small
\caption{Chance level $b(\text{task}, n)$: the accuracy of the best constant answer at input size $n$. Tasks with sequence answers use full-sequence accuracy without normalization. The same values are used in all three experiments.}
\label{tab:chance-levels}
\renewcommand{\arraystretch}{1.2}
\setlength{\tabcolsep}{4pt}
\begin{tabularx}{\linewidth}{@{}>{\raggedright\arraybackslash}p{0.27\linewidth} >{\raggedright\arraybackslash}p{0.22\linewidth} >{\raggedright\arraybackslash}X@{}}
    \toprule
    \textbf{Tasks} & \textbf{Chance level $b$} & \textbf{Reason} \\
    \midrule
    \TaskALL, \TaskANY, \TaskParity, \TaskIndexAndOr, \TaskIndexXor, \TaskSparseAndOr, \TaskSparseXor & $1/2$ & Bits are drawn with a per-bit probability that makes both labels equally likely. \\
    \TaskEdgeExistence, \TaskDirectedReachability, \TaskUndirectedReachability & $1/2$ & Each graph yields one positive and one negative query. With dense supervision, the all-negative answer is correct exactly on negative queries. \\
    \TaskMajority & $\tfrac{1}{2} + \binom{n}{n/2} 2^{-(n+1)}$ ($n$ even) & Bits are uniform and ties are labeled $1$, so label $1$ is more likely: $0.64$ at $n = 8$, $0.52$ at $n = 256$. \\
    \TaskIndex & $1/8$ & The queried symbol is uniform over $8$ values. \\
    \TaskKHop & $\max\bigl(p_{\varnothing}, \tfrac{1 - p_{\varnothing}}{4}\bigr)$ & $4$ equally likely symbols plus the null answer, which has probability $p_{\varnothing}(k, n)$. For $k = 3$, $b = 0.68$ at $n = 8$; for $n \geq 32$, $b \approx 1/4$ for all $k$. \\
    \TaskDyckOne, \TaskDyckTwo, \TaskBinaryAddition & $0$ (full sequence) & Sequence answers: the legal next-token set after each prefix, or all $n{+}1$ bits of the sum. We do not normalize. \\
    \bottomrule
\end{tabularx}
\end{table}

%% file: appendices/experimental-details/llm-prompts.tex
\clearpage
\subsection{LLM prompts}
\label{app:llm-prompts}

This section shows the user message of one evaluated prompt for each task, followed by the expected answer; the system message is given in \cref{app:llm-eval-details}.
Each prompt is the first test input of size $n = 8$ in the evaluation.
The evaluated prompts contain three solved examples of size 12; to save space we show only the first.
For tasks with several settings we show one; the other settings change only the parameters in the task definition (the number of queried positions $k$, the number of bracket types, the depth bound, the number of hops, or the order of the output bits).
Long lines are wrapped for display.

\input{appendices/experimental-details/generated/llm-prompts-body.tex}

%% file: appendices/experimental-details/generated/llm-prompts-body.tex
\paragraph{\TaskALL.}
{\scriptsize
\begin{verbatim}
You are given a string of bits. Each bit is 0 or 1, and the bits are separated by spaces.
Compute the AND of every bit: the answer is 1 if every bit in the string is 1, and 0
otherwise.

Respond with a JSON object holding a single field "answer", and nothing else.

Example 1:
Bits: 1 1 0 1 0 1 1 1 1 1 1 1
Answer: {"answer": 0}

Now solve this instance.
Bits: 1 1 1 1 1 1 1 1
\end{verbatim}
}
\noindent Expected answer:
{\scriptsize
\begin{verbatim}
{"answer": 1}
\end{verbatim}
}

\paragraph{\TaskANY.}
{\scriptsize
\begin{verbatim}
You are given a string of bits. Each bit is 0 or 1, and the bits are separated by spaces.
Compute the OR of every bit: the answer is 1 if at least one bit in the string is 1, and 0
if every bit is 0.

Respond with a JSON object holding a single field "answer", and nothing else.

Example 1:
Bits: 1 0 0 0 0 0 0 0 0 0 0 0
Answer: {"answer": 1}

Now solve this instance.
Bits: 0 0 0 0 0 0 0 0
\end{verbatim}
}
\noindent Expected answer:
{\scriptsize
\begin{verbatim}
{"answer": 0}
\end{verbatim}
}

\paragraph{\TaskMajority.}
{\scriptsize
\begin{verbatim}
You are given a string of bits. Each bit is 0 or 1, and the bits are separated by spaces.
Compute the majority bit: the answer is 1 if at least half of the bits in the string are 1,
and 0 otherwise. An exact tie counts as 1.

Respond with a JSON object holding a single field "answer", and nothing else.

Example 1:
Bits: 1 1 0 0 0 1 0 0 1 1 1 1
Answer: {"answer": 1}

Now solve this instance.
Bits: 1 0 1 1 1 0 1 0
\end{verbatim}
}
\noindent Expected answer:
{\scriptsize
\begin{verbatim}
{"answer": 1}
\end{verbatim}
}

\paragraph{\TaskParity.}
{\scriptsize
\begin{verbatim}
You are given a string of bits. Each bit is 0 or 1, and the bits are separated by spaces.
Compute the parity of the string: the XOR of every bit. The parity is 1 if the number of 1s
in the string is odd, and 0 if it is even. Report that parity bit.

Respond with a JSON object holding a single field "answer", and nothing else.

Example 1:
Bits: 1 1 0 0 0 1 0 0 1 1 1 1
Answer: {"answer": 1}

Now solve this instance.
Bits: 1 0 1 1 1 0 1 0
\end{verbatim}
}
\noindent Expected answer:
{\scriptsize
\begin{verbatim}
{"answer": 1}
\end{verbatim}
}

\paragraph{\TaskIndex.}
{\scriptsize
\begin{verbatim}
You are given a sequence of symbols separated by spaces, drawn from the alphabet a, b, c, d,
e, f, g, h. Positions are numbered from 0, so the first symbol is at position 0, the second
at position 1, and so on.
You are also given a query position. Report the symbol that sits at that position in the
sequence.

Respond with a JSON object holding a single field "answer", and nothing else.

Example 1:
Sequence: c d g f a c e f g h d d
Query position: 8
Answer: {"answer": "g"}

Now solve this instance.
Sequence: e h f a d d d h
Query position: 4
\end{verbatim}
}
\noindent Expected answer:
{\scriptsize
\begin{verbatim}
{"answer": "d"}
\end{verbatim}
}

\paragraph{\TaskIndexAnd ($k = 2$).}
{\scriptsize
\begin{verbatim}
You are given a string of bits separated by spaces. Each bit is 0 or 1. Positions are
numbered from 0, so the first bit is at position 0, the second at position 1, and so on.
You are also given 2 distinct query positions.
Compute the AND of the bits at those query positions: the answer is 1 if every one of them
is 1, and 0 otherwise.

Respond with a JSON object holding a single field "answer", and nothing else.

Example 1:
Bits: 1 1 0 1 0 1 0 1 1 1 1 1
Query positions: 9 5
Answer: {"answer": 1}

Now solve this instance.
Bits: 1 0 1 1 1 1 1 0
Query positions: 3 2
\end{verbatim}
}
\noindent Expected answer:
{\scriptsize
\begin{verbatim}
{"answer": 1}
\end{verbatim}
}

\paragraph{\TaskIndexOr ($k = 2$).}
{\scriptsize
\begin{verbatim}
You are given a string of bits separated by spaces. Each bit is 0 or 1. Positions are
numbered from 0, so the first bit is at position 0, the second at position 1, and so on.
You are also given 2 distinct query positions.
Compute the OR of the bits at those query positions: the answer is 1 if at least one of them
is 1, and 0 otherwise.

Respond with a JSON object holding a single field "answer", and nothing else.

Example 1:
Bits: 1 0 0 0 0 1 0 0 1 1 0 1
Query positions: 9 5
Answer: {"answer": 1}

Now solve this instance.
Bits: 0 0 1 1 0 0 0 0
Query positions: 3 2
\end{verbatim}
}
\noindent Expected answer:
{\scriptsize
\begin{verbatim}
{"answer": 1}
\end{verbatim}
}

\paragraph{\TaskIndexXor ($k = 2$).}
{\scriptsize
\begin{verbatim}
You are given a string of bits separated by spaces. Each bit is 0 or 1. Positions are
numbered from 0, so the first bit is at position 0, the second at position 1, and so on.
You are also given 2 distinct query positions.
Compute the XOR of the bits at those query positions: the answer is 1 if an odd number of
them are 1, and 0 otherwise.

Respond with a JSON object holding a single field "answer", and nothing else.

Example 1:
Bits: 1 1 0 0 0 1 0 0 1 1 1 1
Query positions: 9 5
Answer: {"answer": 0}

Now solve this instance.
Bits: 1 0 1 1 1 0 1 0
Query positions: 3 2
\end{verbatim}
}
\noindent Expected answer:
{\scriptsize
\begin{verbatim}
{"answer": 0}
\end{verbatim}
}

\paragraph{\TaskBinaryAddition (blocked, MSB-first output).}
{\scriptsize
\begin{verbatim}
Add two unsigned binary operands of the same width. Leading zeros are significant for the
required width. Return the sum using exactly one more bit than either operand, in msb first
order.

Respond with only a JSON object of the form {"answer": "<bits>"}. The answer must be a
string so leading zeros are preserved.

Example 1:
Operand A (msb first): 111010001010
Operand B (msb first): 011010011100
Answer: {"answer": "1010100100110"}

Now solve this instance.
Operand A (msb first): 11110110
Operand B (msb first): 01100010
\end{verbatim}
}
\noindent Expected answer:
{\scriptsize
\begin{verbatim}
{"answer": "101011000"}
\end{verbatim}
}

\paragraph{\TaskDyckTwo ($D = \infty$).}
{\scriptsize
\begin{verbatim}
This is next-token prediction in the Dyck language of balanced bracket strings, with 2
bracket types. Type 1 opens with (_1 and closes with )_1. Type 2 opens with (_2 and closes
with )_2.
A bracket string is balanced when each closing bracket matches the most recent opening
bracket that is not yet closed and is of the same type as it, and no opening bracket is left
unclosed at the end.
You are given a balanced string, written as brackets separated by spaces. For each prefix of
the string, list every token that can legally come next.
A token can legally come next when the prefix followed by that token is itself the prefix of
some balanced string. Judge each prefix on its own: use only the brackets in it, ignore the
rest of the given string, and allow the balanced string it extends to be any length. The
token "END" can come next when the prefix is already balanced, meaning the string may stop
there.

Respond with a JSON object holding a single field "answer", and nothing else. "answer" holds
one entry per prefix of the string. Entry 0 is for the empty prefix, before any bracket.
Entry i is for the prefix made of the first i brackets. The last entry is for the whole
string. So a string of m brackets has exactly m + 1 entries, and m is given with each
string. Each entry lists the tokens that can come next after that prefix, in any order.

Example 1:
String: (_1 )_1 (_2 (_2 (_1 )_1 )_2 )_2 (_1 (_1 )_1 (_1 (_2 (_1 )_1 (_1 (_2 )_2 (_2 )_2 )_1
)_2 )_1 )_1
Number of brackets: 24
Answer: {"answer": [["(_1", "(_2", "END"], ["(_1", "(_2", ")_1"], ["(_1", "(_2", "END"],
["(_1", "(_2", ")_2"], ["(_1", "(_2", ")_2"], ["(_1", "(_2", ")_1"], ["(_1", "(_2", ")_2"],
["(_1", "(_2", ")_2"], ["(_1", "(_2", "END"], ["(_1", "(_2", ")_1"], ["(_1", "(_2", ")_1"],
["(_1", "(_2", ")_1"], ["(_1", "(_2", ")_1"], ["(_1", "(_2", ")_2"], ["(_1", "(_2", ")_1"],
["(_1", "(_2", ")_2"], ["(_1", "(_2", ")_1"], ["(_1", "(_2", ")_2"], ["(_1", "(_2", ")_1"],
["(_1", "(_2", ")_2"], ["(_1", "(_2", ")_1"], ["(_1", "(_2", ")_2"], ["(_1", "(_2", ")_1"],
["(_1", "(_2", ")_1"], ["(_1", "(_2", "END"]]}

Now solve this instance.
String: (_2 (_2 (_2 )_2 )_2 (_1 (_1 )_1 )_1 (_2 )_2 )_2 (_2 (_2 )_2 )_2
Number of brackets: 16
\end{verbatim}
}
\noindent Expected answer:
{\scriptsize
\begin{verbatim}
{"answer": [["(_1", "(_2", "END"], ["(_1", "(_2", ")_2"], ["(_1", "(_2", ")_2"], ["(_1",
"(_2", ")_2"], ["(_1", "(_2", ")_2"], ["(_1", "(_2", ")_2"], ["(_1", "(_2", ")_1"], ["(_1",
"(_2", ")_1"], ["(_1", "(_2", ")_1"], ["(_1", "(_2", ")_2"], ["(_1", "(_2", ")_2"], ["(_1",
"(_2", ")_2"], ["(_1", "(_2", "END"], ["(_1", "(_2", ")_2"], ["(_1", "(_2", ")_2"], ["(_1",
"(_2", ")_2"], ["(_1", "(_2", "END"]]}
\end{verbatim}
}

\paragraph{\TaskKHop ($k = 2$).}
{\scriptsize
\begin{verbatim}
Consider the k-hop induction head task with k = 2.
You are given a string s_1 s_2 ... s_N of symbols separated by spaces, drawn from the
alphabet a, b, c, d.
Start at the last position, so p = N. Then do the following 2 times:
  1. Let x = s_p, the symbol at the current position p.
  2. Let i be the largest index with i < p and s_i = x. This is the most recent occurrence
of x strictly before the current position. If there is no such i, stop and return "none".
  3. Set p = i + 1.
After the last repetition, return s_p.
In words: each step looks back for where the current symbol last occurred, then moves to the
symbol immediately after that occurrence.

Respond with a JSON object holding a single field "answer", and nothing else.

Example 1:
Sequence: c d c d b d c d b a c a
Answer: {"answer": "d"}

Now solve this instance.
Sequence: a c a c a c a d
\end{verbatim}
}
\noindent Expected answer:
{\scriptsize
\begin{verbatim}
{"answer": "none"}
\end{verbatim}
}

\paragraph{\TaskEdgeExistence.}
{\scriptsize
\begin{verbatim}
You are given a finite directed graph and an edge query. Edges are written "u -> v", meaning
that there is a directed edge from u to v. Set the Boolean "answer" to true if the queried
edge is one of the listed edges, and false otherwise. Direction matters: "u -> v" being
listed does not mean "v -> u" is listed.

Respond with a JSON object holding the Boolean field "answer".

Example 1:
Directed edges:
v22 -> v105
v56 -> v105
v31 -> v120
v10 -> v105
v56 -> v37
v56 -> v22
v37 -> v55
v98 -> v10
v31 -> v10
v56 -> v120
v58 -> v126
v22 -> v10

Query:
Source: v56
Target: v31
Answer: {"answer": false}

Now solve this instance.
Directed edges:
v48 -> v22
v79 -> v55
v49 -> v22
v76 -> v96
v76 -> v78
v49 -> v76
v96 -> v55
v48 -> v49

Query:
Source: v48
Target: v22
\end{verbatim}
}
\noindent Expected answer:
{\scriptsize
\begin{verbatim}
{"answer": true}
\end{verbatim}
}

\paragraph{\TaskDirectedReachability.}
{\scriptsize
\begin{verbatim}
You are given a finite directed graph and a reachability query. Edges are written "u -> v",
meaning that there is a directed edge from u to v. Set the Boolean "answer" to true if there
is a directed path from the source to the target, and false otherwise.

Respond with a JSON object holding the Boolean field "answer".

Example 1:
Directed edges:
v22 -> v105
v56 -> v105
v31 -> v120
v10 -> v105
v56 -> v37
v56 -> v22
v37 -> v55
v98 -> v10
v31 -> v10
v56 -> v120
v58 -> v126
v22 -> v10

Query:
Source: v94
Target: v56
Answer: {"answer": false}

Now solve this instance.
Directed edges:
v48 -> v22
v79 -> v55
v49 -> v22
v76 -> v96
v76 -> v78
v49 -> v76
v96 -> v55
v48 -> v49

Query:
Source: v76
Target: v96
\end{verbatim}
}
\noindent Expected answer:
{\scriptsize
\begin{verbatim}
{"answer": true}
\end{verbatim}
}

\paragraph{\TaskUndirectedReachability.}
{\scriptsize
\begin{verbatim}
You are given a finite undirected graph and a connectivity query. Edges are written "u --
v", meaning that u and v are joined by an edge. The same edge may be listed more than once;
a repeat carries no extra meaning. Set the Boolean "answer" to true if some path of edges
joins the two query nodes, and false otherwise.

Respond with a JSON object holding the Boolean field "answer".

Example 1:
Undirected edges:
v22 -- v105
v56 -- v105
v31 -- v120
v10 -- v105
v56 -- v37
v56 -- v22
v37 -- v55
v98 -- v10
v31 -- v10
v56 -- v120
v58 -- v126
v22 -- v10

Query:
Node A: v94
Node B: v56
Answer: {"answer": false}

Now solve this instance.
Undirected edges:
v33 -- v73
v33 -- v7
v78 -- v73
v73 -- v60
v5 -- v29
v60 -- v5
v60 -- v73
v78 -- v7

Query:
Node A: v86
Node B: v5
\end{verbatim}
}
\noindent Expected answer:
{\scriptsize
\begin{verbatim}
{"answer": false}
\end{verbatim}
}

%% file: appendices/additional-results/additional-results.tex
\section{Additional Experimental Results}
\label{app:additional-results}

This appendix reports the full results of the three experiments, including task settings that the main-text figures omit, and two analyses referenced in \cref{sec:results}: pretrained LLMs with reasoning enabled (\cref{app:llm-reasoning}) and \TaskEdgeExistence with dense supervision (\cref{app:edge-existence-supervision}).

\subsection{\TaskEdgeExistence with dense supervision}
\label{app:edge-existence-supervision}
Small transformers fail to learn \TaskEdgeExistence when only the final answer is supervised, so we also train it with dense supervision (\Cref{app:task-graphs}).
With dense supervision, the last token of every edge is labeled with whether that edge equals the query, and the query is placed before the edge list so that a causal model can compare each edge with it.
Both settings are scored on the final answer only.
We find that dense supervision makes the task learnable.
In distribution, it is solved at every size except $n = 128$ (best seed $0.87$ normalized accuracy), with $N = 2^{12}$ training examples for $n \leq 32$ and $N = 2^{14}$ for $n = 64$ and $n = 256$ (\cref{fig:app-id-landscape}).
In the length-generalization experiment, one configuration (medium size, Shaw relative encoding) generalizes up to $n = 256$ (\cref{fig:app-lg-landscape}).
A small transformer can therefore represent and learn the lookup that \TaskEdgeExistence requires; the failure under sparse supervision appears to be a failure of learning.

\subsection{In-distribution learning}
\label{app:results-id}

The full $n \times N$ grid behind panel (a) of \cref{fig:results} appears in \cref{fig:app-id-landscape}, covering all task settings. The settings included in \cref{fig:app-id-landscape} that are omitted from the main figure are: \TaskANY; \TaskIndexOr ($k = 2, 3$); \TaskIndexAnd and \TaskIndexXor with $k = 3$; \TaskSparseOr ($k = 2, 3$); \TaskSparseAnd and \TaskSparseXor with $k = 3$; the other three \TaskBinaryAddition formats; \TaskDyck with $D = 8$; and \TaskEdgeExistence with dense supervision.

\begin{figure}[h]
    \centering
    \includegraphics[width=\linewidth]{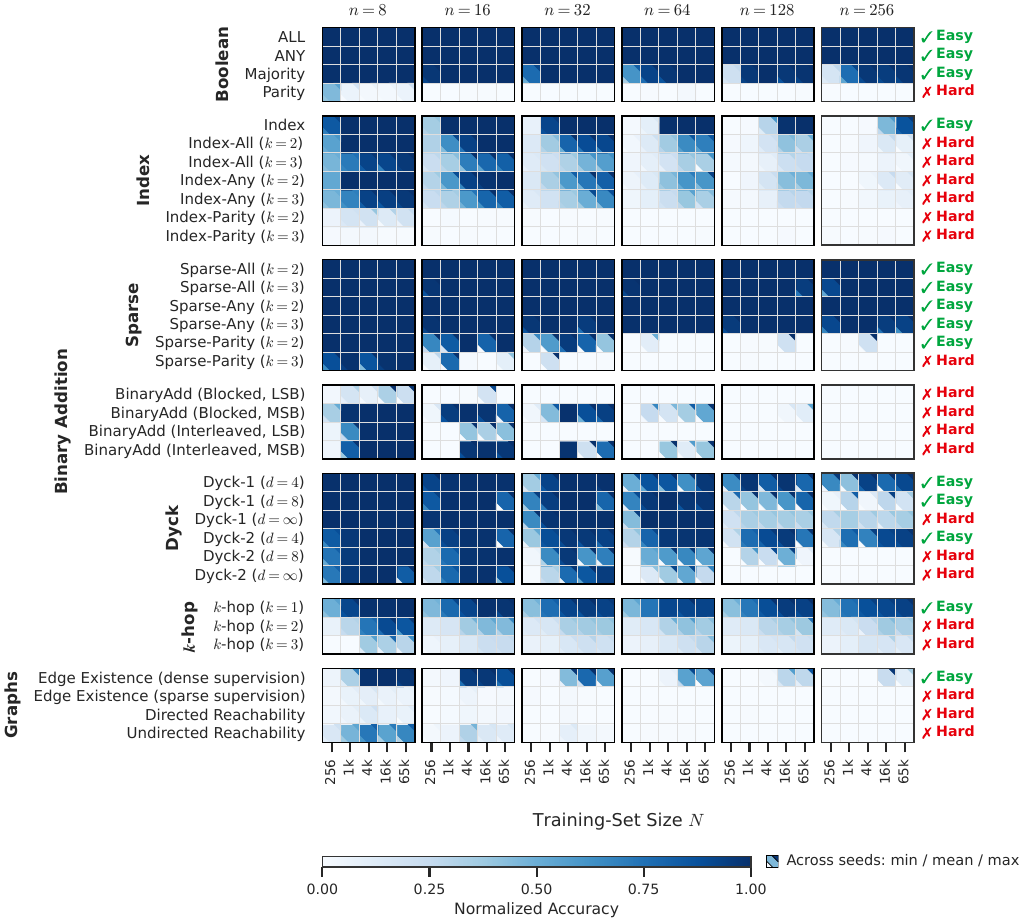}
    \caption{\textbf{In-distribution learning, all task settings.}
    Each cell shows normalized accuracy for one task setting at input size $n$ (column groups) and training-set size $N$ (columns within a group).
    Each cell depicts the minimum, mean, and maximum over five seeds.
    \Easy / \Hard labeling criteria are described in \cref{app:in-distribution-details}.}
    \label{fig:app-id-landscape}
\end{figure}

\clearpage
\subsection{Length generalization}
\label{app:results-lg}

\Cref{fig:app-lg-landscape} shows all task settings for the length generalization experiments, highlighting only the best positional encoding and model size configuration.
\Cref{fig:app-lg-positional} shows how length generalization depends on the positional encoding used.
\Cref{fig:app-lg-masks} compares causal and bidirectional attention.
\Cref{fig:app-lg-uniform} compares uniform attention with the best configuration shown in \Cref{fig:app-lg-landscape}.

\begin{figure}[h]
    \centering
    \includegraphics[width=\linewidth]{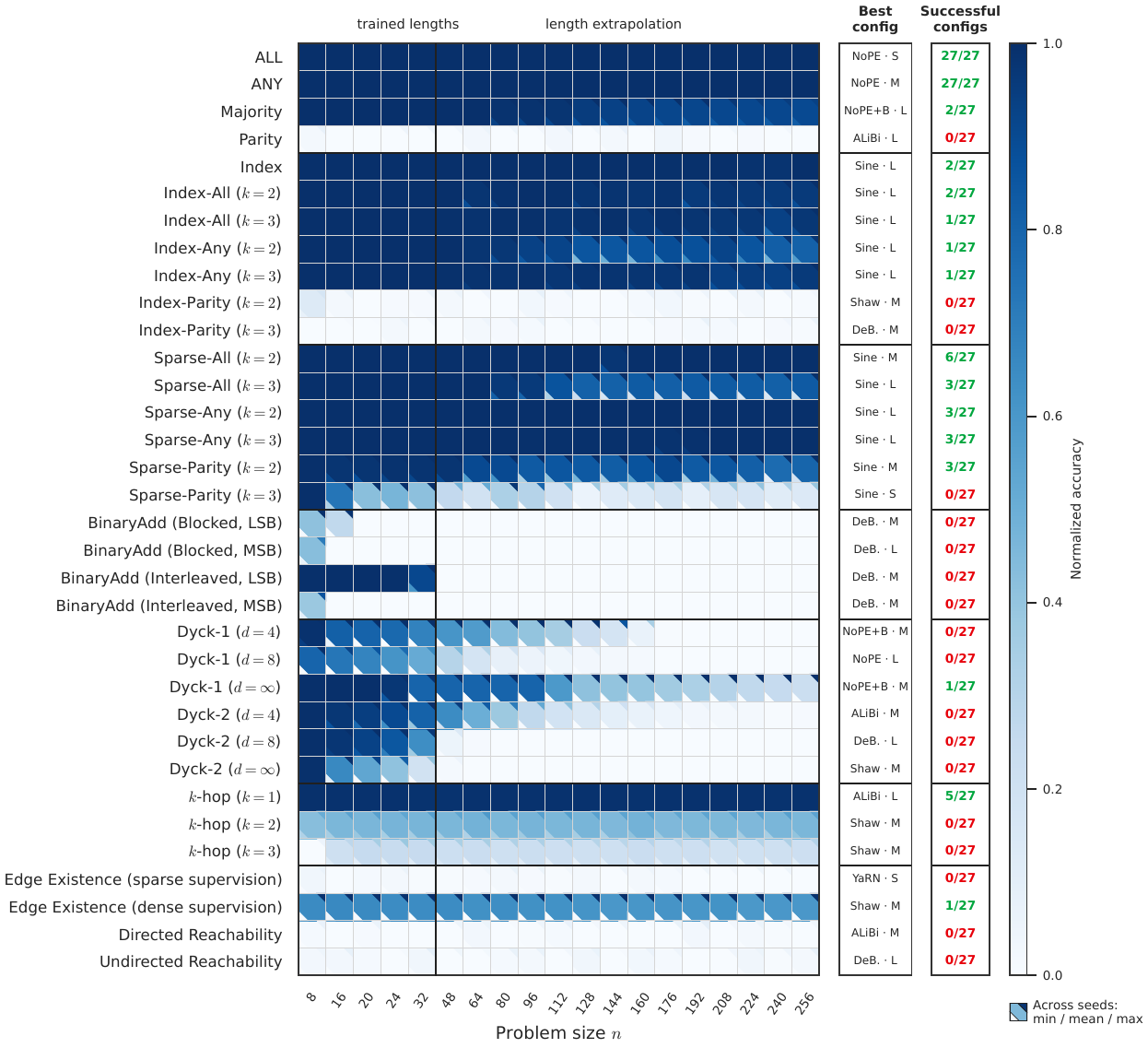}
    \caption{\textbf{Length generalization, all task settings.}
    Same conventions as panel (b) of \cref{fig:results}: for each task setting, we plot the best positional encoding and model size configuration, noted on the right.}
    \label{fig:app-lg-landscape}
\end{figure}

\begin{figure}[h]
    \centering
    \includegraphics[width=0.8\linewidth]{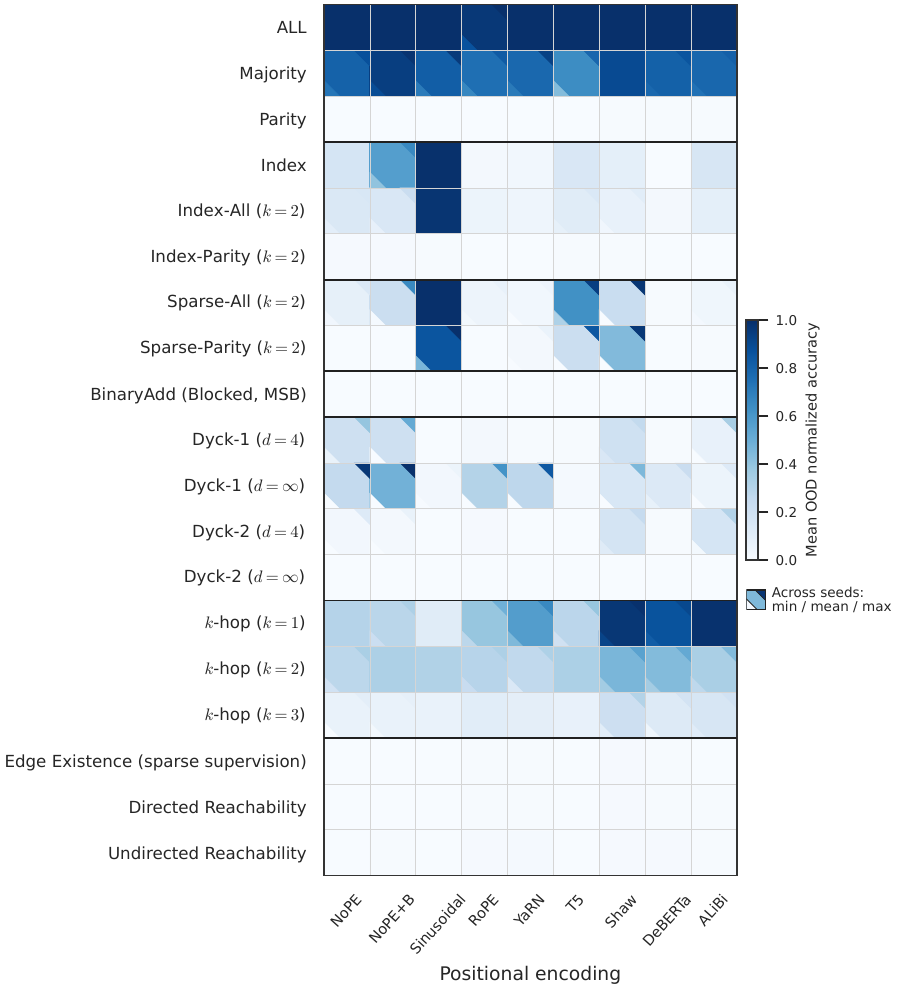}
    \caption{\textbf{Dependence of length generalization on positional encoding.}
    Mean out-of-distribution normalized accuracy ($n > 32$) for each task and positional encoding, using causal attention and the best model size for each pair.
    Corner tabs show the minimum and maximum over seeds when they differ by more than $0.05$.
    Some tasks like \TaskSparseAnd, \TaskSparseXor, \TaskOneHop, and \TaskDyckOne succeed only for certain positional encodings.
    }
    \label{fig:app-lg-positional}
\end{figure}

\begin{figure}[h]
    \centering
    \includegraphics[width=\linewidth]{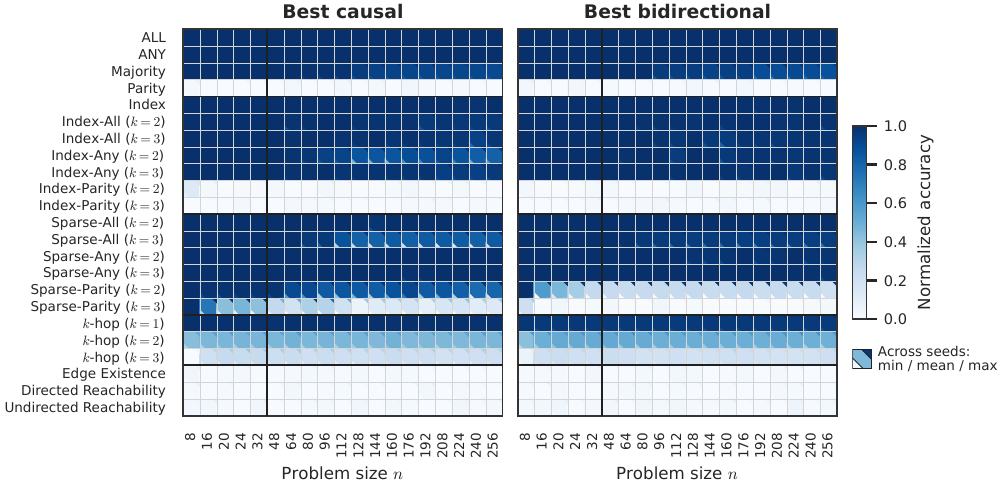}
    \caption{\textbf{Length generalization with causal versus bidirectional attention.}
    Here, we compare length generalization performance of the best configurations with causal and bidirectional attention.
    Cells show mean normalized accuracy across seeds, with a shared color scale.
    Corner tabs show the minimum and maximum over seeds when they differ by more than $0.05$.
    Performance is fairly consistent between causal and bidirectional attention, although mean performance differs for some sparse tasks.
    }
    \label{fig:app-lg-masks}
\end{figure}

\begin{figure}[h]
    \centering
    \includegraphics[width=\linewidth]{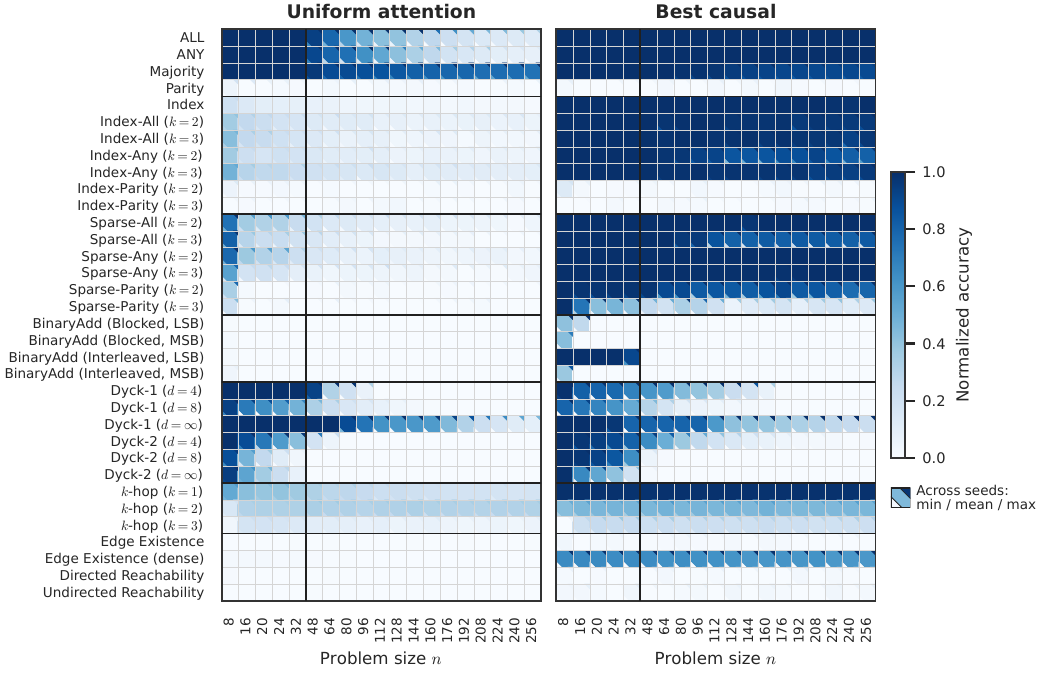}
    \caption{\textbf{Length generalization with uniform attention versus the best configuration.}
    This plot shows length generalization performance of uniform attention ($\WQ = \WK = 0$, with a BOS token) at its best model size with causal attention, compared with the best configuration of the main search in \Cref{fig:app-lg-landscape}.
    Uniform attention performs substantially worse across many problems.
    }
    \label{fig:app-lg-uniform}
\end{figure}

\FloatBarrier
\subsection{Pretrained language models}
\label{app:results-llm}

\subsubsection{LLMs with reasoning }
\label{app:llm-reasoning}

Our main LLM evaluation disables reasoning to align with theory predictions about single-step generation.
To measure how much intermediate reasoning changes the results, we evaluate GPT-5 with medium reasoning effort on the same prompts.
\Cref{fig:llm-reasoning} compares reasoning and non-reasoning performance.
With medium reasoning, many \TaskBinaryAddition responses exhaust the completion-token budget, and we count these as incorrect, as in the main evaluation (\cref{app:llm-eval-details}).
Depending on the input format, $3\%$ to $25\%$ of these responses are truncated.
Truncation is most frequent at intermediate sizes, so accuracy is not monotone in $n$ for three of the four formats.
Excluding truncated responses instead would raise accuracy at intermediate sizes by up to $0.58$, but would not change the grade the model receives for any task.

\begin{figure}[h]
\centering
\includegraphics[width=\linewidth]{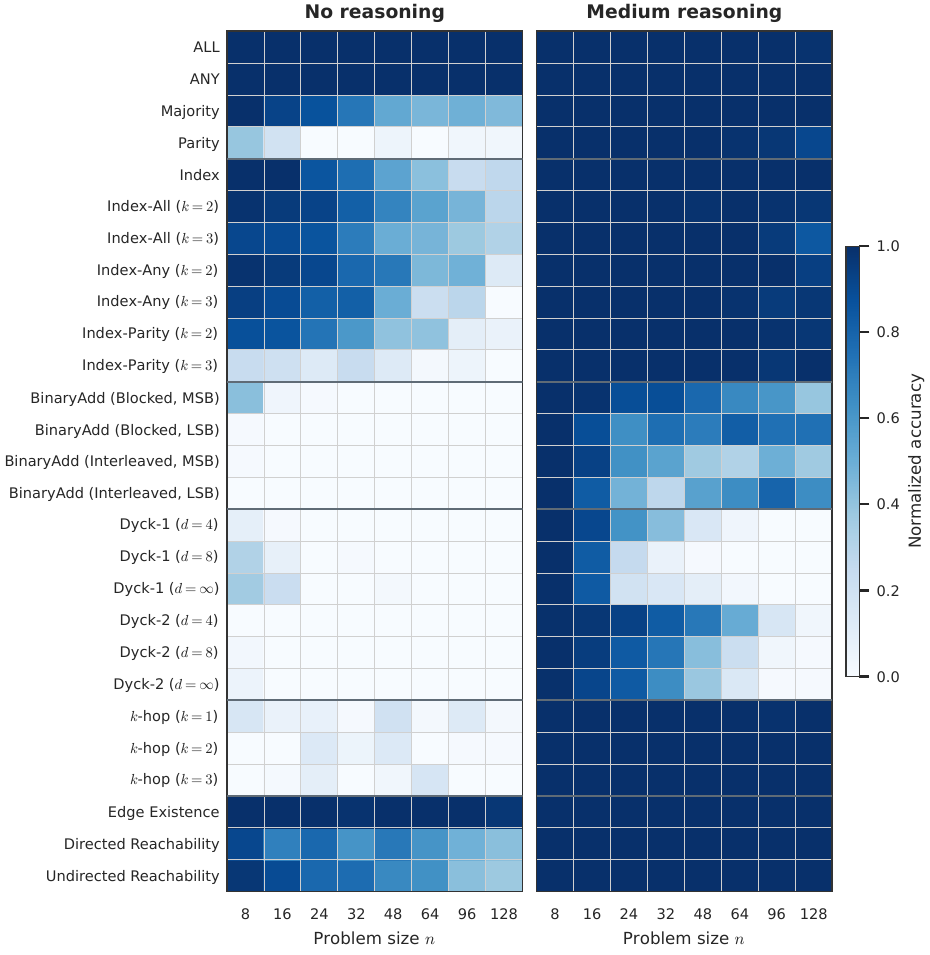}
\caption{\textbf{GPT-5 with and without reasoning.}
Normalized accuracy for the same task settings and prompts, with reasoning disabled (left) or medium reasoning effort (right).}
\label{fig:llm-reasoning}
\end{figure}

%% file: appendices/theory-background/theory-background.tex
\section{Background on the Theories}
\label{app:theory-background}

\Cref{tab:theory-transformer-edge-guide} below discusses the source and assumptions of each edge in \cref{fig:theory-transformer-relations} describing a relationship between an abstract theory and a class of transformers.
The figure is reproduced below for ease of reference.
A containment arrow $A \ContainmentArrow B$ means that every task in $A$ is also in $B$, and a non-containment arrow $A \NonContainmentArrow B$ means that there exist some tasks in $A$ that are not in $B$.
A dashed arrow indicates a caveat.
Since various technical assumptions (e.g., precision, width, etc.) can affect the theoretical containment relations, the figure compares against the base transformer class associated with each line of work.
Where there are deviations, we mark the relation with a caveat, and discuss the details in the table below.
Notably, positional encoding assumptions can differ across different base transformer classes.
In particular, the base transformer classes associated with \Starfree and \CRASP assume no positional encoding, with positional information encoded by the causal mask.
In contrast, the bidirectional classes (\ACz, \TCz, and \MPC) allow arbitrary positional encodings.

\begin{figure}[h!]
    \centering
    \input{figures/diagrams/theory-transformer-relations/theory-transformer-relation-tikz.tex}
    \caption{
    \textbf{Formal expressivity relations between theories and transformer classes (reproduced from \cref{fig:theory-transformer-relations}).}
    }
    \label{fig:theory-transformer-relations-keyed}
\end{figure}
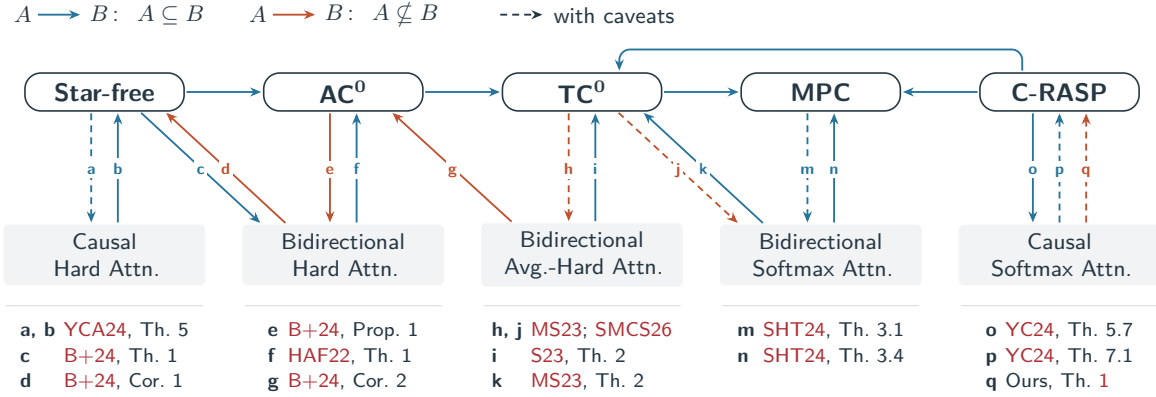

\input{tables/theory-transformer-relations/theory-transformer-edge-guide.tex}

%% file: tables/theory-transformer-relations/theory-transformer-edge-guide.tex
\begingroup
\setlength{\tabcolsep}{5pt}
\renewcommand{\arraystretch}{1.12}
\setlength{\LTcapwidth}{\textwidth}
\newcommand{\reltag}[1]{\newline{\small\itshape #1}}
\newcommand{\relhairline}[2]{\arrayrulecolor{black!40}\specialrule{.4pt}{#1}{#2}\arrayrulecolor{black}}
\newcommand{\relgroup}[1]{\addlinespace[10pt]\relhairline{0pt}{3pt}%
  \multicolumn{3}{@{}l}{\textbf{#1}} \\ \relhairline{2pt}{5pt}}
\newcommand{\relgroupfirst}[1]{\addlinespace[6pt]%
  \multicolumn{3}{@{}l}{\textbf{#1}} \\ \relhairline{2pt}{5pt}}
\newcommand{\rowsep}{\arrayrulecolor{black!14}\specialrule{.3pt}{3pt}{3pt}\arrayrulecolor{black}}
\begin{longtable}{@{}>{\raggedright\arraybackslash}p{1.1em}>{\raggedright\arraybackslash}p{0.325\textwidth}>{\raggedright\arraybackslash}p{\dimexpr0.675\textwidth-1.1em-4\tabcolsep\relax}@{}}
\caption{Guide to the edges in \cref{fig:theory-transformer-relations,fig:theory-transformer-relations-keyed}.}
\label{tab:theory-transformer-edge-guide}\\
\toprule
& \textbf{Relation} & \textbf{Result} \\
\midrule
\endfirsthead
\toprule
& \textbf{Relation} & \textbf{Result (continued)} \\
\midrule
\endhead
\bottomrule
\endfoot
\relgroupfirst{\Starfree}
\leavevmode\hypertarget{rel-edge-a}{\textbf{a}} & \Starfree~$\ContainmentSymbol$ Causal Hard Attn.\reltag{with caveats} &
\citet[Theorem~5]{yang2024masked} show that strictly future-masked, rightmost-hard-attention transformers without positional encodings recognize every star-free language.
The construction goes through the \BRASP programming language, which allows both strictly future-masked and strictly past-masked attention.
\BRASP is equivalent to linear temporal logic with both \emph{since} and \emph{until}, which defines exactly the star-free languages.
Every LTL formula is equivalent to one that uses only \emph{since} when evaluated at the last position \citep{kamp1968tense,gabbay1980temporal}, so every \BRASP program with both kinds of masking has an equivalent program that uses only causal future masking.
However, this rewriting has a cost, as the past-only formula can be exponentially larger than an equivalent formula that uses both directions \citep{artale2023pastification}.
\\ \rowsep

\leavevmode\hypertarget{rel-edge-b}{\textbf{b}} & Causal Hard Attn.~$\ContainmentSymbol$ \Starfree &
By the same theorems \citep[Theorems~2 and~4]{yang2024masked}, showing equivalence between masked hard-attention transformers, \BRASP, and LTL, every language recognized by masked hard-attention transformers is star-free; this direction preserves depth.
This result describes masked hard-attention transformers \emph{without} positional encoding, with positional information implicitly encoded in the masking.
With rational sinusoidal encoding, the class of languages recognized by masked hard-attention transformers (with both past and future attention) grows to the regular languages in \ACz \citep[Corollary~8]{yang2024masked}.
\\ \rowsep

\leavevmode\hypertarget{rel-edge-c}{\textbf{c}} & \Starfree~$\ContainmentSymbol$ Bidir. Hard Attn. &
\citet[Theorem~1]{barcelo2024logical} show that \emph{bidirectional} unique-hard-attention transformers with positional encodings recognize every language definable in $\FO[\mathrm{Mon}]$, which contains $\FO[<]$ and hence every star-free language.
\\ \rowsep

\leavevmode\hypertarget{rel-edge-d}{\textbf{d}} & Bidir. Hard Attn.~$\NonContainmentSymbol$ \Starfree &
Bidirectional hard-attention transformers with positional encodings recognize every regular language in \ACz \citep[Corollary~1]{barcelo2024logical}. This includes $(aa)^*$, for example, which is regular and in \ACz but not star-free \citep{schutzenberger1965finite}.
\\

\relgroup{\ACz}
\leavevmode\hypertarget{rel-edge-e}{\textbf{e}} & \ACz~$\NonContainmentSymbol$ Bidir. Hard Attn. &
No bidirectional unique-hard-attention transformer recognizes approximate majority, a language in \ACz that contains every bitstring with at least two-thirds ones and no bitstring with at most one-third ones \citep[Proposition~1]{barcelo2024logical}. \\ \rowsep

\leavevmode\hypertarget{rel-edge-f}{\textbf{f}} & Bidir. Hard Attn.~$\ContainmentSymbol$ \ACz &
\citet[Theorem~1]{hao2022formal} show that every language recognized by a bidirectional unique-hard-attention transformer is in \ACz, even for their generalized model with unrestricted position-wise functions and positional encodings.
The circuit family is non-uniform, unlike the uniform \ACz used elsewhere in this table. \\ \rowsep

\leavevmode\hypertarget{rel-edge-g}{\textbf{g}} & Bidir. Avg.-Hard Attn.~$\NonContainmentSymbol$ \ACz &
\citet[Theorem~5]{chiang2023tighter} build a bidirectional transformer, using only uniform attention (a special case of average-hard attention) and a CLS token, that recognizes every language definable in $\FOC[+;\mathrm{MOD}]$, including \TaskMajority \citep[Proposition~11]{chiang2023tighter}.
\TaskMajority is not in \ACz \citep{furst1984parity,ajtai1983sigma}.
\citet[Corollary~2]{barcelo2024logical} also show that \TaskMajority and \TaskParity can be recognized by average-hard attention with future positional masking.
\\

\relgroup{\TCz}
\leavevmode\hypertarget{rel-edge-h}{\textbf{h}} & \TCz~$\NonContainmentSymbol$ Bidir. Avg.-Hard Attn.\reltag{conjectured} &
It is conjectured that \TCz contains languages that are not expressible by average-hard (or softmax) attention transformers.
The \TCz upper bounds for average-hard and softmax attention (edges~i and~k) are one-directional.
\citet{merrill2023logic} state that they believe these bounds are not tight.
\citet[Conjecture~6.1]{svete2026padded} state the conjecture formally for constant-width, $\Theta(\log n)$-precision softmax and average-hard-attention transformers, conditional on a uniform \TCz size hierarchy.
\\ \rowsep

\leavevmode\hypertarget{rel-edge-i}{\textbf{i}} & Bidir. Avg.-Hard Attn.~$\ContainmentSymbol$ \TCz &
\citet[Theorem~2]{strobl2023average} shows that every language recognized by a bidirectional average-hard-attention transformer with $O(\log n)$-bit precision is in log-space-uniform \TCz. \\ \rowsep

\leavevmode\hypertarget{rel-edge-j}{\textbf{j}} & \TCz~$\NonContainmentSymbol$ Bidir. Softmax Attn.\reltag{conjectured} &
Conjectured, as for edge~h.
Narrower results exist for restricted models.
Fixed-precision bidirectional transformers with rational-frequency sinusoidal positional encodings recognize only languages definable in $\FOC[+;\mathrm{MOD}]$, which is strictly smaller than uniform \TCz \citep[Theorem~2 and Proposition~4]{chiang2023tighter}.
With constant precision and every arithmetic operation rounded, transformers are contained in \ACz \citep[Theorem~4.1]{london2025pause}.
These results do not cover soft-attention transformers with $O(\log n)$-bit precision. 
\\ \rowsep

\leavevmode\hypertarget{rel-edge-k}{\textbf{k}} & Bidir. Softmax Attn.~$\ContainmentSymbol$ \TCz &
\citet[Theorem~2]{merrill2023logic} show that every transformer with $O(\log n)$-bit precision and positional encodings computable in $O(\log n)$ time can be expressed in $\FOM[\mathsf{BIT}]$, which equals DLOGTIME-uniform \TCz. \\

\relgroup{\MPC}
\leavevmode\hypertarget{rel-edge-m}{\textbf{m}} & \MPC~$\ContainmentSymbol$ Bidir. Softmax Attn.\reltag{with caveats} &
\citet[Theorem~3.1]{sanford2024parallel} show that every $R$-round \MPC protocol with local memory $O(n^\delta)$ is simulated by a transformer of depth $R+1$, width $O(n^{4\delta}\log n)$ and $O(\log\log n)$ heads.
However, this assumes unbounded pre-computation in the MLP and the query, key, and value maps, and that this computation may differ from position to position.
Moreover, the width grows with the input length.
\\ \rowsep

\leavevmode\hypertarget{rel-edge-n}{\textbf{n}} & Bidir. Softmax Attn.~$\ContainmentSymbol$ \MPC &
\citet[Theorem~3.4]{sanford2024parallel} show that every transformer with width times heads $O(n^\delta)$, which includes every fixed-width transformer, is simulated by an $O(L)$-round \MPC protocol with $O(n^2)$ machines.
The protocol's local computation evaluates the transformer's position-wise maps.
Note that this uses a different memory regime compared to edge~m, with $O(n^2)$ machines whereas edge~m uses fewer machines than tokens. \\

\relgroup{\CRASP}
\leavevmode\hypertarget{rel-edge-o}{\textbf{o}} & \CRASP~$\ContainmentSymbol$ Causal Softmax Attn. &
\citet[Theorems~4.3 and~5.7]{yang2024counting} compile every \CRASP program into a future-masked softmax transformer without positional encodings, with uniform attention and one layer per level of count nesting. 
This construction uses infinite precision and layer normalization; \citet[Proposition~B.4]{yang2025kneedeep} give a fixed-precision version with the same depth, for a model that rounds every value except the sums inside attention.
\\ \rowsep

\leavevmode\hypertarget{rel-edge-p}{\textbf{p}} & Causal Softmax Attn.~$\ContainmentSymbol$ \CRASP\reltag{with caveats} &
\citet[Theorems~4.3 and~7.1]{yang2024counting} show that every fixed-precision, future-masked softmax transformer without positional encodings can be simulated by a \CRASP program.
The construction enumerates the finitely many possible activation values, which makes the program exponentially larger than the transformer.
Edge q shows that this exponential increase is, in general, necessary.
\\ \rowsep

\leavevmode\hypertarget{rel-edge-q}{\textbf{q}} & Causal Softmax Attn.~$\NonContainmentSymbol$ \CRASP\reltag{with caveats} &
Holds at comparable size (\cref{thm:c-rasp-retrieval-separation}).
On query-dependent retrieval over $L = 2^m$ key-value items, one causal softmax attention head solves the task with residual width $m+2$, while every \CRASP program for the task needs $\Omega(2^m/m)$ prefix-count nodes. \\

\relgroup{Relations between theories (unlettered edges)}
& \Starfree $\subsetneq$ \ACz &
Star-free languages are exactly the languages definable in $\FO[<]$ \citep{mcnaughton1971counter}, and $\FO[<] \subseteq \FO[<,\mathrm{BIT}]$, which equals DLOGTIME-uniform \ACz \citep{immerman1999descriptive}.
The containment is strict: $(aa)^*$ is in \ACz but is not star-free \citep{schutzenberger1965finite}. \\ \rowsep

& \ACz $\subsetneq$ \TCz &
Every \ACz circuit is a \TCz circuit; \TaskMajority is in \TCz but not in \ACz \citep{furst1984parity}. \\ \rowsep

& \CRASP $\subsetneq$ \TCz &
Every \CRASP program translates into a uniform \TCz circuit family with a linear number of gates \citep[Proposition~31 and Remark~32]{huang2024formal}.
The containment is strict: \TaskParity is in \TCz but not expressible in \CRASP \citep[Lemma~11]{huang2024formal}. \\ \rowsep

& \TCz $\subseteq$ \MPC &
\citet[proof of Theorem~6.1]{roughgarden2018shuffles} simulate every logspace-uniform \NCone circuit family, and hence \TCz, in $O(1/\delta)$ rounds with polynomially many machines that each read $O(n^\delta)$ words per round.
\citet[Theorem~9]{frei2019circuit} give a constant-round simulation in a MapReduce model with local memory of at least $\sqrt{n}$.
\\ \rowsep

& \CRASP $\subseteq$ \MPC &
Follows from \CRASP $\subseteq$ \TCz $\subseteq$ \MPC. \\

\relgroup{Not shown}
& \BAPO, \Sensitivity &
Omitted because they have no formal containment relation to transformer classes.
\\

\end{longtable}
\endgroup

%% file: tables/theory-task-matrix/theory-task-prediction-notes.tex
\section{Notes on the Theory--Task Predictions}
\label{app:prediction-notes}

\Cref{tab:theory-task-predictions-annotated} annotates each task prediction with a link to a row in \Cref{tab:theory-task-predictions-explanations}, which explains the source of each prediction.

\input{tables/theory-task-matrix/theory-task-matrix-annotated}
\input{tables/theory-task-matrix/theory-task-matrix-explanations}

%% file: tables/theory-task-matrix/theory-task-matrix-annotated.tex
\begin{table*}[h!]
\centering
\caption{\textbf{Theory--task prediction matrix.} Annotated version of the theory-task predictions in \Cref{tab:theory-task-predictions}. For each theory, we apply its hardness criterion, adapting tasks when necessary to fit within the theory's assumptions about inputs. Superscripts link to explanations and references in \Cref{tab:theory-task-predictions-explanations}. \predneg{}: \Hard, \predpos{}: \Easy (in some cases, not \Hard), \textsuperscript{\prednewproof}: non-trivial result proved here, $^*$: assuming a complexity theory conjecture.}
\label{tab:theory-task-predictions-annotated}
\renewcommand{\predglyphsize}{\fontsize{15}{18}\selectfont}
\input{tables/theory-task-matrix/theory-task-matrix-annotated-content}
\end{table*}

%% file: tables/theory-task-matrix/theory-task-matrix-annotated-content.tex
\begingroup
\setlength{\tabcolsep}{2.25pt}
\renewcommand{\arraystretch}{1.15}
\begin{tabular*}{\textwidth}{@{\extracolsep{\fill}}l c c c c c c c@{}}
\toprule
\textbf{Task} & \textbf{\Starfree} & \textbf{\boldmath\ACz} & \textbf{\boldmath\TCz} & \textbf{\CRASP} & \textbf{\MPC} & \textbf{\BAPO} & \textbf{\Sensitivity} \\
\midrule
\multicolumn{8}{@{}l}{\textbf{Boolean}} \\[1pt]
\quad \TaskALL/\TaskANY & \predpos{\hyperlink{pred-all-any-star-free}{\textnormal{1}}} & \predpos{\hyperlink{pred-all-any-circuits}{\textnormal{2}}} & \predpos{\hyperlink{pred-all-any-circuits}{\textnormal{2}}} & \predpos{\hyperlink{pred-all-any-crasp}{\textnormal{3}}} & \predpos{\hyperlink{pred-regular-mpc}{\textnormal{4}}} & \predpos{\hyperlink{pred-regular-bapo}{\textnormal{5}}} & \predpos{\hyperlink{pred-all-any-sensitivity}{\textnormal{6}}} \\
\quad \TaskMajority & \predneg{\hyperlink{pred-majority-parity-ac0-hard}{\textnormal{7}}} & \predneg{\hyperlink{pred-majority-parity-ac0-hard}{\textnormal{7}}} & \predpos{\hyperlink{pred-majority-parity-tc0}{\textnormal{8}}} & \predpos{\hyperlink{pred-majority-crasp}{\textnormal{9}}} & \predpos{\hyperlink{pred-majority-mpc}{\textnormal{10}}} & \predneg{\hyperlink{pred-majority-bapo}{\textnormal{11}}} & \predpos{\hyperlink{pred-majority-sensitivity}{\textnormal{12}}} \\
\quad \TaskParity & \predneg{\hyperlink{pred-majority-parity-ac0-hard}{\textnormal{7}}} & \predneg{\hyperlink{pred-majority-parity-ac0-hard}{\textnormal{7}}} & \predpos{\hyperlink{pred-majority-parity-tc0}{\textnormal{8}}} & \predneg{\hyperlink{pred-parity-crasp}{\textnormal{13}}} & \predpos{\hyperlink{pred-regular-mpc}{\textnormal{4}}} & \predpos{\hyperlink{pred-regular-bapo}{\textnormal{5}}} & \predneg{\hyperlink{pred-parity-sensitivity}{\textnormal{14}}} \\
\predfamilyrule
\multicolumn{8}{@{}l}{\textbf{\TaskIndex}} \\[1pt]
\quad \TaskIndex & \predneg{\hyperlink{pred-index-nonregular}{\textnormal{15}}} & \predpos{\hyperlink{pred-index-circuits}{\textnormal{16}}} & \predpos{\hyperlink{pred-index-circuits}{\textnormal{16}}} & \predneg[\prednewproof]{\hyperlink{pred-index-growing-range}{\textnormal{17}}} & \predpos{\hyperlink{pred-index-mpc}{\textnormal{18}}} & \predpos{\hyperlink{pred-index-bapo}{\textnormal{19}}} & \predpos{\hyperlink{pred-index-sensitivity}{\textnormal{20}}} \\
\quad \TaskIndexAndOr & \predneg{\hyperlink{pred-index-nonregular}{\textnormal{15}}} & \predpos{\hyperlink{pred-index-circuits}{\textnormal{16}}} & \predpos{\hyperlink{pred-index-circuits}{\textnormal{16}}} & \predneg[\prednewproof]{\hyperlink{pred-index-growing-range}{\textnormal{17}}} & \predpos{\hyperlink{pred-index-mpc}{\textnormal{18}}} & \predpos[\prednewproof]{\hyperlink{pred-index-star-bapo}{\textnormal{21}}} & \predpos{\hyperlink{pred-index-sensitivity}{\textnormal{20}}} \\
\quad \TaskIndexXor & \predneg{\hyperlink{pred-index-nonregular}{\textnormal{15}}} & \predpos{\hyperlink{pred-index-circuits}{\textnormal{16}}} & \predpos{\hyperlink{pred-index-circuits}{\textnormal{16}}} & \predneg[\prednewproof]{\hyperlink{pred-index-growing-range}{\textnormal{17}}} & \predpos{\hyperlink{pred-index-mpc}{\textnormal{18}}} & \predpos[\prednewproof]{\hyperlink{pred-index-star-bapo}{\textnormal{21}}} & \predpos{\hyperlink{pred-index-sensitivity}{\textnormal{20}}} \\
\predfamilyrule
\multicolumn{8}{@{}l}{\textbf{Sparse}} \\[1pt]
\quad \TaskSparseAndOr & \predpos{\hyperlink{pred-sparse-logic}{\textnormal{22}}} & \predpos{\hyperlink{pred-sparse-logic}{\textnormal{22}}} & \predpos{\hyperlink{pred-sparse-logic}{\textnormal{22}}} & \predpos{\hyperlink{pred-sparse-crasp}{\textnormal{23}}} & \predpos{\hyperlink{pred-sparse-mpc}{\textnormal{24}}} & \predpos{\hyperlink{pred-sparse-bapo}{\textnormal{25}}} & \predpos{\hyperlink{pred-sparse-sensitivity}{\textnormal{26}}} \\
\quad \TaskSparseXor & \predpos{\hyperlink{pred-sparse-logic}{\textnormal{22}}} & \predpos{\hyperlink{pred-sparse-logic}{\textnormal{22}}} & \predpos{\hyperlink{pred-sparse-logic}{\textnormal{22}}} & \predpos{\hyperlink{pred-sparse-crasp}{\textnormal{23}}} & \predpos{\hyperlink{pred-sparse-mpc}{\textnormal{24}}} & \predpos{\hyperlink{pred-sparse-bapo}{\textnormal{25}}} & \predpos{\hyperlink{pred-sparse-sensitivity}{\textnormal{26}}} \\
\predfamilyrule
\textbf{\TaskBinaryAddition} & \predneg{\hyperlink{pred-addition-hardness}{\textnormal{27}}} & \predpos{\hyperlink{pred-addition-circuits}{\textnormal{28}}} & \predpos{\hyperlink{pred-addition-circuits}{\textnormal{28}}} & \predneg{\hyperlink{pred-addition-hardness}{\textnormal{27}}} & \predpos{\hyperlink{pred-addition-circuits}{\textnormal{28}}} & \predneg[\prednewproof]{\hyperlink{pred-addition-bapo}{\textnormal{29}}} & \predpos{\hyperlink{pred-addition-sensitivity}{\textnormal{30}}} \\
\predfamilyrule
\multicolumn{8}{@{}l}{\textbf{\TaskDyck}} \\[1pt]
\quad \TaskDyckOne (bounded) & \predpos{\hyperlink{pred-bounded-dyck-logic}{\textnormal{31}}} & \predpos{\hyperlink{pred-bounded-dyck-logic}{\textnormal{31}}} & \predpos{\hyperlink{pred-bounded-dyck-logic}{\textnormal{31}}} & \predpos{\hyperlink{pred-dyck-one-crasp}{\textnormal{32}}} & \predpos{\hyperlink{pred-regular-mpc}{\textnormal{4}}} & \predpos{\hyperlink{pred-regular-bapo}{\textnormal{5}}} & \prednone \\
\quad \TaskDyckOne (unbounded) & \predneg{\hyperlink{pred-unbounded-dyck-circuits}{\textnormal{33}}} & \predneg{\hyperlink{pred-unbounded-dyck-circuits}{\textnormal{33}}} & \predpos{\hyperlink{pred-unbounded-dyck-circuits}{\textnormal{33}}} & \predpos{\hyperlink{pred-dyck-one-crasp}{\textnormal{32}}} & \predpos{\hyperlink{pred-unbounded-dyck-mpc}{\textnormal{34}}} & \predneg[\prednewproof]{\hyperlink{pred-unbounded-dyck-bapo}{\textnormal{35}}} & \prednone \\
\quad \TaskDyckTwo (bounded) & \predpos{\hyperlink{pred-bounded-dyck-logic}{\textnormal{31}}} & \predpos{\hyperlink{pred-bounded-dyck-logic}{\textnormal{31}}} & \predpos{\hyperlink{pred-bounded-dyck-logic}{\textnormal{31}}} & \predpos[\prednewproof]{\hyperlink{pred-bounded-dyck-crasp}{\textnormal{36}}} & \predpos{\hyperlink{pred-regular-mpc}{\textnormal{4}}} & \predpos{\hyperlink{pred-regular-bapo}{\textnormal{5}}} & \prednone \\
\quad \TaskDyckTwo (unbounded) & \predneg{\hyperlink{pred-unbounded-dyck-circuits}{\textnormal{33}}} & \predneg{\hyperlink{pred-unbounded-dyck-circuits}{\textnormal{33}}} & \predpos{\hyperlink{pred-unbounded-dyck-circuits}{\textnormal{33}}} & \predneg{\hyperlink{pred-unbounded-dyck-two-crasp}{\textnormal{37}}} & \predpos{\hyperlink{pred-unbounded-dyck-mpc}{\textnormal{34}}} & \predneg[\prednewproof]{\hyperlink{pred-unbounded-dyck-bapo}{\textnormal{35}}} & \prednone \\
\predfamilyrule
\textbf{\TaskKHop} & \predpos[\prednewproof]{\hyperlink{pred-fixed-hop-logic}{\textnormal{38}}} & \predpos[\prednewproof]{\hyperlink{pred-fixed-hop-logic}{\textnormal{38}}} & \predpos[\prednewproof]{\hyperlink{pred-fixed-hop-logic}{\textnormal{38}}} & \predneg[\prednewproof]{\hyperlink{pred-fixed-hop-crasp}{\textnormal{39}}} & \predpos{\hyperlink{pred-fixed-hop-mpc}{\textnormal{40}}} & \predpos[\prednewproof]{\hyperlink{pred-regular-bapo}{\textnormal{5}},\hyperlink{pred-fixed-hop-logic}{\textnormal{38}}} & \prednone \\
\predfamilyrule
\multicolumn{8}{@{}l}{\textbf{Graphs}} \\[1pt]
\quad \TaskEdgeExistence & \predneg{\hyperlink{pred-edge-nonregular}{\textnormal{41}}} & \predpos{\hyperlink{pred-edge-circuits}{\textnormal{42}}} & \predpos{\hyperlink{pred-edge-circuits}{\textnormal{42}}} & \predneg[\prednewproof]{\hyperlink{pred-edge-crasp}{\textnormal{43}}} & \predpos{\hyperlink{pred-edge-mpc}{\textnormal{44}}} & \predpos{\hyperlink{pred-edge-bapo}{\textnormal{45}}} & \prednone \\
\quad \TaskDirectedReachability & \predneg{\hyperlink{pred-reachability-ac0}{\textnormal{46}}} & \predneg{\hyperlink{pred-reachability-ac0}{\textnormal{46}}} & \prednegc{\hyperlink{pred-directed-reachability-tc0}{\textnormal{47}}} & \prednegc{\hyperlink{pred-directed-reachability-tc0}{\textnormal{47}}} & \prednegc{\hyperlink{pred-reachability-mpc}{\textnormal{48}}} & \predneg{\hyperlink{pred-reachability-bapo}{\textnormal{49}}} & \prednone \\
\quad \TaskUndirectedReachability & \predneg{\hyperlink{pred-reachability-ac0}{\textnormal{46}}} & \predneg{\hyperlink{pred-reachability-ac0}{\textnormal{46}}} & \prednegc{\hyperlink{pred-undirected-reachability-tc0}{\textnormal{50}}} & \prednegc{\hyperlink{pred-undirected-reachability-tc0}{\textnormal{50}}} & \prednegc{\hyperlink{pred-reachability-mpc}{\textnormal{48}}} & \predneg{\hyperlink{pred-reachability-bapo}{\textnormal{49}}} & \prednone \\
\bottomrule
\end{tabular*}
\endgroup

%% file: tables/theory-task-matrix/theory-task-matrix-explanations.tex
\begingroup
\setlength{\tabcolsep}{4pt}
\renewcommand{\arraystretch}{1.15}
\begin{longtable}{@{}m{0.07\textwidth}p{\dimexpr0.93\textwidth-8pt\relax}@{}}
\caption{\textbf{Explanations for the theory--task predictions.} IDs correspond to the numbered links in \Cref{tab:theory-task-predictions-annotated}.}
\label{tab:theory-task-predictions-explanations}\\
\toprule \textbf{ID} & \textbf{Explanation} \\ \midrule
\endfirsthead
\toprule \textbf{ID} & \textbf{Explanation (continued)} \\ \midrule
\endhead
\bottomrule\endfoot
\input{tables/theory-task-matrix/theory-task-matrix-explanations-content}
\end{longtable}
\endgroup

%% file: tables/theory-task-matrix/theory-task-matrix-explanations-content.tex
\hypertarget{pred-all-any-star-free}{\textbf{1}} & \textbf{\TaskALL/\TaskANY are star-free.} They are defined in $\FO[<]$ by $\forall i\,Q_1(i)$
and $\exists i\,Q_1(i)$, respectively, where $Q_1(i)$
holds when position $i$ contains a one.
\\
\hypertarget{pred-all-any-circuits}{\textbf{2}} & \textbf{\TaskALL/\TaskANY have depth-one circuits.} \TaskALL and \TaskANY are respectively a single unbounded-fan-in AND and OR gate, so they lie in both $\ACz$ and $\TCz$.
\\
\hypertarget{pred-all-any-crasp}{\textbf{3}} & \textbf{\TaskALL/\TaskANY have small \CRASP programs.} \TaskALL tests whether the prefix count of zeros is zero; \TaskANY tests whether the prefix count of ones is positive.
\\
\hypertarget{pred-regular-mpc}{\textbf{4}} & \textbf{Regular languages have constant-round, sublinear-memory \MPC algorithms.} \citet[Theorem~3.2]{gilbert2023brief} give an $O(1/\delta)$-round \MPC algorithm with $O(n^\delta)$ local memory for any fixed regular language. This covers \TaskANY/\TaskALL, \TaskParity, and constant-depth \TaskDyck.
\\
\hypertarget{pred-regular-bapo}{\textbf{5}} & \textbf{Regular languages are \BAPO-easy.} \citet[Theorem~2]{schnabel2025lost} show that every regular language has a constant-bandwidth \BAPO; this includes \TaskALL/\TaskANY, \TaskParity, and fixed-depth \TaskDyck. For fixed alphabet $\Sigma$ and hop count $k$, each output language of \TaskKHop is star-free and hence regular by \Cref{res:k-hop-logic}. Concatenating the finite-state summaries for the $|\Sigma|+1$ possible outputs gives a constant-bandwidth \BAPO for the full output.
\\
\hypertarget{pred-all-any-sensitivity}{\textbf{6}} & \textbf{\TaskALL/\TaskANY have low average sensitivity.} Flipping a bit changes \TaskALL (\TaskANY) exactly when all other bits are one (zero). Under uniform Boolean inputs, both therefore have average sensitivity $n/2^{n-1}=o(n)$, giving an \Easy prediction. This prediction uses the uniform distribution rather than the class-balanced distribution we use in experiments.
\\
\hypertarget{pred-majority-parity-ac0-hard}{\textbf{7}} & \textbf{\TaskMajority and \TaskParity are not in $\ACz$~\citep{strobl2024what}}. This also means they are not \Starfree due to containment in $\ACz$.
\\
\hypertarget{pred-majority-parity-tc0}{\textbf{8}} & \textbf{\TaskMajority and \TaskParity are in $\TCz$~\citep{strobl2024what}.}
\\
\hypertarget{pred-majority-crasp}{\textbf{9}} & \textbf{\TaskMajority is in \CRASP~\citep{huang2024formal}.}
\\
\hypertarget{pred-majority-mpc}{\textbf{10}} & \textbf{\TaskMajority has a constant-round, sublinear-memory \MPC algorithm.} The bounded-memory prefix-sum algorithm of \citet{goodrich2011sorting} computes the total number of ones in $O(\log_M n)$ rounds. Taking $M=\Theta(n^\delta)$ gives $O(1/\delta)$ rounds with $O(n^\delta)$ local memory; a final threshold comparison computes \TaskMajority.
\\
\hypertarget{pred-majority-bapo}{\textbf{11}} & \textbf{\TaskMajority is \BAPO-hard~\citep{schnabel2025lost}.}
\\
\hypertarget{pred-majority-sensitivity}{\textbf{12}} & \textbf{\TaskMajority has low average sensitivity~\citep{hahn2024sensitive}.} \TaskMajority has average sensitivity $\Theta(\sqrt n)$ under uniform Boolean inputs, so it is not predicted \Hard by sensitivity results~\citep[Theorem 4, Theorem 6]{hahn2024sensitive}.
\\
\hypertarget{pred-parity-crasp}{\textbf{13}} & \textbf{\TaskParity is not expressible in \CRASP~\citep[Lemma 11]{huang2024formal}.} \citet{huang2024formal} show the stronger result that \TaskParity is not expressible in \CRASP{}[periodic, local], which adds additional positional operators to base \CRASP.
\\
\hypertarget{pred-parity-sensitivity}{\textbf{14}} & \textbf{\TaskParity is maximally sensitive.} Flipping any input bit flips the answer to \TaskParity, giving it average sensitivity $n$~\citep{hahn2024sensitive}.
\\
\hypertarget{pred-index-nonregular}{\textbf{15}} & \textbf{Growing-range \TaskIndex is not star-free.} \Cref{res:index-nonregular} shows that \TaskIndex and fixed-arity \TaskIndexAnd, \TaskIndexOr, and \TaskIndexXor are not finite-state computable, even when correctness is required only on valid inputs. In particular, their valid-input behavior has no star-free realization. In contrast, fixed-range \TaskIndex and \TaskIndexStar are star-free by \Cref{res:index-fixed-logic,res:index-bool-fixed-logic}.
\\
\hypertarget{pred-index-circuits}{\textbf{16}} & \textbf{Binary-index \TaskIndex is in \ACz.} \Cref{res:index-fo-bit} shows growing-range \TaskIndex and \TaskIndexStar are in DLOGTIME-uniform $\ACz$ and $\TCz$, encoding indices in binary.
\\
\hypertarget{pred-index-growing-range}{\textbf{17}\textsuperscript{\prednewproof}} & \textbf{Growing-range \TaskIndex requires growing \CRASP size.} \Cref{res:index-crasp-growing} rules out a fixed-size base \CRASP program for \TaskIndex and fixed-arity \TaskIndexAnd, \TaskIndexOr, and \TaskIndexXor. The result applies to binary and atomic query encodings. The displayed \Hard predictions concern a growing query range, as in the ID and LLM experiments. In contrast, length-generalization experiments fix the query range; fixed-range \TaskIndex and \TaskIndexStar admit fixed-size \CRASP programs by \Cref{res:index-fixed-crasp,res:index-bool-fixed-crasp}.
\\
\hypertarget{pred-index-mpc}{\textbf{18}} & \textbf{Growing-range \TaskIndex is \MPC-easy.} \Cref{res:index-mpc} gives a three-round protocol for \TaskIndex and \TaskIndexStar with fixed arity and Boolean function, using $O(n^\delta)$ words of local memory and polynomial total memory.
\\
\hypertarget{pred-index-bapo}{\textbf{19}} & \textbf{Atomic-index \TaskIndex is \BAPO-easy.} \TaskIndex has a $(0,1)$-\BAPO~\citep{schnabel2025lost}: the attention oracle selects the queried data token. This bound uses an atomic address token, not a binary-encoded index.
\\
\hypertarget{pred-index-sensitivity}{\textbf{20}} & \textbf{\TaskIndex and fixed-arity \TaskIndexStar have low sensitivity.} Fix a layout of $n$ data bits followed by $k$ binary addresses of width $\ell=\lceil\log_2(n+1)\rceil$. Extend the task to all bitstrings of this layout by returning zero whenever any address is outside $[n]$. At any input, at most $k$ queried data bits and $k\ell$ address bits can be sensitive. Thus pointwise sensitivity, and hence average sensitivity, is at most $k+k\ell=O(k\log(n+1))$, sublinear in the total input length for fixed $k$ and below the linear hardness threshold of \citet[Theorem~6]{hahn2024sensitive}.
\\
\hypertarget{pred-index-star-bapo}{\textbf{21}\textsuperscript{\prednewproof}} & \textbf{Atomic-index \TaskIndexStar is \BAPO-easy.} \Cref{res:index-bapo-growing} gives a $(k-1,k)$-\BAPO for \TaskIndexStar with fixed arity $k$ and Boolean function. This extends the existing \TaskIndex result to Boolean functions of multiple retrieved bits. The bound uses atomic address tokens; it does not establish the same bound for binary-encoded indices.
\\
\hypertarget{pred-sparse-logic}{\textbf{22}} & \textbf{Fixed-support Boolean tasks are star-free.} \Cref{res:sparse-logic} gives an $\FO[<]$ definition of \TaskSparseStar for fixed selected positions and Boolean function, hence membership in \Starfree and DLOGTIME-uniform $\ACz$ and $\TCz$.
\\
\hypertarget{pred-sparse-crasp}{\textbf{23}} & \textbf{Fixed-support Boolean tasks are in \CRASP.} \Cref{res:sparse-crasp} gives a base \CRASP program of counting depth at most two for \TaskSparseStar with fixed selected positions and Boolean function.
\\
\hypertarget{pred-sparse-mpc}{\textbf{24}} & \textbf{Fixed-support Boolean tasks are \MPC-easy.} These are special cases of \Cref{res:index-mpc} in which the selected indices are known in advance. The coordinator therefore skips query collection, requests the selected bits, and applies the fixed Boolean function to the replies. The algorithm uses constant rounds and $O(n^\delta)$ local memory.
\\
\hypertarget{pred-sparse-bapo}{\textbf{25}} & \textbf{Fixed-support Boolean tasks are \BAPO-easy.} \Cref{res:sparse-bapo} gives a $(0,k)$-\BAPO for \TaskSparseStar with fixed arity $k$, selected positions, and Boolean function.
\\
\hypertarget{pred-sparse-sensitivity}{\textbf{26}} & \textbf{Sparse tasks have constant sensitivity.} Only the $k$ selected bits can affect the output, so sensitivity is at most $k$, which is fixed independently of input length.
\\
\hypertarget{pred-addition-hardness}{\textbf{27}} & \textbf{Checking binary addition is \Starfree- and \CRASP-hard.} Consider the decision version of addition: given $x\#y\#z$, where the three numbers are written in binary, decide whether $x+y=z$. A simple pumping lemma argument shows this is nonregular and hence not star-free. Across the equal-length split $x\#0\mid 0\#z$, equality requires $\Omega(n)$ bits of deterministic one-way communication, whereas a fixed \CRASP{}[periodic, local] program gives an $O(\log n)$-bit protocol~\citep[Theorems~9 and~12]{huang2024formal}.
\\
\hypertarget{pred-addition-circuits}{\textbf{28}} & \textbf{Binary addition is in \ACz{}~\citep[Proposition~1.9]{immerman1999descriptive}.} This also makes it \Easy for \TCz{} by containment and for \MPC since constant-round \MPC with $O(n^\delta)$ local memory contains $\NCone$~\citep{roughgarden2018shuffles,frei2019circuit}.
\\
\hypertarget{pred-addition-bapo}{\textbf{29}\textsuperscript{\prednewproof}} & \textbf{Blocked \TaskBinaryAddition verification is \BAPO-hard.} For input $x\#y\#z$ with $n$-bit operands and an $(n+1)$-bit proposed sum, all MSB-first, \Cref{res:addition-bapo} proves that every $(a,b)$-\BAPO requires $a+4b\log_2(4n+1)\geq n-2$.
\\
\hypertarget{pred-addition-sensitivity}{\textbf{30}} & \textbf{\TaskBinaryAddition verification has low uniform average sensitivity.} \Cref{res:addition-sensitivity} gives average sensitivity $\Theta(n/2^n)$ for $n$-bit operands and an $(n+1)$-bit proposed sum. Each valid sum is sensitive to all $3n+1$ bit flips, but valid sums constitute only a $2^{-(n+1)}$ fraction of uniform inputs.
\\
\hypertarget{pred-bounded-dyck-logic}{\textbf{31}} & \textbf{Bounded \TaskDyck is star-free.} For fixed bracket-type count $k$ and depth bound $D$, \TaskDyck{}$(k,D)$ is known to be star-free~\citep[e.g.,][]{strobl2024what}, with classical proofs for $k=1$~\citep{brzozowski1978dot,thomas1984application}, hence also in DLOGTIME-uniform $\ACz$ and $\TCz$.
\\
\hypertarget{pred-dyck-one-crasp}{\textbf{32}} & \textbf{Bounded and unbounded \TaskDyckOne are in \CRASP.} \citet{yang2024counting} provide a \CRASP program for \TaskDyckOne that compares prefix counts of opening and closing brackets, requiring nonnegative height at every prefix and zero final height. This program can be easily extended to require every prefix height to be at most a fixed bound $D$.
\\
\hypertarget{pred-unbounded-dyck-circuits}{\textbf{33}} & \textbf{Unbounded \TaskDyckOne and \TaskDyckTwo are outside \ACz{} but in \TCz.} Both languages are nonregular, hence not star-free, and neither is in $\ACz$~\citep{hao2022formal,yang2024masked}. \TaskDyck with any fixed number of bracket types is in \TCz~\citep{barrington1989relative}.
\\
\hypertarget{pred-unbounded-dyck-mpc}{\textbf{34}} & \textbf{Unbounded \TaskDyck has constant-round \MPC algorithms.} \citet[Theorem~3.3]{gilbert2023brief} decide \TaskDyck languages with a fixed number of bracket types in $O(1/\delta)$ \MPC rounds with $O(n^\delta)$ local memory.
\\
\hypertarget{pred-unbounded-dyck-bapo}{\textbf{35}\textsuperscript{\prednewproof}} & \textbf{Unbounded \TaskDyckOne and \TaskDyckTwo are \BAPO-hard.} \Cref{res:dyck-bapo} proves that recognizing unbounded \TaskDyck with any fixed number of bracket types on inputs of length at most $N$ requires $a+\log_2(b+1)\geq\log_2N-O(1)$.
\\
\hypertarget{pred-bounded-dyck-crasp}{\textbf{36}\textsuperscript{\prednewproof}} & \textbf{Bounded multi-type \TaskDyck is in \CRASP.} \Cref{res:dyck-crasp} gives a base \CRASP program for \TaskDyck{}$(k,D)$ with $O(kD)$ channels and counting depth at most three. For fixed $k,D$, this is a fixed-size program; in particular, it handles bounded \TaskDyckTwo.
\\
\hypertarget{pred-unbounded-dyck-two-crasp}{\textbf{37}} & \textbf{Unbounded \TaskDyckTwo is not in \CRASP~\citep[Remark 21]{huang2024formal}.} Encode two bitstrings as opening-bracket types and closing-bracket types in reverse order. Their concatenation is in \TaskDyckTwo exactly when the strings are equal. This requires linear deterministic one-way communication, whereas a fixed \CRASP{}[periodic, local] program provides an $O(\log n)$-bit protocol by \citet[Theorems~9 and~12]{huang2024formal}.
\\
\hypertarget{pred-fixed-hop-logic}{\textbf{38}\textsuperscript{\prednewproof}} & \textbf{Fixed-hop induction is star-free.} For fixed alphabet $\Sigma$ and hop count $k$, \Cref{res:k-hop-logic} defines each final-position output language, including the undefined output $\bot$, in $\FO[<]$. Thus each is star-free and in DLOGTIME-uniform $\ACz$ and $\TCz$. Their regularity also yields the \BAPO upper bound via \hyperlink{pred-regular-bapo}{\textnormal{5}}.
\\
\hypertarget{pred-fixed-hop-crasp}{\textbf{39}\textsuperscript{\prednewproof}} & \textbf{Fixed-hop induction is \CRASP-hard.} For every fixed $k\geq1$, \Cref{res:k-hop-crasp} shows that a base \CRASP program for \TaskKHop would imply a \CRASP program for $\{a,b,e\}^*be^*b\{a,b,e\}^*$, contradicting \citet[Proposition~31 and Lemma~35]{huang2024formal}.
\\
\hypertarget{pred-fixed-hop-mpc}{\textbf{40}} & \textbf{Fixed-hop induction is \MPC-easy.} \citet[Theorem~4.2]{sanford2024parallel} construct a constant-width transformer for \TaskKHop with depth $\lfloor\log_2 k\rfloor+2$, which is constant for fixed $k$. Their transformer-to-\MPC simulation~\citep[Theorem~3.4]{sanford2024parallel} therefore gives a constant-round protocol with $O(n^\delta)$ local memory for fixed $0<\delta<1$ and polynomial total memory.
\\
\hypertarget{pred-edge-nonregular}{\textbf{41}} & \textbf{\TaskEdgeExistence is not star-free.} Encode node IDs in fixed-width binary over a finite alphabet, with the edge list before the query and an unbounded node range. Any two different edge sets can be distinguished by appending a query for an edge present in only one of them. Thus arbitrarily many graph prefixes require distinct automaton states, so the task is nonregular.
\\
\hypertarget{pred-edge-circuits}{\textbf{42}} & \textbf{Binary-encoded \TaskEdgeExistence is in \ACz.} With fixed-width binary node IDs and fixed edge-block boundaries, compare the queried endpoints with the endpoints of every listed edge and OR the matches. These comparisons give DLOGTIME-uniform, constant-depth, polynomial-size Boolean circuits, hence membership in $\ACz$ and $\TCz$.
\\
\hypertarget{pred-edge-crasp}{\textbf{43}\textsuperscript{\prednewproof}} & \textbf{\TaskEdgeExistence requires growing base \CRASP size.} Use the encoding and input order in \hyperlink{pred-edge-nonregular}{\textnormal{41}}. As in the proof of \Cref{res:index-crasp-growing}, a program with $s$ primitive counts has at most $(L+1)^s$ distinct count vectors after a graph prefix of length $L$. Choosing half of $M$ possible edges gives $2^{\Omega(M)}$ pairwise distinguishable graph prefixes of the same length $L=O(M\log M)$. Their count vectors must be distinct, which is impossible for fixed $s$.
\\
\hypertarget{pred-edge-mpc}{\textbf{44}} & \textbf{\TaskEdgeExistence is \MPC-easy.} Distribute the queried pair to the machines holding the edge list. Each machine checks whether any of its edges matches the query, and an OR reduction combines the local answers. For input size $n$, broadcast and reduction trees use $O(1/\delta)$ rounds and $O(n^\delta)$ local memory.
\\
\hypertarget{pred-edge-bapo}{\textbf{45}} & \textbf{\TaskEdgeExistence is \BAPO-easy with atomic edge tokens.} Encode each directed edge $(u,v)$ as one token and append the queried edge $(s,t)$ as the final atomic token. A $(0,1)$-\BAPO suffices: the attention function selects precisely the prefix edges matching the query, and the suffix oracle also checks the listed edges in the suffix. We conjecture that other encodings (one token per node or binary node encodings) may be \BAPO-hard, but not because of the core lookup mechanism of the task.
\\
\hypertarget{pred-reachability-ac0}{\textbf{46}} & \textbf{Reachability is outside \ACz.} Both directed and undirected reachability are outside \ACz~\citep{wigderson1992complexity} and therefore neither is star-free.
\\
\hypertarget{pred-directed-reachability-tc0}{\textbf{47}} & \textbf{Directed reachability is conditionally hard for \TCz{} and \CRASP.} \TaskDirectedReachability is $\NL$-complete under first-order reductions~\citep[Theorem~3.16]{immerman1999descriptive}. Since DLOGTIME-uniform $\TCz\subseteq\Lclass$, it is outside this circuit class assuming $\Lclass\ne\NL$. \CRASP is contained in DLOGTIME-uniform $\TCz$~\citep[Proposition~31 and Remark~32]{huang2024formal}, so the same assumption gives its \Hard prediction.
\\
\hypertarget{pred-reachability-mpc}{\textbf{48}} & \textbf{Reachability is conditionally \MPC-hard.} Assuming the one-vs-two-cycles conjecture for \MPC~\citep{ghaffari2019conditional}, \TaskUndirectedReachability requires $\Omega( \log n )$ \MPC rounds with sublinear local memory~\citep{sanford2024parallel}. \TaskDirectedReachability is similarly \Hard under the same conjecture, as well as under a weaker conjecture that $\NL \nsubseteq \MPC(o(\log n))$~\citep{nanongkai2022equivalence}.
\\
\hypertarget{pred-reachability-bapo}{\textbf{49}} & \textbf{Reachability is \BAPO-hard~\citep[Theorem~3]{schnabel2025lost}.} The theorem is stated for \TaskDirectedReachability, but the proof directly transfers to \TaskUndirectedReachability as it only uses graphs that are unions of directed paths (which have the same connectivity for queried nodes if they are undirected).
\\
\hypertarget{pred-undirected-reachability-tc0}{\textbf{50}} & \textbf{\TaskUndirectedReachability is conditionally hard for $\TCz$ and \CRASP.} \TaskUndirectedReachability is outside DLOGTIME-uniform $\TCz$ given the conjecture that $\Lclass$-uniform $\NCone \subset \Lclass$~\citep[Section~5.4.2]{strobl2024what}. Since \CRASP is contained in DLOGTIME-uniform $\TCz$, the same assumption gives its \Hard prediction~\citep[Proposition~31 and Remark~32]{huang2024formal}.

%% file: appendices/theory-task-results/theory-task-results.tex
\section{Supporting Theory--Task Results}
\label{app:theory-task-results}
\input{appendices/theory-task-results/macros.tex}

This section gives the new results behind the theory--task predictions in \cref{tab:theory-task-predictions}, including constructions showing that a task is \Easy under a theory, and lower bounds showing that it is \Hard.
\Cref{tab:theory-task-predictions-annotated} lists the result behind each prediction.
The results below are grouped by task family.

\subsection{Definitions and conventions}
\label{app:prediction-conventions}

\subsubsection{Circuits and logic}
\input{appendices/theory-task-results/definitions/circuits.tex}
\subsubsection{\CRASP}
\input{appendices/theory-task-results/definitions/c-rasp.tex}
\subsubsection{\BAPO}

\input{appendices/theory-task-results/definitions/bapo.tex}

\subsection{Fixed-range Index}
\label{app:index-results}
\input{appendices/theory-task-results/index/fixed-task.tex}

\input{appendices/theory-task-results/index/fixed-boolean-task.tex}
\input{appendices/theory-task-results/index/fixed-logic.tex}
\input{appendices/theory-task-results/index/fixed-boolean-logic.tex}
\input{appendices/theory-task-results/index/fixed-c-rasp.tex}
\input{appendices/theory-task-results/index/fixed-boolean-c-rasp.tex}
\input{appendices/theory-task-results/index/fixed-bapo.tex}

\subsection{Growing-range Index}
\label{app:index-growing-range}
\input{appendices/theory-task-results/index/growing-range.tex}
\input{appendices/theory-task-results/index/c-rasp-lower-bound.tex}

\subsection{Sparse Boolean functions}
\label{app:sparse-results}
\input{appendices/theory-task-results/sparse/task.tex}
\input{appendices/theory-task-results/sparse/logic.tex}
\input{appendices/theory-task-results/sparse/c-rasp.tex}
\input{appendices/theory-task-results/sparse/bapo.tex}

\subsection{Binary addition}
\label{app:addition-results}
\input{appendices/theory-task-results/addition/bapo.tex}
\input{appendices/theory-task-results/addition/sensitivity.tex}

\subsection{Dyck languages}
\label{app:dyck-results}
\input{appendices/theory-task-results/dyck/c-rasp.tex}
\input{appendices/theory-task-results/dyck/bapo-lower-bound.tex}
\input{appendices/theory-task-results/k-hop/logic.tex}

\input{appendices/theory-task-results/k-hop/c-rasp.tex}

%% file: appendices/theory-task-results/macros.tex
\providecommand{\bits}{\{0,1\}}
\providecommand{\indicator}{\mathbf{1}}
\providecommand{\BAPO}{\mathsf{BAPO}}
\providecommand{\FObit}{\FO[<,\mathrm{BIT}]}

%% file: appendices/theory-task-results/definitions/circuits.tex
A language is in \(\ACz\) if it is decided by a family of Boolean circuits $\{C_n\}_n$ with polynomial size and constant depth (independent of $n$), using unbounded-fan-in AND and OR gates and fan-in-one NOT gates.
A language is in \(\TCz\) if the circuits may additionally use unbounded-fan-in majority gates.
For a fixed finite alphabet $\Sigma$, we encode each symbol with $\lceil \log_2 |\Sigma| \rceil$ bits.
This changes circuit size and depth by at most a constant factor, so it does not affect membership in \ACz or \TCz.
A \emph{uniformity} condition imposes that the same algorithm can generate the circuit \(C_n\) for each input length \(n\).
We consider DLOGTIME-uniform \(\ACz\), where a deterministic random-access Turing machine can determine a gate's type and whether two gates are connected in time \(O(\log n)\).

Fix a finite alphabet \(\Sigma\).
A word \(w=w_0\cdots w_{N-1}\) is represented by a structure with domain
\[
[N]_0:=\{0,\ldots,N-1\}.
\]
Variables range over positions in \([N]_0\).
The atomic predicates are
\[
Q_\sigma(p),\qquad p=q,\qquad p<q,\qquad \mathrm{BIT}(p,b),
\]
where \(Q_\sigma(p)\) means that position \(p\) contains the symbol \(\sigma\), and \(\mathrm{BIT}(p,b)\) means that bit \(b\) of the binary representation of \(p\) is \(1\).
Because the structure represents a word, exactly one predicate \(Q_\sigma\) holds at each position.
\emph{Formulas} are built from these predicates using Boolean connectives and the quantifiers \(\exists\) and \(\forall\).
A \emph{sentence} is a formula that has no free variables and defines the words on which it is true.

From a descriptive complexity perspective, DLOGTIME-uniform \(\ACz\) can be characterized as the class of languages definable by first-order logic with the BIT predicate~\citep{immerman1999descriptive}:
\[
\FObit
=
\text{DLOGTIME-uniform }\ACz.
\]
Thus, a language is in DLOGTIME-uniform \(\ACz\) if and only if it can be defined by a first-order sentence using the predicates \(Q_\sigma(p)\), \(p=q\), \(p<q\), and \(\mathrm{BIT}(p,b)\). This gives a way to reason about DLOGTIME-uniform \(\ACz\) without explicitly constructing circuits.

The fragment \(\FO[<]\), which omits \(\mathrm{BIT}\), defines exactly the star-free languages~\citep{immerman1999descriptive}.
We will use the inclusions
\begin{equation}\label{eq:fo-inclusions}
    \FO[<]
=
\Starfree
\subset
\text{DLOGTIME-uniform }\ACz
\subset
\text{DLOGTIME-uniform }\TCz.
\end{equation}

Thus, an \(\FO[<]\) representation immediately shows a problem is in all three classes.

%% file: appendices/theory-task-results/definitions/c-rasp.tex
Following \citet{yang2024counting,yang2025kneedeep}, a \CRASP program is a straight-line sequence of named Boolean streams \(P(i)\) and integer-valued count streams \(C(i)\).
The primitives in a \CRASP program are:
\[
\begin{array}{ll}
\text{Boolean:}
&Q_\sigma(i),\quad 1,\quad \neg P(i),\quad
P_1(i)\land P_2(i),\quad C_1(i)\leq C_2(i),\\[2mm]
\text{Count:}
&1,\quad \#[j\leq i:P(j)],\quad C_1(i)+C_2(i),\quad
C_1(i)-C_2(i).
\end{array}
\]

The definition above omits the conditional and max/min operators included in the \CRASP definition by \citet{yang2024counting}, since they can be constructed from these primitives.
Max/min can be replaced by conditionals, e.g., $\max(C_1(i), C_2(i))$ is equivalent to $C_1(i) \geq C_2(i)\;?\;C_1(i):C_2(i)$.
Conditionals can then be replaced by Boolean operators, by rewriting each comparison that contains a conditional as a Boolean combination of comparisons, without increasing depth~\citep[Lemma A.3]{yang2025kneedeep}.
Other operations can also be derived from these base primitives, e.g., disjunction from negation and conjunction, and equality from $\leq$ and $\land$.
A program accepts by reading its final Boolean stream at the last input position.
The ``counting depth'' is the largest number of counting operations along a dependency chain in the program.

In the notation of \citet{huang2024formal}, this base version of \CRASP is also called \CRASP{}[$\varnothing$], to distinguish it from enhanced versions such as \CRASP{}[periodic, local].

%% file: appendices/theory-task-results/definitions/bapo.tex
Following \citet{schnabel2025lost}, an \((a,b)\text{-}\BAPO\) consists of three functions
\[
\begin{aligned}
f&:\Sigma^*\to\bits^a,\\
g&:\Sigma^*\times\mathbb N\times\Sigma\times\mathbb N\to\bits,\\
h&:\bits^a\times
\bigcup_{q=0}^b(\Sigma\times\mathbb N)^q
\times\Sigma^*\times\mathbb N\to\bits.
\end{aligned}
\]
The function $f$ is called the \emph{prefix oracle}, $g$ the \emph{attention function}, and $h$ the \emph{suffix oracle}.
Intuitively, $f$ summarizes the prefix in $a$ bits, $g$ selects which prefix tokens the suffix attends to, and $h$ computes the answer from the summary, at most $b$ attended prefix tokens, and the suffix.
For an input \(s=s_1\cdots s_N\) and split index \(1\leq t<N\), let
\[
\mathbb G_t
=
\{(s_j,j):1\leq j\leq t,\,
g(s_{t+1:N},t,s_j,j)=1\}.
\]
The \BAPO computes \(p\) when, for every split and every attended set \(G\subseteq\mathbb G_t\) with \(\lvert G\rvert=\min\{b,\lvert\mathbb G_t\rvert\}\),
\[
h(f(s_{1:t}),G,s_{t+1:N},t)=p(s).
\]

See \citet{schnabel2025lost} for a full introduction to the model.

%% file: appendices/theory-task-results/index/fixed-task.tex
We use the fixed-range tasks \(\text{\TaskIndex}_L\) and \(\text{\TaskIndexStar}_{L,F,k}\) from \Cref{def:task-index}. Fix an index bound \(L\geq 1\) and distinct index symbols \(\overline{1},\ldots,\overline{L}\), disjoint from the data alphabet \(\bits\).

For \(\text{\TaskIndex}_L\), valid inputs have the form
\[
x_1\cdots x_n\overline r,
\qquad
x_j\in\bits,
\qquad
1\leq r\leq\min\{n,L\}.
\]

%% file: appendices/theory-task-results/index/fixed-boolean-task.tex
For \(\text{\TaskIndexStar}_{L,F,k}\), valid inputs have the form
\[
x_1\cdots x_n\overline{i_1}\cdots\overline{i_k},
\qquad
x_j\in\bits,
\qquad
i_r\in[\min\{n,L\}].
\]
The bit and index alphabets are disjoint, so no separator is needed.
Indices may repeat.

%% file: appendices/theory-task-results/index/fixed-logic.tex
\begin{proposition}[Fixed-range \(\text{\TaskIndex}_L\) in \Starfree]
\label{res:index-fixed-logic}
For every fixed index bound \(L\), \(\text{\TaskIndex}_L\) is representable in \(\FO[<]\) on valid inputs, and is therefore star-free and in DLOGTIME-uniform \(\ACz\) and \(\TCz\).
\end{proposition}

\begin{proof}
\medskip We prove this by constructing an \(\FO[<]\) sentence that, on every valid input, is true exactly when \(\text{\TaskIndex}_L\) outputs \(1\).

\medskip
To slightly simplify the presentation of the proof, let's relabel the data bits and index symbols to use \(0\)-based indexing:
\[
x_0\cdots x_{n-1}\overline r,
\qquad
0\leq r<\min\{n,L\}.
\]
The output is \(x_r\).
This is more aligned with the convention for positions in \(\FO[<]\), which are \(0\)-based.
Now \(x_\ell\) occurs at position \(\ell\), and the final index token occurs at position \(n\).

\medskip
\noindent\textbf{Position predicates.}
The main step is to define, for each fixed \(\ell\in\{0,\ldots,L-1\}\), a first-order formula that names position \(\ell\).
We achieve this by first defining some basic predicates that allow us to manipulate positions in \([N]_0\):
\[
\begin{aligned}
\mathrm{min}(p)
&:=\neg\exists q\,(q<p),\\
\mathrm{max}(p)
&:=\neg\exists q\,(p<q),\\
\mathrm{SUCC}(p,q)
&:=p<q\land\neg\exists s\,(p<s\land s<q).
\end{aligned}
\]
In words, the successor predicate \(\mathrm{SUCC}(p,q)\) holds when \(p < q\) and no other \(s\) is between them.
For each fixed nonnegative integer \(m\), define \(\mathrm{Eq}_m(q)\) as the predicate meaning \(q=m\), by
\[
\mathrm{Eq}_0(q):=\mathrm{min}(q),
\qquad
\mathrm{Eq}_{m+1}(q):=
\exists p\,\bigl(\mathrm{Eq}_m(p)\land\mathrm{SUCC}(p,q)\bigr).
\]
Each \(\mathrm{Eq}_m\) is shorthand for a separate finite first-order formula, obtained by expanding the recursive definition \(m\) times.
For example, \(\mathrm{Eq}_2(q)\) says that \(q\) is two successor steps after the first position.
Since \(L\) is fixed, the construction needs only the fixed finite list \(\mathrm{Eq}_0,\ldots,\mathrm{Eq}_{L-1}\), independent of the input length \(n\).

\medskip
\noindent\textbf{Sentence.}
On valid inputs, the positive instances of \(\text{\TaskIndex}_L\) are defined by
\[
\varphi_L
:=
\exists p\left(
\mathrm{max}(p)
\land
\bigvee_{\ell=0}^{L-1}
\left[
Q_{\overline\ell}(p)
\land
\exists q\left(
\mathrm{Eq}_{\ell}(q)
\land Q_1(q)
\right)
\right]
\right)
\]

\medskip
The idea is the following: the sentence is an OR over the allowed indices \(\ell=0,\ldots,L-1\).
For each \(\ell\), \(Q_{\overline\ell}(p)\) checks that the final token (which is the index) is \(\overline\ell\), the \(\mathrm{Eq}_{\ell}(q)\) predicate checks that \(q\) is at position \(\ell\), and \(Q_1(q)\) checks that the data bit there is \(1\). The formula \(\exists q\left(
\mathrm{Eq}_{\ell}(q)
\land Q_1(q)
\right)\) checks whether there is a position \(q\) matching \(\ell\) that contains the bit \(1\).

More formally: fix a valid input \(x_0\cdots x_{n-1}\overline r\).
The condition \(\mathrm{max}(p)\) requires \(p=n\), the final position.
Since that position contains \(\overline r\), only the disjunct \(\ell=r\) can hold.
Within this disjunct, \(\mathrm{Eq}_{r}(q)\) requires \(q=r\), the position containing \(x_r\).
The remaining condition \(Q_1(q)\) holds exactly when \(x_r=1\).
Therefore, on valid inputs, \(\varphi_L\) is true exactly when \(\text{\TaskIndex}_L\) outputs \(1\).

Since \(L\) is fixed, the displayed disjunction has fixed size and \(\varphi_L\) is one fixed \(\FO[<]\) sentence.
The \(\FO[<]\)--\Starfree correspondence and the inclusions in \Cref{eq:fo-inclusions} prove the claim.
\end{proof}

%% file: appendices/theory-task-results/index/fixed-boolean-logic.tex
\begin{proposition}[Fixed-range \(\text{\TaskIndexStar}_{L,F,k}\) in \Starfree]
\label{res:index-bool-fixed-logic}
For every fixed \(L\), \(k\), and \(F\), \(\text{\TaskIndexStar}_{L,F,k}\) is representable in \(\FO[<]\) on valid inputs, and is therefore star-free and in DLOGTIME-uniform \(\ACz\) and \(\TCz\).
\end{proposition}

\begin{proof}
Let \(\widehat F(z_1,\ldots,z_k)\) be a fixed Boolean formula computing \(F\).
As in \Cref{res:index-fixed-logic}, relabel the data bits and index symbols to use \(0\)-based indexing.
A valid input then has the form
\[
x_0\cdots x_{n-1}\overline{r_1}\cdots\overline{r_k},
\qquad
0\leq r_j<\min\{n,L\},
\]
and the output is \(F(x_{r_1},\ldots,x_{r_k})\).
The total input length is \(N=n+k\).
The data bit \(x_\ell\) occurs at position \(\ell\), and the index token \(\overline{r_j}\) occurs at position \(n+j-1\).
We construct one \(\FO[<]\) formula \(Y_j\) for each index token.
The truth value of \(Y_j\) will be the selected bit \(x_{r_j}\).
We then apply \(\widehat F\) to these \(k\) formulas.

\medskip
\noindent\textbf{Position predicates.}
We first identify the \(j\)-th index token by its fixed distance from the end of the input.
Using \(\mathrm{max}\) and \(\mathrm{SUCC}\) from the proof of \Cref{res:index-fixed-logic}, define
\[
\begin{aligned}
\mathrm{Slot}_k(p)
&:=\mathrm{max}(p),\\
\mathrm{Slot}_j(p)
&:=
\exists u\left(
\mathrm{SUCC}(p,u)
\land \mathrm{Slot}_{j+1}(u)
\right),
\qquad 1\leq j<k.
\end{aligned}
\]
Each \(\mathrm{Slot}_j\) is shorthand for a separate finite first-order formula, obtained by expanding the definition from \(\mathrm{Slot}_k\).
The predicate \(\mathrm{Slot}_k\) identifies the final position \(n+k-1\), and each preceding predicate moves one position to the left.
Therefore \(\mathrm{Slot}_j(p)\) holds exactly when \(p=n+j-1\), the position containing \(\overline{r_j}\).

For each \(j\in[k]\), define
\[
Y_j
:=
\bigvee_{\ell=0}^{L-1}
\left[
\exists p\left(
\mathrm{Slot}_j(p)
\land Q_{\overline\ell}(p)
\right)
\land
\exists q\left(
\mathrm{Eq}_{\ell}(q)
\land Q_1(q)
\right)
\right].
\]
This is applying the \(\text{\TaskIndex}_L\) construction from the proof of \Cref{res:index-fixed-logic} for the \(j\)-th index token, to extract the bit \(x_{r_j}\) selected by that token.

\medskip
\noindent\textbf{Sentence.}
On valid inputs, the positive instances of \(\text{\TaskIndexStar}_{L,F,k}\) are defined by
\[
\chi_{L,k,F}
:=
\widehat F(Y_1,\ldots,Y_k)
\]
where each input variable \(z_j\) in \(\widehat F\) is replaced by the first-order formula \(Y_j\).

To see why this sentence correctly defines the positive instances of \(\text{\TaskIndexStar}_{L,F,k}\), fix a valid input
\[
x_0\cdots x_{n-1}\overline{r_1}\cdots\overline{r_k}.
\]
For each \(j\), \(\mathrm{Slot}_j(p)\) requires \(p=n+j-1\), the position containing \(\overline{r_j}\).
Because the input is valid, \(r_j<L\), so the disjunct \(\ell=r_j\) appears in \(Y_j\).
No other disjunct can hold because the index-token position contains exactly one symbol.
Validity also gives \(r_j<n\), so position \(r_j\) exists and contains \(x_{r_j}\).
Within the surviving disjunct, \(\mathrm{Eq}_{r_j}(q)\) requires \(q=r_j\), and \(Q_1(q)\) is true exactly when \(x_{r_j}=1\).
Thus \(Y_j\) has truth value \(x_{r_j}\).
It follows that \(\chi_{L,k,F}\) has truth value \(F(x_{r_1},\ldots,x_{r_k})\), including when indices repeat.
Therefore, on valid inputs, \(\chi_{L,k,F}\) is true exactly when \(\text{\TaskIndexStar}_{L,F,k}\) outputs \(1\).

Since \(L\), \(k\), and \(F\) are fixed, all displayed disjunctions and formulas have fixed size.
Thus \(\chi_{L,k,F}\) is one fixed \(\FO[<]\) sentence.
The \(\FO[<]\)--\Starfree correspondence and the inclusions from \Cref{eq:fo-inclusions} prove the claim.
\end{proof}

\noindent\textbf{Remark (malformed inputs).}
These proofs assume the valid input formats defined at the start of this subsection.
If malformed inputs are admitted, first define an \(\FO[<]\) sentence \(\mathrm{Valid}\) that checks the corresponding format, then conjoin \(\mathrm{Valid}\) with the task sentence above.
Because $L$ and $k$ are fixed, these format checks are expressible in \(\FO[<]\).

%% file: appendices/theory-task-results/index/fixed-c-rasp.tex
\begin{proposition}[Fixed-range \(\text{\TaskIndex}_L\) in \CRASP]
\label{res:index-fixed-crasp}
For every fixed index bound \(L\), \(\text{\TaskIndex}_L\) has a \(\CRASP\) program of counting depth at most two.
\end{proposition}

\begin{proof}
First identify payload bits and count how many have appeared:
\[
\begin{aligned}
\operatorname{Bit}(i)
&:=Q_0(i)\lor Q_1(i),\\
N(i)
&:=\#[j\leq i:\operatorname{Bit}(j)].
\end{aligned}
\]
At each payload position \(i\), \(N(i)=i\).

For each \(\ell\in[L]\), use \(N\) to mark payload position \(\ell\):
\[
A_\ell(i)
:=\operatorname{Bit}(i)\land(N(i)=\ell).
\]
Thus \(A_\ell\) is true exactly at position \(\ell\), when that position exists.

The second count retrieves the bit at \(\ell\), and \(B_\ell\) converts the count to a Boolean:
\[
\begin{aligned}
C_\ell(i)
&:=\#[j\leq i:Q_1(j)\land A_\ell(j)],\\
B_\ell(i)
&:=(1\leq C_\ell(i)).
\end{aligned}
\]
At the final index token, \(B_\ell=x_\ell\) for every payload position \(\ell\leq n\).

Finally, use the index token to select the requested bit:
\[
D(i)
:=
\bigvee_{\ell=1}^{L}
\left(Q_{\overline\ell}(i)\land B_\ell(i)\right).
\]
If the input ends in \(\overline r\), exactly \(Q_{\overline r}\) is true, so \(D(n)=x_r\).
The stream \(N\) has depth one, each \(C_\ell\) has depth two, and all remaining operations are Boolean.
\end{proof}

%% file: appendices/theory-task-results/index/fixed-boolean-c-rasp.tex
\begin{proposition}[Fixed-range \(\text{\TaskIndexStar}_{L,F,k}\) in \CRASP]
\label{res:index-bool-fixed-crasp}
For every fixed \(L\), \(k\), and \(F\), \(\text{\TaskIndexStar}_{L,F,k}\) has a \(\CRASP\) program of counting depth at most two.
\end{proposition}

\begin{proof}
Let \(\widehat F\) be a fixed Boolean formula computing \(F\). Define \(\operatorname{Bit}(i),N(i),A_\ell(i),C_\ell(i),\) and \(B_\ell(i)\) as in the proof of \Cref{res:index-fixed-crasp}. Additionally, identify and count index tokens:
\[
\begin{aligned}
\operatorname{Index}(i)
&:=\bigvee_{\ell=1}^{L}Q_{\overline\ell}(i),\\
R(i)
&:=\#[j\leq i:\operatorname{Index}(j)].
\end{aligned}
\]

For every query slot \(r\in[k]\) and index \(\ell\in[L]\), use \(R\) to test whether slot \(r\) contains \(\overline\ell\):
\[
J_{r,\ell}(i)
:=
\left(
1\leq
\#[j\leq i:
Q_{\overline\ell}(j)\land(R(j)=r)]
\right).
\]
At the final position, exactly \(J_{r,i_r}\) is true for each slot \(r\).

Use these selectors to retrieve each requested bit, then apply \(F\):
\[
\begin{aligned}
Y_r(i)
&:=
\bigvee_{\ell=1}^{L}
\left(J_{r,\ell}(i)\land B_\ell(i)\right),
\qquad 1\leq r\leq k,\\
D(i)
&:=\widehat F(Y_1(i),\ldots,Y_k(i)).
\end{aligned}
\]
Thus, at the final position, \(Y_r=x_{i_r}\) and \(D=F(x_{i_1},\ldots,x_{i_k})\).
The streams \(N\) and \(R\) have depth one, each \(C_\ell\) and \(J_{r,\ell}\) has depth two, and all remaining operations are Boolean.
\end{proof}

%% file: appendices/theory-task-results/index/fixed-bapo.tex
\begin{proposition}[Fixed-range \(\text{\TaskIndexStar}_{L,F,k}\) is \BAPO-easy]
\label{res:bapo-index}
For every fixed \(L\), \(k\), and \(F\), \(\text{\TaskIndexStar}_{L,F,k}\) has a \((k-1,k)\text{-}\BAPO\).
\end{proposition}

\begin{proof}
Consider a valid input of total length \(N=n+k\):
\[
s=x_1\cdots x_n\overline{i_1}\cdots\overline{i_k}.
\]
Let \(t\) be the split index and let \(r\) be the number of index tokens in \(s_{1:t}\).
If \(t\leq n\), then \(r=0\) and \(s_{1:t}=x_1\cdots x_t\).
If \(t>n\), then \(r=t-n\) and \(s_{1:t}=x_1\cdots x_n\overline{i_1}\cdots\overline{i_r}\).
Since \(t<N=n+k\), we have \(r\leq k-1\).
Define the prefix oracle \(f\) to send the bits that are selected by the index tokens in the prefix, padded with zeros to length \(k-1\):
\[
f(s_{1:t})
=
\begin{cases}
0^{k-1}
& \text{if the prefix contains no index token},\\
(x_{i_1},\ldots,x_{i_r},0^{k-1-r})
& \text{otherwise}.
\end{cases}
\]
For each prefix token \(s_j\), define the attention function \(g\) to select it if its index occurs in the suffix:
\[
g(s_{t+1:N},t,s_j,j)
=
\indicator[
s_j\in\bits,\,
j\leq L,\,
\overline j\text{ occurs in }s_{t+1:N}
].
\]
Thus \(\mathbb G_t\) contains the (distinct) prefix positions requested by \(\overline{i_{r+1}},\ldots,\overline{i_k}\), so \(\lvert\mathbb G_t\rvert\leq k-r\leq k=b\).
Therefore we always have \(G=\mathbb G_t\).

The suffix oracle $h$ counts the index tokens in $s_{t+1:N}$ to recover $r$. It reads the first $r$ requested values from $f(s_{1:t})$. For each remaining index $i_\ell$, it reads $x_{i_\ell}$ from $G$ when $i_\ell\leq t$ and from suffix position $i_\ell-t$ when $i_\ell>t$. It then outputs $F(x_{i_1},\ldots,x_{i_k})$. If an index is repeated, $h$ reuses the same bit. For any valid input and any split index, this BAPO can thus compute \(\text{\TaskIndexStar}_{L,F,k}\).
\end{proof}

%% file: appendices/theory-task-results/index/growing-range.tex
The results in \Cref{app:index-results} use a fixed index range. We now consider \TaskIndex and \TaskIndexStar from \Cref{def:task-index} without that restriction. We use binary-encoded indices to keep the alphabet finite as the range grows with the data length.
As before, the data are binary and indices are \(1\)-based.
Instead of one symbol per index, use two query-bit symbols \(\overline{0},\overline{1}\), distinct from the data symbols \(0,1\).
For \(n\geq1\), let \(\ell=\lceil\log_2(n+1)\rceil\), and let \(\operatorname{bin}_{\ell}(r)\) be the \(\ell\)-bit, least-significant-bit-first binary representation of \(r\), padded with zeros and written using query bits.
Valid inputs for \TaskIndex have the form
\[
 x_1\cdots x_n\operatorname{bin}_{\ell}(r),
 \qquad x_j\in\bits,
 \qquad 1\leq r\leq n.
\]
For fixed \(k\geq1\) and \(F:\bits^k\to\bits\), valid inputs for \(\text{\TaskIndexStar}_{F,k}\) have the form
\[
 x_1\cdots x_n
 \operatorname{bin}_{\ell}(i_1)\cdots\operatorname{bin}_{\ell}(i_k),
 \qquad x_j\in\bits,
 \qquad i_r\in[n].
\]
Indices may repeat.
The disjoint data and query alphabets identify the start of the query suffix, and its \(k\) equal-width blocks identify the individual indices, so no separators are needed.
The alphabet \(\{0,1,\overline{0},\overline{1}\}\) is fixed and independent of \(n\).
Unlike the native atomic-address encoding, this encoding exposes the bits of each index; the circuit result below is for this binary encoding.

\begin{proposition}[Growing-range \TaskIndex and \TaskIndexStar in uniform \(\ACz\)]
\label{res:index-fo-bit}
For every fixed \(k\geq1\) and \(F:\bits^k\to\bits\), growing-range \(\text{\TaskIndexStar}_{F,k}\) is representable in \(\FO[<,\operatorname{BIT}]\) on valid inputs.
In particular, this holds for \TaskIndex, \TaskIndexAnd, \TaskIndexOr, and \TaskIndexXor.
Consequently, these tasks have DLOGTIME-uniform \(\ACz\) and \(\TCz\) circuits on valid inputs.
\end{proposition}

\begin{proof}
The universe is the set of \(0\)-based input positions, and \(\operatorname{BIT}(u,b)\) holds when bit \(b\) of \(u\) is one.
Use \(\mathrm{max}\) and \(\mathrm{SUCC}\) from \Cref{res:index-fixed-logic}, and let \(\operatorname{PLUS}(u,v,w)\) express the equality \(u+v=w\), which is definable in \(\FO[<,\operatorname{BIT}]\)~\citep{immerman1999descriptive}.

\medskip
\noindent\textbf{Query blocks.}
Let \(\operatorname{Query}(p):=Q_{\overline{0}}(p)\lor Q_{\overline{1}}(p)\).
Define predicates that identify the first query bit and partition the query suffix into \(k\) consecutive, equal-length blocks with start positions \(\mathbf s=(s_1,\ldots,s_k)\) and end positions \(\mathbf e=(e_1,\ldots,e_k)\):
\[
\begin{aligned}
\operatorname{FirstQuery}(s)
&:=\operatorname{Query}(s)\land
\neg\exists z\,(z<s\land\operatorname{Query}(z)),\\
\operatorname{Blocks}_k(\mathbf s,\mathbf e)
&:=\operatorname{FirstQuery}(s_1)\land\mathrm{max}(e_k)\\
&\quad{}\land\exists d\,\bigwedge_{j=1}^k
\operatorname{PLUS}(s_j,d,e_j)
\land\bigwedge_{j=1}^{k-1}\mathrm{SUCC}(e_j,s_{j+1}).
\end{aligned}
\]
On a valid input, $\operatorname{Blocks}_k(\mathbf s,\mathbf e)$ holds when \(\mathbf s\) and \(\mathbf e\) partition the query suffix into \(k\) consecutive blocks of equal length \(d+1\), where positions \(s_j\) through \(e_j\), inclusive, contain the encoding of the \(j\)th query index \(i_j\).
Since the suffix has length \(k\ell\), necessarily \(d+1=\ell\); thus the endpoints are uniquely determined.
The conjunctions have fixed size because \(k\) is fixed.

\medskip
\noindent\textbf{Address matching.}
The following formula says that the block from \(s\) through \(e\) encodes the integer \(r\):
\[
\begin{aligned}
\operatorname{Addr}(r,s,e)
:=\forall b\,\bigl[
&\operatorname{BIT}(r,b)\leftrightarrow\\
&\exists z\,(\operatorname{PLUS}(s,b,z)\land z\leq e
\land Q_{\overline{1}}(z))\bigr].
\end{aligned}
\]
This matches bit \(b\) of \(r\) with the query bit at position \(z=s+b\), for every $b$.
When \(b>e-s\), the right side is false, requiring all higher bits of \(r\) to be zero.
We can then use this to retrieve the data bit at a queried position. The data bit \(x_r\) occurs at position \(r-1\), so the retrieved bit is defined by
\[
Y_j(s_j,e_j)
:=\exists p\,\exists r\,
\bigl(Q_1(p)\land\mathrm{SUCC}(p,r)
\land\operatorname{Addr}(r,s_j,e_j)\bigr).
\]
The integer \(r\leq n\) is in the universe even when \(r=n\), since at least one query digit follows the data.
Thus, for the uniquely determined block endpoints, \(Y_j\) holds exactly when \(x_{i_j}=1\).

\medskip
\noindent\textbf{Sentence.}
Let \(\widehat F\) be a fixed Boolean formula computing \(F\).
The desired sentence is
\[
\varphi_{k,F}
:=\exists s_1\,\exists e_1\cdots\exists s_k\,\exists e_k\,
\bigl(\operatorname{Blocks}_k(\mathbf s,\mathbf e)
\land\widehat F(Y_1(s_1,e_1),\ldots,Y_k(s_k,e_k))\bigr).
\]
For \TaskIndex, take \(k=1\) and \(F(b)=b\).
For \TaskIndexAnd, \TaskIndexOr, and \TaskIndexXor, use conjunction, disjunction, and parity, respectively.
The sentence has fixed size, so the standard \(\FO[<,\operatorname{BIT}]\) characterization gives DLOGTIME-uniform \(\ACz\), and hence \(\TCz\).
As in the fixed-range results, only agreement on valid inputs is required; the sentence defines a total language with unrestricted behavior on malformed inputs.
\end{proof}

\begin{proposition}[Growing-range \TaskIndex is not regular]
\label{res:index-nonregular}
Growing-range \TaskIndex is not computable by any finite-state machine, even when correctness is required only on valid inputs.
For every fixed \(k\geq1\), growing-range \TaskIndexAnd, \TaskIndexOr, and \TaskIndexXor are likewise not finite-state computable.
In particular, their valid-input behavior has no star-free realization.
\end{proposition}

\begin{proof}
Fix \(n\).
Any two distinct data strings \(x,y\in\bits^n\) differ at some position \(r\in[n]\).
Appending the same valid suffix \(\operatorname{bin}_{\ell}(r)\) makes \TaskIndex output \(x_r\) on one string and \(y_r\) on the other.
Therefore a machine correct on all valid inputs must reach distinct states after all \(2^n\) length-\(n\) data prefixes.
No fixed finite-state machine can do so for every \(n\).

For a fixed \(k\), reserve \(k-1\) data positions and set them to the neutral element of the operation: one for AND and zero for OR or XOR. Let all remaining positions vary.
Given two different strings in this restricted family, choose a differing, nonreserved position \(r\), query \(r\) together with the \(k-1\) reserved positions, and use each address once.
The two outputs differ.
Thus there are at least \(2^{n-(k-1)}\) distinguishable length-\(n\) data prefixes, again ruling out a finite-state machine.
The argument uses only that every address \(r\in[n]\) has some valid suffix representation; it does not depend on the particular binary syntax.
\end{proof}

\begin{proposition}[\BAPO bounds for growing-range \TaskIndex with atomic addresses]
\label{res:index-bapo-growing}
In the native encoding with one atomic token \(\overline r\) for every available address \(r\), growing-range \TaskIndex has a \((0,1)\)-\BAPO. For every fixed \(k\) and \(F:\bits^k\to\bits\), growing-range \(\text{\TaskIndexStar}_{F,k}\) has a \((k-1,k)\)-\BAPO.
\end{proposition}

\begin{proof}
Use the constructions of \Cref{res:bapo-index}. The bandwidth depends on the number of query tokens, not on the size of the atomic address vocabulary. The proof of \Cref{res:bapo-index} therefore applies for every growing address range and give bandwidths \((0,1)\) and \((k-1,k)\) respectively. Note that this result uses atomic address tokens. For the binary-block encoding above, a \BAPO split can fall inside an address and separate its bits.
\end{proof}

\begin{proposition}[Three-round \MPC protocol for growing-range \TaskIndex]
\label{res:index-mpc}
Fix \(0<\delta<1\).
Under the indexed block input convention of \citet{sanford2024parallel}, growing-range \TaskIndex and \(\text{\TaskIndexStar}_{F,k}\) with fixed \(k,F\) have a three-round \MPC protocol with \(O(n^\delta)\) words of local memory per machine and polynomial total memory.
\end{proposition}

\begin{proof}
Let \(N=n+k\ell=\Theta(n)\) be the total input length and \(B=\lceil N^\delta\rceil\) the block size.
As in \citet{sanford2024parallel}, each input symbol carries its position, and the symbol at \(0\)-based position \(p\) is initially on machine \(1+\lfloor p/B\rfloor\).
Each word has \(O(\log N)\) bits.

In the first round, every machine sends its query bits, identified by the symbols \(\overline{0},\overline{1}\), together with their positions, to machine \(1\), retaining the data bits locally.
Machine \(1\) orders the query bits by position and decodes the \(k\) equal-width indices.
In the second round, it sends a request tagged by \(j\) for each index \(i_j\) to machine \(1+\lfloor(i_j-1)/B\rfloor\).
In the third round, each owner returns the requested bits with their query tags.
Machine \(1\) then locally outputs the retrieved bit, or applies \(F\) to the returned bits in query order.

The first round gathers \(k\ell=O(\log(n+1))\) words for fixed \(k\).
The next two rounds each communicate \(O(k)\) constant-size messages.
Together with the \(O(B)\) words of retained data per machine, these satisfy \(O(n^\delta)\) local memory and per-round communication bounds.
The input blocks use \(O(n)\) total words; any additional machines allowed by the model can remain idle.
\end{proof}

\textbf{Remark.} If the indexed input records are instead initially distributed arbitrarily, subject to the same local memory bound, one preliminary round can route each record at position \(p\) to machine \(1+\lfloor p/B\rfloor\).
Each sender holds \(O(B)\) words and each destination receives at most \(B\) records, so this establishes the required layout within the same asymptotic bounds.
Thus unknown initial owners cost at most one additional round, provided the records retain their input positions.

%% file: appendices/theory-task-results/index/c-rasp-lower-bound.tex
\begin{proposition}[Growing-range \TaskIndex requires growing \CRASP size]
\label{res:index-crasp-growing}
Fix \(n\geq1\), and let \(s\) be the number of primitive prefix-count nodes in a base \CRASP program.
If the program answers every growing-range \TaskIndex query \(r\in[n]\) on every binary data prefix \(x_1\cdots x_n\), then
\[
    s\geq \frac{n}{\log_2(n+1)}.
\]
For fixed \(k\geq1\) and \(n\geq k\), any such program computing \TaskIndexAnd, \TaskIndexOr, or \TaskIndexXor with \(k\) distinct query indices satisfies
\[
    s\geq \frac{n-k+1}{\log_2(n+1)}.
\]
Consequently, none of these growing-range tasks has a fixed-size base \CRASP program valid for all \(n\).
\end{proposition}

\begin{proof}
Fix \(n\), and write the primitive prefix counts as \(C_1,\ldots,C_s\).
After reading a data prefix \(x=x_1\cdots x_n\in\bits^n\), record
\[
    \tau(x):=(C_1(n),\ldots,C_s(n))\in\{0,\ldots,n\}^s.
\]
For this fixed \(n\), the vector \(\tau(x)\) determines the computation on any given query suffix.
To see this, process the suffix symbols in order and the program nodes in topological order at each position.
All non-counting operations are pointwise, so their values are determined by the current symbol and earlier nodes at that position.
If \(C_t\) counts predicate \(P_t\), then
\[
    C_t(i)=C_t(i-1)+\mathbf{1}[P_t(i)].
\]
The first term is its saved count, and \(P_t(i)\) depends only on earlier program nodes at position \(i\).
Thus all values on the suffix are determined inductively, even when counts depend on other counts.
There is no need to save Boolean streams or arithmetic combinations of counts separately.

If \(x,y\in\bits^n\) differ, choose \(r\in[n]\) with \(x_r\neq y_r\) and append the same query suffix \(\operatorname{bin}_{\ell}(r)\) to both prefixes.
Since the prefixes have the same length, this suffix is valid for both.
Correctness requires different outputs, so \(\tau(x)\neq\tau(y)\).
There are \(2^n\) possible data prefixes and at most \((n+1)^s\) count vectors.
Hence \((n+1)^s\geq2^n\), proving the first bound.

For \TaskIndexAnd, \TaskIndexOr, or \TaskIndexXor, reserve \(k-1\) fixed data positions and set their bits to the operation's neutral value: \(1\) for AND and \(0\) for OR or XOR.
Let the remaining \(n-k+1\) data bits vary freely.
Given two different prefixes in this family, append the same query suffix selecting one differing position and all reserved positions.
These \(k\) query indices are distinct, and the output is the differing data bit.
Thus the \(2^{n-k+1}\) prefixes have distinct count vectors, giving \((n+1)^s\geq2^{n-k+1}\) and the second bound.

The argument also applies to atomic query indices: it uses only that the same query suffix can be appended to any two data prefixes of the same length.
\end{proof}

%% file: appendices/theory-task-results/sparse/task.tex
We use \(\text{\TaskSparseStar}_{I,F}\) from \Cref{def:task-sparse}, with fixed \(I=(i_1,\ldots,i_k)\) and \(F\).
Write \(m=\max_r i_r\).
Inputs are nonempty binary strings \(x=x_1\cdots x_n\), with one token per bit.

%% file: appendices/theory-task-results/sparse/logic.tex
\begin{proposition}[\TaskSparseStar in \Starfree]
\label{res:sparse-logic}
For every fixed \(k\), \(I\), and \(F\), the task \(\text{\TaskSparseStar}_{I,F}\) is representable in \(\FO[<]\), and is therefore star-free and computable by DLOGTIME-uniform \(\ACz\) and \(\TCz\) circuits.
\end{proposition}

\begin{proof}
The selected indices \(I=(i_1,\ldots,i_k)\) are fixed, so we can identify each selected bit with a fixed position predicate and apply a Boolean formula computing \(F\).
We retain the one-based indices of \Cref{def:task-sparse}; in the zero-based logical structure, \(x_{i_r}\) occurs at position \(i_r-1\).

\medskip
\noindent\textbf{Position predicates.}
Using \(\mathrm{Eq}_j(q)\) from \Cref{res:index-fixed-logic}, which holds exactly when \(q=j\), define
\[
B_r
:=
\exists q\left(
\mathrm{Eq}_{i_r-1}(q)
\land Q_1(q)
\right).
\]
For each \(r\in[k]\), \(B_r\) is true exactly when \(n\geq i_r\) and \(x_{i_r}=1\).

\medskip
\noindent\textbf{Sentence.}
Let \(\widehat F(z_1,\ldots,z_k)\) be a fixed Boolean formula computing \(F\).
The positive instances of \(\text{\TaskSparseStar}_{I,F}\) are defined by
\[
\psi_{I,F}
:=
\left(
\exists p\,\mathrm{Eq}_{m-1}(p)
\right)
\land
\widehat F(B_1,\ldots,B_k).
\]
Here each variable \(z_r\) in \(\widehat F\) is replaced by the sentence \(B_r\).

The first conjunct holds exactly when \(n\geq m\), ensuring that all selected positions exist.
In that case, each \(B_r\) has truth value \(x_{i_r}\), so \(\psi_{I,F}\) has truth value \(F(x_{i_1},\ldots,x_{i_k})\).
If \(n<m\), the first conjunct forces the output to \(0\), as required, regardless of \(F\).
Thus \(\psi_{I,F}\) defines \(\text{\TaskSparseStar}_{I,F}\), including when selected indices repeat.

Since \(k\), \(I\), and \(F\) are fixed, \(\psi_{I,F}\) is one fixed \(\FO[<]\) sentence.
The \(\FO[<]\)--\Starfree correspondence and the inclusions in \Cref{eq:fo-inclusions} prove the claim.
\end{proof}

%% file: appendices/theory-task-results/sparse/c-rasp.tex
\begin{proposition}[\TaskSparseStar in \CRASP]
\label{res:sparse-crasp}
For every fixed \(k\), \(I\), and \(F\), the task \(\text{\TaskSparseStar}_{I,F}\) has a base \(\CRASP\) program of counting depth at most two.
\end{proposition}

\begin{proof}
Let \(\widehat F\) be a fixed Boolean formula computing \(F\).
Reuse \(\operatorname{Bit},N,A_\ell,C_\ell,\) and \(B_\ell\) from the proof of \Cref{res:index-fixed-crasp}, taking \(L=m\).
Here every input symbol is a data bit, so \(N(i)=i\).
For each selected index \(i_r\), the stream \(B_{i_r}(i)\) has truth value \(x_{i_r}\) whenever \(i\geq i_r\).

Finally, apply \(F\) (and guard against short inputs):
\[
D(i)
:=
(m\leq N(i))
\land
\widehat F(B_{i_1}(i),\ldots,B_{i_k}(i)).
\]
At the final position, \(D(n)=0\) if \(n<m\), and \(D(n)=F(x_{i_1},\ldots,x_{i_k})\) otherwise.
As in the fixed-range construction, \(N\) has counting depth one and each \(C_\ell\) has counting depth two.
The length check and Boolean formula add no counting depth.
\end{proof}

%% file: appendices/theory-task-results/sparse/bapo.tex
\begin{proposition}[\TaskSparseStar is \BAPO-easy]
\label{res:sparse-bapo}
For every fixed \(k\), \(I\), and \(F\), the task \(\text{\TaskSparseStar}_{I,F}\) has a \((0,k)\text{-}\BAPO\).
\end{proposition}

\begin{proof}
For a split after position \(t\), define
\[
g(x_{t+1:n},t,x_j,j)
=
\indicator[j\in\{i_1,\ldots,i_k\}].
\]
The candidate set is
\[
\mathbb G_t
=
\{(x_j,j):j\in\{i_1,\ldots,i_k\},\,j\leq t\}.
\]
It satisfies \(\lvert\mathbb G_t\rvert\leq k=b\).
Hence the required subset is necessarily \(G=\mathbb G_t\), and \(h\) receives every selected prefix bit.

The suffix oracle \(h\) returns \(0\) when \(n<m\) (the input is too short to contain all indices).
Otherwise, for each \(i_r\), it reads \(x_{i_r}\) from \(G\) if \(i_r\leq t\) and from suffix position \(i_r-t\) if \(i_r>t\).
It then computes \(F(x_{i_1},\ldots,x_{i_k})\). The prefix oracle $f$ is not needed.
\end{proof}

%% file: appendices/theory-task-results/addition/bapo.tex
\begin{proposition}[Blocked \TaskBinaryAddition verification is \BAPO-hard]
\label{res:addition-bapo}
Encode an instance of \TaskBinaryAddition with $n$-bit operands (\Cref{def:task-addition}) as $x\#y\#z$. For $n\geq2$, every $(a,b)$-\BAPO solving this task satisfies
\[
    a+4b\log_2(4n+1)\geq n-2.
\]
In particular, this encoding of \TaskBinaryAddition is \BAPO-hard.
\end{proposition}

\begin{proof}
Fix the split $x\#y\#\mid z$, so that both operands lie in the prefix.
Write the \BAPO functions as $f,g,h$.
All prefixes have length $2n+2$, and their only variable positions are the $2n$ operand bits.
For each integer
\[
    c\in\mathcal C
    :=\{3\cdot2^{n-2},\ldots,2^n-1\},
\]
let $z_c$ be its $(n+1)$-bit representation.
Call a prefix $x\#y\#$ a \emph{completion for $c$} if its operands sum to $c$.
There are $2^{n-2}$ candidate sums in $\mathcal C$.
This range gives exponentially many candidate sums while ensuring that, among operand pairs with any one of these sums, neither value of any operand bit is too rare.

We will associate each candidate sum $c$ with a \emph{record} $r_c$ consisting of the $a$-bit prefix summary output by $f$ and at most $4b$ operand bits, together with their positions.
We will ensure any prefix matching the record $r_c$ has operand sum $c$.
Counting the possible records will then give the lower bound.

\medskip
\noindent\textbf{A small record that fixes the sum $c$.}
Fix a candidate sum $c$ and its suffix $z_c$.
Choose a completion $p_c$ maximizing the number of prefix positions selected by $g$, and call this number $M_c$.
This count includes any selected separators.
Call an operand position \emph{active} if $g$ selects at least one of its two possible bit values with suffix $z_c$.
Let $m_c$ be the number of active positions.

We first show that $m_c\leq4M_c$.
Choose a completion for $c$ uniformly by choosing an integer $X$ uniformly from $\{0,\ldots,c\}$ and setting $Y=c-X$.
Both $X$ and $Y$ are uniform on $\{0,\ldots,c\}$.
For either operand, each bit is $0$ with probability at least $1/4$ and $1$ with probability at least $1/4$.
To see this for any bit other than the most significant one, the interval $\{0,\ldots,c\}$ contains $\{0,\ldots,2^{n-1}-1\}$, on which that bit takes each value $2^{n-2}$ times.
For the most significant bit, at least $2^{n-2}$ values lie on each side of $2^{n-1}$ because $c\geq3\cdot2^{n-2}$.
In both cases, there are at most $2^n$ possible operand values for $X$ and $Y$, giving the $1/4 = 2^{n-2} / 2^n$ lower bounds.

For each active position $j$, let $I_j$ be $1$ if its bit is selected by $g$ and $0$ otherwise.
At least one bit value is selected at that position, and that value occurs with probability at least $1/4$, so $\mathbb E[I_j]\geq1/4$.
Let $S$ be the number of selected operand positions.
Inactive positions are never selected, so $S = \sum_{j\text{ active}}I_j$.
Thus, by linearity of expectation,
\[
    \mathbb E [S] = \sum_{j\text{ active}}\mathbb E [I_j] \ge \sum_{j\text{ active}} 1/4 = m_c / 4.
\]
Also, $S\leq M_c$ for every completion by the definition of $M_c$.
Combining these facts yields $m_c/4\leq\mathbb E[S]\leq M_c$, so $m_c\leq4M_c$.

We form $r_c$ by recording $f(p_c)$ and the values and positions of a subset of operand bits, according to the following cases:
\begin{itemize}
\item Case 1: If $M_c\geq b$, choose any $b$ selected token-position pairs as $G_c$ and record their operand bits, at most $b$ bits.
\item Case 2: If $M_c<b$, let $G_c$ be the full selected set and record all active operand bits, at most $m_c\leq4M_c<4b$ bits.
\end{itemize}
The separator tokens need not be recorded because their positions and values are fixed.
The first alternative also covers $b=0$, with $G_c$ empty.

For any prefix $p'$ matching the record (i.e., with $f(p') = f(p_c)$ and with matching recorded bits), $G_c$ remains a legal attended set.
In Case 1, its $b$ pairs remain selected because $g$ sees the same suffix, split, token, and position.
In Case 2, the full selected set is unchanged: all active positions have their recorded values, inactive positions are never selected by $g$, and separators are fixed.
Since the recorded prefix summary also agrees, the suffix oracle receives identical arguments:
\[
    h(f(p'),G_c,z_c,2n+2)
    =h(f(p_c),G_c,z_c,2n+2)=1.
\]
Correctness therefore forces the operands in $p'$ to sum to $c$.

\medskip
\noindent\textbf{Counting the records.}
Since any prefix matching $r_c$ must have operand sum $c$, different candidate sums must have different records.
Otherwise, if $r_c=r_{c'}$ for $c\ne c'$, the prefix $p_{c'}$ would match $r_c$.
With the suffix fixed at $z_c$, the preceding argument would force acceptance of $p_{c'}z_c$, even though its operands sum to $c'\ne c$.

There are at most $2^a$ distinct outputs of $f$.
Each recorded pair has $2n$ possible positions and two possible bit values.
Sorting the pairs by position and padding to $4b$ slots with a blank symbol therefore gives at most $(4n+1)^{4b}$ possible lists (each of the $4b$ slots contains one of $2\cdot 2n$ position-bit pairs or a blank).
Since records for each $c \in \mathcal C$ must be distinct,
\[
    2^{n-2}=|\mathcal C|
    \leq 2^a(4n+1)^{4b}.
\]
Taking logarithms proves the bound.
In particular, constant $a,b$ cannot suffice.
\end{proof}

%% file: appendices/theory-task-results/addition/sensitivity.tex
\begin{proposition}[Average sensitivity of \TaskBinaryAddition verification]
\label{res:addition-sensitivity}
Let $f_n(x,y,z)$ test whether the $n$-bit operands $x,y$ sum to the $(n+1)$-bit proposed result $z$.
Under the uniform distribution on the $m=3n+1$ input bits, its average sensitivity is
\[
    \mathrm{as}_m(f_n)=\frac{3n+1}{2^n}=\Theta(n/2^n).
\]
\end{proposition}

\begin{proof}
The sensitivity at an input is the number of single-bit flips that change the answer.
There are exactly $2^{2n}$ accepting inputs, since each operand pair has a unique $(n+1)$-bit sum.
At every accepting input, flipping any one of the $m$ bits makes the equality false: flipping an operand bit changes the left-hand side by a nonzero power of two, and flipping a result bit changes the right-hand side.
Thus each accepting input has sensitivity $m$.

View the inputs as vertices of the Boolean hypercube, with an edge for each single-bit flip.
There are exactly $m2^{2n}$ edges joining an accepting input to a rejecting one.
Each such edge contributes once to the sensitivity at each endpoint, and all other edges contribute zero.
Averaging over all $2^{3n+1}$ inputs gives
\[
    \mathrm{as}_m(f_n)
    =\frac{2m2^{2n}}{2^{3n+1}}
    =\frac{3n+1}{2^n}.
\]
\end{proof}

%% file: appendices/theory-task-results/dyck/c-rasp.tex
We use the bounded-depth language \(\TaskDyck\text{-}(k,D)\) and bracket alphabet \(\Sigma_k\) from \Cref{def:task-dyck}.

\begin{proposition}[\TaskDyck-$(k,D)$ in \CRASP]
\label{res:dyck-crasp}
For every fixed \(k\geq 1\) and \(D\geq 1\), membership in \(\TaskDyck\text{-}(k,D)\) on nonempty inputs is computed by a base \(\CRASP[\varnothing]\) program with \(O(kD)\) channels and counting depth at most three.
\end{proposition}
\begin{proof}
Let the input have length \(n\geq1\).
First, we identify whether each token is an open bracket or a close bracket, regardless of its type.
\[
\begin{aligned}
\mathsf{Open}(i)
&:=\bigvee_{t=1}^{k}Q_{\mathsf{open}_t}(i),\\
\mathsf{Close}(i)
&:=\bigvee_{t=1}^{k}Q_{\mathsf{close}_t}(i).
\end{aligned}
\]
Then, we compute the open-minus-close prefix height at each position (inclusive of the current position):
\[
\begin{aligned}
O(i)
&:=\#[j\leq i:\mathsf{Open}(j)],\\
C(i)
&:=\#[j\leq i:\mathsf{Close}(j)],\\
H(i)
&:=O(i)-C(i).
\end{aligned}
\]
Then we define Boolean streams that identify whether the current prefix is at a given depth $d \in \{0,\ldots,D\}$, or has underflowed or overflowed:
\[
\begin{aligned}
\mathsf{AtDepth}_d(i)
&:=(H(i)=d),
\qquad 0\leq d\leq D,\\
\mathsf{Underflow}(i)
&:=(H(i)<0),\\
\mathsf{Overflow}(i)
&:=(D<H(i)).
\end{aligned}
\]

For every depth \(d\in[D]\) and type \(t\in[k]\), mark type-\(t\) openings and closings at depth \(d\):
\[
\begin{aligned}
\mathsf{OpenAtDepth}_{d,t}(i)
&:=Q_{\mathsf{open}_t}(i)\land\mathsf{AtDepth}_d(i),\\
\mathsf{CloseAtDepth}_{d,t}(i)
&:=Q_{\mathsf{close}_t}(i)\land\mathsf{AtDepth}_{d-1}(i).
\end{aligned}
\]
An opening at depth \(d\) crosses from height $d - 1$ to height $d$, while a closing at depth \(d\) crosses from height \(d\) to height \(d-1\).

We then count the openings and closings, at each depth \(d\) and for each type \(t\), over each prefix:
\[
\begin{aligned}
\mathsf{NumOpensAtDepth}_{d,t}(i)
&:=\#[j\leq i:\mathsf{OpenAtDepth}_{d,t}(j)],\\
\mathsf{NumClosesAtDepth}_{d,t}(i)
&:=\#[j\leq i:\mathsf{CloseAtDepth}_{d,t}(j)].
\end{aligned}
\]
A type mismatch is detected when a prefix contains more type-\(t\) closings than openings at the same depth:
\[
\mathsf{TypeBad}_{d,t}(i)
:=
\bigl(
\mathsf{NumOpensAtDepth}_{d,t}(i)
<
\mathsf{NumClosesAtDepth}_{d,t}(i)
\bigr).
\]

A violation occurs at a position \(i\) if the prefix up to \(i\) has underflowed, overflowed, or has a type mismatch at any depth and type.
\[
\begin{aligned}
\mathsf{Bad}(i)
&:=
\mathsf{Underflow}(i)
\lor
\mathsf{Overflow}(i)
\lor
\bigvee_{d=1}^{D}\bigvee_{t=1}^{k}
\mathsf{TypeBad}_{d,t}(i).
\end{aligned}
\]
We accept the string if no violation occurs at any position, and the final height is zero:
\[
V(i) :=\#[j\leq i:\mathsf{Bad}(j)], \quad
\mathsf{Accept}(i)
:=
(V(i)=0)\land(H(i)=0).
\]
The program accepts from the last position through \(\mathsf{Accept}(n)\).
The streams \(O\) and \(C\) have counting depth one, the depth-specific opening and closing counts have counting depth two, and \(V\) has counting depth three.

\textbf{Correctness.} If the word is in \(\TaskDyck\text{-}(k,D)\), every prefix height lies in \(\{0,\ldots,D\}\) and the final height is zero, so the height checks pass.
Every closing also matches an earlier opening of the same type at the same depth, so no closing count exceeds its corresponding opening count.
Thus every nonempty word in \(\TaskDyck\text{-}(k,D)\) is accepted.
Conversely, the height checks reject underflow, overflow, and a nonzero final height.
If these checks pass, openings and closings at each depth \(d\) alternate, so the \(j\)th closing at depth \(d\) pairs with the \(j\)th opening at depth \(d\).
If a type mismatch occurs, consider the first mismatched closing, of type \(t\) at depth \(d\).
Previously completed pairs at that depth contribute equally to the opening and closing counts, while the unmatched opening has a different type.
This closing therefore makes \(\mathsf{NumClosesAtDepth}_{d,t}\) exceed \(\mathsf{NumOpensAtDepth}_{d,t}\), triggering \(\mathsf{TypeBad}_{d,t}\).
Checking every prefix ensures that a later symbol cannot erase the violation.
\end{proof}

%% file: appendices/theory-task-results/dyck/bapo-lower-bound.tex
\begin{proposition}[Unbounded \TaskDyck is \BAPO-hard]
\label{res:dyck-bapo}
Fix a bracket-type count \(k\geq1\).
Every \((a,b)\)-\BAPO recognizing unbounded \TaskDyck with \(k\) bracket types on all inputs of length at most \(N\) satisfies
\[
a+\log_2(b+1)\geq\log_2N-O(1).
\]
In particular, unbounded \TaskDyck is \BAPO-hard for every fixed \(k\geq1\).
\end{proposition}

\begin{proof}
We use only bracket type \(1\), with opening bracket \(\mathsf{open}_1\) and closing bracket \(\mathsf{close}_1\).
Suppose an \((a,b)\)-\BAPO recognizes the language on all inputs of length at most \(N\).
Write its oracles as \(f,g,h\), and set
\[
m=2^a+1,
\qquad
V=m(2b+1).
\]
Assume for contradiction that \(4V\leq N\).
We will construct \(m\) prefixes that require distinct summaries from \(f\), although \(f\) has only \(2^a\) possible outputs.
All complete inputs in the construction have length \(4V\).

\medskip
\noindent\textbf{Suffixes that distinguish different balances.}
All prefixes will have length \(2V\), start with \(V\) opening brackets, and differ only in their last \(V\) positions.
For \(r=0,\ldots,m-1\), define $m$ distinct suffixes $q_r$:
\[
D_r=2(bm+r),
\qquad
q_r=\mathsf{close}_1^{D_r}\,\bigl(\mathsf{open}_1\mathsf{close}_1\bigr)^{V-bm-r}.
\]
The exponent is positive because \(V-bm-r=bm+m-r\geq1\).
Each suffix has length \(2V\) and closes a net total of \(D_r\) brackets.
A prefix that never underflows and ends at balance \(D_s\) therefore forms a valid string with \(q_r\) exactly when \(s=r\).

\medskip
\noindent\textbf{Fixing a common attended set for each suffix.}
Initially fix the first \(V\) prefix positions to \(\mathsf{open}_1\) and leave the last \(V\) positions unassigned.
Process the suffixes \(q_0,\ldots,q_{m-1}\) in order, keeping every assignment made in earlier steps.
For each \(q_r\), count the already assigned positions selected by \(g\).
While fewer than \(b\) are selected, look for an unassigned position \(i\) and a bracket \(z\in\{\mathsf{open}_1,\mathsf{close}_1\}\) such that
\[
g(q_r,2V,z,i)=1.
\]
If such a pair exists, assign \(z\) to position \(i\); otherwise stop processing \(q_r\).
Each assignment increases the selected count by one.
Thus each suffix fixes at most \(b\) new positions, and all suffixes together fix at most \(bm\) positions in the last \(V\) positions.

For any fixed \(q_r\), there is now a legal attended set \(G_r\) shared by every completion of the unassigned positions:
\begin{itemize}
\item If at least \(b\) assigned positions were selected when processing \(q_r\) stopped, choose any \(b\) of their token-position pairs.
These pairs remain selected in every completion, so they form a legal attended set.
\item Otherwise, \(g\) rejected both possible brackets at every position still unassigned at that point.
Later assignments cannot add any selected positions.
The full set of selected token-position pairs is therefore the same in every completion and has size less than \(b\); take this set as \(G_r\).
\end{itemize}
This argument also covers \(b=0\), when \(G_r\) is empty.
The suffixes must be fixed before this construction because \(g\) can depend on the entire suffix.

\medskip
\noindent\textbf{Completing the prefixes.}
Let \(A\) be the number of opening brackets assigned among the last \(V\) positions.
For each \(r\), construct a prefix \(p_r\) by assigning exactly \(bm+r-A\) of the remaining positions to \(\mathsf{open}_1\) and all others to \(\mathsf{close}_1\).
This is possible because \(0\leq A\leq bm\), at least
\[
V-bm=bm+m
\]
positions remain unassigned, and \(0\leq bm+r-A\leq bm+m-1\).
The leading \(V\) opening brackets prevent underflow anywhere in \(p_r\), regardless of the order of the last \(V\) brackets.
Its final balance is
\[
V+(bm+r)-(V-bm-r)=2(bm + r)=D_r.
\]
It follows that
\[
p_s q_r\in\TaskDyckOne
\quad\Longleftrightarrow\quad
s=r.
\]

\medskip
\noindent\textbf{Forcing distinct prefix summaries.}
Suppose \(f(p_r)=f(p_s)\) for distinct \(r,s\).
With suffix \(q_r\), the same attended set \(G_r\) is legal for both prefixes, so the suffix oracle can receive identical arguments:
\[
h(f(p_r),G_r,q_r,2V)=h(f(p_s),G_r,q_r,2V).
\]
But \(p_rq_r\) is valid and \(p_sq_r\) is not.
Correctness for every legal attended set therefore requires all \(m\) prefix summaries to be distinct.
However, \(f\) has only \(2^a<m\) possible outputs, so this is impossible.
Our assumption \(4V\leq N\) must therefore be false, giving
\[
N<4(2^a+1)(2b+1)\leq16\cdot2^a(b+1).
\]
The second inequality uses \(2^a+1\leq2\cdot2^a\) and \(2b+1\leq2(b+1)\).
Taking logarithms gives \(a+\log_2(b+1)>\log_2N-4\), proving the claimed bound.
For any fixed \(0<\epsilon<1\), if \(b=o(N^{1-\epsilon})\), the bound forces \(a=\Omega(\log N)\), ruling out \(a=o(\log N)\).
Because every constructed string uses only one bracket type, the result holds for every \(k\geq1\).
\end{proof}

%% file: appendices/theory-task-results/k-hop/logic.tex
\subsection{Fixed-hop induction}
\label{app:k-hop-results}

We use \TaskKHop and its hop map $h$ from \Cref{def:task-k-hop}, with fixed finite alphabet $\Sigma$ and $k\geq1$.
Since its answer is a symbol in \(\Sigma\cup\{\bot\}\), we view \TaskKHop as the \(|\Sigma|+1\) languages \(L_a\) of inputs whose answer is \(a\), and a theory predicts \Easy if each \(L_a\) is easy under it.

\begin{proposition}[\TaskKHop in \Starfree]
\label{res:k-hop-logic}
For fixed \(\Sigma\) and \(k\), each language \(L_a\), \(a\in\Sigma\cup\{\bot\}\), is definable by an \(\FO[<]\) sentence.
Thus each \(L_a\) is star-free and has DLOGTIME-uniform \(\ACz\) and \(\TCz\) circuits.
\end{proposition}

\begin{proof}
Our construction defines one hop as a relation between input positions, then composes it \(k\) times.
Use \(\mathrm{SUCC}\) and \(\mathrm{max}\) from \Cref{res:index-fixed-logic}.

\medskip
\noindent\textbf{One hop.}
Define
\[
\mathrm{Same}(p,i)
:=\bigvee_{a\in\Sigma}\bigl(Q_a(p)\land Q_a(i)\bigr),
\]
which holds exactly when \(x_p=x_i\).
The hop relation is
\[
\begin{aligned}
\mathrm{Hop}(i,j):=\exists p\,\bigl(
&p<i\land\mathrm{Same}(p,i)\land\mathrm{SUCC}(p,j)\\
&\land\neg\exists q\,(p<q\land q<i\land\mathrm{Same}(q,i))
\bigr).
\end{aligned}
\]
Here \(p\) is the latest previous occurrence of \(x_i\), and \(j=p+1\).
Thus \(\mathrm{Hop}(i,j)\) holds exactly when \(h(i)\) is defined and equals \(j\).
This also handles adjacent repeats: if \(x_{i-1}=x_i\), then \(j=i\).

\medskip
\noindent\textbf{Fixed-hop composition.}
Define
\[
\begin{aligned}
\mathrm{Hop}_0(i,j)&:=(i=j),\\
\mathrm{Hop}_{r+1}(i,j)
&:=\exists u\,\bigl(\mathrm{Hop}(i,u)\land\mathrm{Hop}_r(u,j)\bigr).
\end{aligned}
\]
For \(r=0,\ldots,k\), these abbreviate finite first-order formulas obtained by expanding the definition.
Induction on \(r\) shows that \(\mathrm{Hop}_r(i,j)\) holds exactly when all \(r\) hops are defined and \(h^r(i)=j\).

\medskip
\noindent\textbf{Output sentences.}
For each \(a\in\Sigma\), define
\[
\varphi_a
:=\exists i\,\exists j\,
\bigl(\mathrm{max}(i)\land\mathrm{Hop}_k(i,j)\land Q_a(j)\bigr).
\]
This holds exactly when the final-position output is \(a\).
The output is \(\bot\) exactly when no \(k\)-hop endpoint exists:
\[
\varphi_\bot
:=\exists i\,
\bigl(\mathrm{max}(i)\land\neg\exists j\,\mathrm{Hop}_k(i,j)\bigr).
\]
Since \(\Sigma\) and \(k\) are fixed, each sentence has fixed size.
The \(\FO[<]\)--\Starfree correspondence and the inclusions in \Cref{eq:fo-inclusions} give the result.
\end{proof}

%% file: appendices/theory-task-results/k-hop/c-rasp.tex
\begin{proposition}[\TaskKHop not in \CRASP]
\label{res:k-hop-crasp}
For every fixed $k\geq1$ and fixed alphabet $\Sigma$ with at least $k+3$ symbols, the final-position output of \TaskKHop is not computable in base \CRASP.
\end{proposition}

\begin{proof}
We use the language
\[
    L_{bb}=\{a,b,e\}^*be^*b\{a,b,e\}^*.
\]
In words, $L_{bb}$ is the language of strings such that deleting all $e$s leaves a string containing two consecutive $b$s.
By \citet[Proposition~31 and Lemma~35]{huang2024formal}, $L_{bb}$ is not definable even in \CRASP{}[periodic, local], and hence not in base \CRASP.
We show that a \CRASP program for \TaskKHop would let us recognize $L_{bb}$.

To do this, we use a construction underlying a closure property of \CRASP~\citep[Lemma~42]{huang2024formal}.
If $\mu$ replaces each symbol by a string of the same fixed length $d$ (what we call a \emph{block}), a \CRASP program operating on $\mu(w)$ can be simulated by a new \CRASP program operating on $w$, with a separate stream for each position within a block.
\citet[Lemma~42]{huang2024formal} show this closure property for \CRASP{}[periodic, local], but their construction uses no positional operators when the simulated program uses none, so it also applies to base \CRASP.

\medskip
\noindent\textbf{Encoding.}
Use distinct symbols $a,b,e,q_1,\ldots,q_k\in\Sigma$.
Replace each source symbol $w_i \in \{a, b, e\}$ by a \emph{block} of length $d=k(k+1)$ according to the mapping $\mu:\{a, b, e\} \rightarrow \Sigma^d$:
\[
    \mu(a)=(q_k\cdots q_1a)^k,\qquad
    \mu(b)=(q_k\cdots q_1b)^k,\qquad
    \mu(e)=e^d.
\]
Call each copy of $q_k\cdots q_1x$, for $x\in\{a,b\}$, a \emph{group}.
An $a$- or $b$-block contains $k$ groups.

\medskip
\noindent\textbf{Recovering the previous non-neutral symbol.}
Consider the first position of the block for a source symbol $w_i\in\{a,b\}$.
Its token is $q_k$.
For $r>1$, a hop from $q_r$ goes to the $q_{r-1}$ immediately following the previous occurrence of $q_r$, which lies in the preceding group.
A hop from $q_1$ similarly goes to the data symbol at the end of the preceding group.
Thus $k$ hops move back exactly $k$ groups and return that group's data symbol.
Any intervening $e$-blocks are skipped because they contain no $q_i$.

Since every original $a$ and $b$ contributes exactly $k$ groups, the $k$-hop output at position $d(i-1)+1$ of $\mu(w)$ (the first position of the block for $w_i$) is the last non-$e$ symbol in $w_1\cdots w_{i-1}$.
If this source prefix contains no non-$e$ symbol, the first hop is undefined.
In particular, when $w_i=b$, the output at that encoded position is $b$ exactly when $w_i$ completes an occurrence of $be^*b$ in the source word.
For example, with $k=3$, the groups around two successive non-neutral symbols $x,y$ are
\[
    \begin{tikzpicture}[x=4mm, y=1cm]
        \foreach \pos/\name/\token in {
            0/xqthree/q_3, 1/xqtwo/q_2, 2/xqone/q_1, 3/hopend/x,
            4.5/xqthreeagain/q_3, 5.5/xqtwoagain/q_2, 6.5/hoptwo/q_1, 7.5/xagain/x,
            9/xqthreelast/q_3, 10/hopone/q_2, 11/xqonelast/q_1, 12/xlast/x,
            14/neutral/e^{*},
            16/hopstart/q_3, 17/yqtwo/q_2, 18/yqone/q_1, 19/yfirst/y,
            20.5/yqthreeagain/q_3, 21.5/yqtwoagain/q_2, 22.5/yqoneagain/q_1, 23.5/yagain/y,
            25/yqthreelast/q_3, 26/yqtwolast/q_2, 27/yqonelast/q_1, 28/ylast/y
        } {
            \node[inner sep=1pt] (\name) at (\pos,0) {$\token$};
        }
        \node at (6,-0.4) {$\underbrace{\hspace{5.1cm}}_{\mu(x)}$};
        \node at (22,-0.4) {$\underbrace{\hspace{5.1cm}}_{\mu(y)}$};
        \draw[-{Stealth[length=1.4mm]}, semithick]
            (hopstart.north) to[out=155,in=25]
            node[above, font=\scriptsize] {hop 1} (hopone.north);
        \draw[-{Stealth[length=1.4mm]}, semithick]
            (hopone.north) to[out=155,in=25]
            node[above, font=\scriptsize] {hop 2} (hoptwo.north);
        \draw[-{Stealth[length=1.4mm]}, semithick]
            (hoptwo.north) to[out=155,in=25]
            node[above, font=\scriptsize] {hop 3} (hopend.north);
    \end{tikzpicture}
\]

\medskip
\noindent\textbf{Simulating the encoded input in \CRASP.}
Suppose for a contradiction that a base \CRASP program $P$ has final Boolean stream $F$ testing whether the \TaskKHop output is $b$ on every input over $\Sigma$.
We run $P$ on the encoded string $\mu(w)$.
Because the program is causal, $F$ gives the correct answer at every position: its value there is also its final value on the prefix ending at that position.

The simulation of \citet[Lemma~42]{huang2024formal} converts $P$, running on $\mu(w)$, into a new program $\widehat P$ running on $w$.
For each block offset $r\in\{1,\ldots,d\}$, the new program has a stream $F_r(i)$ equal to the value of $F$ at position $d(i-1)+r$ in the original run on $\mu(w)$. In particular, as we showed above, if $w_i=b$, then $F_1(i)$ is true exactly when the previous non-$e$ source symbol is also $b$, i.e., when $w_i$ completes an occurrence of $be^*b$.
We therefore recognize $L_{bb}$ by accepting when
\[
    \#[j\leq i:Q_b(j)\land F_1(j)]>0
\]
at the end of the source word.
This is a base \CRASP test, contradicting the lower bound for $L_{bb}$.
\end{proof}

%% file: appendices/c-rasp-separation/c-rasp-separation.tex
\section{An Exponential Width Separation between \CRASP and Softmax Transformers on Content-Dependent Retrieval}
\label{app:c-rasp-separation}

\input{appendices/c-rasp-separation/sections/intro.tex}

\input{appendices/c-rasp-separation/sections/query-dependent-retrieval-task.tex}

\input{appendices/c-rasp-separation/sections/main-result.tex}

\input{appendices/c-rasp-separation/sections/c-rasp-bounds.tex}

\input{appendices/c-rasp-separation/sections/attention-upper-bound.tex}

\input{appendices/c-rasp-separation/sections/experiment.tex}

%% file: appendices/c-rasp-separation/sections/intro.tex
Under \CRASP \citep{yang2024counting}, cross-token communication in causal transformers is modeled as an act of temporal \emph{counting}.
Under this framework, the only cross-position primitive is unary prefix counting, $\#[j \leq i: P(j)]$, which computes how many tokens in the history satisfy a position-wise predicate $P$.
\CRASP programs compile to causally-masked softmax transformers by zeroing the query and key projection weights ($W_Q = W_K = 0$). This forces attention weights to be uniform over the causal prefix ($1/i$ at position $i$), turning each attention head into a simple prefix-averaging accumulator.

Restricting attention to uniform prefix-averaging appears at first glance to be a severe architectural constraint, as it strips away the defining operational mechanism of the transformer architecture: dynamic, content-dependent information routing based on query-key interactions ($q_i^\intercal k_j$).
Yet, \citet{yang2024counting} and \citet{yang2025kneedeep} prove a formal \emph{equivalence} between fixed-precision transformers and \CRASP. Specifically, \citet{yang2025kneedeep} claim a depth-for-depth equivalence:
\begin{equation}
\operatorname{Lang}(\mathcal{T}_k^{\mathrm{round}}) = \operatorname{Lang}(\mathrm{C\text{-}RASP}_k),
\end{equation}
implying that \CRASP and (fixed-precision) transformers are expressively equivalent, and moreover that the equivalence preserves depth.
Because \CRASP compiles to uniform attention, this result in turn implies that uniform-attention transformers can simulate general transformers with content-dependent attention.
Taken at face value, this equivalence directly challenges the centrality of content-dependent attention, implying that uniform prefix-averaging suffices to capture the full computational power of causal transformers.

This coarse equivalence claim, however, masks a severe resource blowup.
In the simulation construction of~\citet{yang2024counting,yang2025kneedeep}, simulating content-dependent attention with \CRASP (and, in turn, uniform attention) requires a brute-force exhaustive enumeration of all possible query vectors.
For a residual stream of dimension $d$ and precision $p$, the state space of possible query vectors has size $|\mathbb{F}_{p}^d| = 2^{pd}$, where $\mathbb{F}_p$ is the set of $p$-bit numbers (\Cref{eq:fixed-point}).
Because uniform attention cannot dynamically condition on a query at test time, the simulation must instantiate separate prefix-count accumulators for every possible query $\vec{q} \in \mathbb{F}_{p}^d$, yielding an exponential blowup of roughly $O(d \cdot p \cdot 2^{pd})$ count terms per layer.
When compiled back into a uniform-attention transformer, each counting term requires dedicated residual dimensions, yielding a uniform-attention model width of $O(d \cdot p \cdot 2^{pd})$.

This leaves open whether the exponential width blowup in this result is an artifact of the construction, or an unavoidable separation between counting and content-dependent retrieval.
To understand whether \CRASP is a useful conceptual and theoretical model of transformers, it is important to understand its relationship not only to uniform-attention transformers, but also to general content-dependent attention transformers.
In this section, we show that this exponential gap is fundamental.

We establish an explicit resource separation between \CRASP (and uniform-attention transformers) and standard content-dependent attention on a simple information-routing task: query-dependent retrieval.
In this task, a sequence of $L = 2^m$ key-value items is followed by a query key, and the model must retrieve the binary bit value associated with the matching key.

We prove a three-part separation result:
\begin{enumerate}
    \item A single causal, content-dependent softmax attention head with linear residual width $m+2$ solves this task.
    \item In contrast, every \CRASP program requires at least $\Omega(2^m / m)$ distinct prefix-count nodes to solve the task.
    \item Consequently, any depth-$D$, fixed-precision transformer restricted to uniform attention ($W_Q = W_K = 0$) requires exponential channel width $d = \Omega(2^m / [D(p + m)])$.
\end{enumerate}

This result exposes a central limitation of \CRASP as a theoretical model of transformer capabilities, and a fundamental separation between a transformer's content-dependent attention and temporal counting.
In standard transformers, the query at position $i$ dynamically selects which prior positions $j$ to attend to by matching query and key representations ($q_i^\intercal k_j$).
Uniform attention disables this dynamic selection and forces position $i$ to receive only an unweighted average of all preceding activations.
Because position $i$ cannot selectively retrieve items from the history based on its query, a model without content-dependent attention is forced to store the entire history in its residual channels so that positionwise computation can later select the relevant information based on the query.
Our separation demonstrates that the width-asymptotic \CRASP expressive equivalence results of~\citet{yang2024counting,yang2025kneedeep} fail to capture the dynamic information routing power of transformers. 
In short, \emph{transformers are more than counting machines}.

Below, we formally define the query-dependent retrieval task, state the main separation theorem, and prove each part of the result in turn.

%% file: appendices/c-rasp-separation/sections/query-dependent-retrieval-task.tex
\subsection{The Query-Dependent Retrieval Task}

We consider a query-dependent key-value retrieval task.
Each input sequence consists of a database of $L = 2^m$ key-value items followed by a single query token, yielding a total sequence length of $n = L + 1$, where $m \geq 1$ is the key dimension. 
We study the complexity of this task as a function of the key dimension $m$, showing an exponential separation between counting and content-dependent retrieval.

The key space is the set of binary vectors $\Km = \{-1, +1\}^m$, containing $L = 2^m$ distinct keys.
The input is a sequence of $L+1$ tokens: the first $L$ tokens provide the key-value pairs of the database and the last token is the query.
For every key $k \in \Km$ and value $v \in \{-1, +1\}$, define the vector representation of the item as the concatenation $x(k, v) = (k, v) \in \reals^{m+1}$.
The first $m$ coordinates store the key $k$ and the $(m+1)$-th coordinate stores the value $v$.
Define the query token for a key $k$ by $q(k) = (k, 0) \in \reals^{m+1}$.

Fix an enumeration $k_1, \ldots, k_L$ of all $L=2^m$ keys in $\Km$. For a value assignment $\bm{v} = (v_1, \ldots, v_L) \in \{-1, +1\}^L$ and query index $r \in [L]$, define the input
\begin{equation}
    X(\bm{v}, r) = [x(k_1, v_1), \ldots, x(k_L, v_L), q(k_r)] \in \reals^{(L+1) \times (m+1)}.
\end{equation}
The value assignment $\bm{v}$ and query index $r$ define the input instance, whereas the key enumeration and the vector representation mappings $x(\cdot), q(\cdot)$ are taken to be fixed.
The task is to output the value $v_r$ specified by the query $q(k_r)$ at the final position.

%% file: appendices/c-rasp-separation/sections/main-result.tex
\subsection{Main Separation Result}
\label{subsec:c-rasp-separation-main-result}

The query-dependent retrieval task separates \CRASP and uniform attention from standard content-dependent attention by an exponential factor in the key dimension.

\begin{theorem}[Query-dependent retrieval separation]
\label{thm:c-rasp-retrieval-separation}
Fix $m \geq 1$ and let $L=2^m$. Consider the query-dependent retrieval task defined above.
Then:
\begin{enumerate}
    \item A single causal, content-dependent softmax attention head solves the task with residual width $m+2$.
    \item Every \CRASP program that solves the task requires at least
    \begin{equation}
        \frac{L}{\log_2(L+1)}
        =
        \Omega\left(\frac{2^m}{m}\right)
    \end{equation}
    distinct prefix-count nodes.
    \item Every depth-$D$, $p$-bit causal transformer restricted to uniform attention requires width
    \begin{equation}
        d
        =
        \Omega\left(
            \frac{2^m}{D(p+m)}
        \right).
    \end{equation}
\end{enumerate}
\end{theorem}

The key factor driving this separation is that \CRASP and uniform attention aggregate information independently of the query. They must therefore encode enough information about the database to answer any of the $L$ possible queries, distinguishing all $2^L$ value assignments. 
This requires exponentially many prefix-count nodes or residual channels.
In contrast, content-dependent attention can use query-key comparisons to select the unique matching database item and retrieve its value.
As a result, 
content-dependent attention solves the task with linear width, whereas \CRASP and uniform attention require exponential size, even when allowed polynomial depth and precision.

%% file: appendices/c-rasp-separation/sections/c-rasp-bounds.tex
\subsection{\CRASP and Uniform Attention Bounds}
\label{subsec:c-rasp-uniform-attention-bounds}

We first prove the bounds for \CRASP and uniform attention, showing that they need size scaling exponentially in $m$ to solve the task.
Recall that, in \CRASP, the only cross-position primitive is unary prefix counting
\begin{equation}
    \#[j \leq i : P(j)].
\end{equation}
That is, the only way information can flow across positions is by counting predicates in the history.
All other operations are pointwise.
We measure the size of a \CRASP program by the number of distinct prefix-count nodes it has.

\begin{proposition}[\CRASP bounds]
\label{prop:c-rasp-retrieval-bounds}
Fix $m\geq1$ and let $L=2^m$.
Then:
\begin{enumerate}
    \item Every \CRASP program of any count depth that solves the query-dependent retrieval task has at least
    \begin{equation}
        \frac{L}{\log_2(L+1)} = \Omega\left(\frac{2^m}{m}\right)
    \end{equation}
    distinct prefix-count nodes.
    \item There is a depth-one \CRASP program that solves the task with $L = 2^m$ prefix-count nodes.
\end{enumerate}
\end{proposition}

\begin{proof}
\textbf{Lower bound.}
Consider a correct \CRASP program with $s$ distinct prefix-count nodes, $C_1, \ldots, C_s$.
Order the program nodes topologically so that, in multi-layer programs, $C_j$ depends only on $\{C_i\}_{i < j}$.
For each value assignment $\bm{v}$, let $\tau(\bm{v})$ be the vector of all prefix-count values after the first $L$ positions, immediately before the query:
\begin{equation}
    \tau(\bm{v})
    =
    \left(
        C_1(L),\ldots,C_s(L)
    \right)
    \in
    \{0,\ldots,L\}^s.
\end{equation}
Note that the prefix-count vector $\tau(\bm{v})$ depends only on the database $[x(k_1, v_1), \ldots, x(k_L, v_L)]$ because the query occurs at position $L+1$.

Observe that the prefix-count vector $\tau(\bm{v})$ and the query token $q(k_r)$ determine the computation at the final position, since information about other positions can reach the query position only through counts.
Let $P_t$ be the predicate that $C_t$ counts, such that $C_t(i) = \#[j \leq i: P_t(j)]$.
Then, for each count node, its value at the final position is
\begin{equation}
    C_t(L+1) = C_t(L) + \mathbf{1}[P_t(L+1)].
\end{equation}
The first term $C_t(L)$ is contained in the prefix-count vector $\tau(\bm{v})$, while the second term $\mathbf{1}[P_t(L+1)]$ depends only on the query token and the values of the earlier program nodes at the final position.
In particular, $P_t$ may depend on previous count nodes in the program, and no restriction is made on the count depth of the program.

For the \CRASP program to correctly solve the task, the map $\bm{v} \mapsto \tau(\bm{v})$ must be injective.
Suppose that two distinct value assignments $\bm{v} \neq \bm{v}'$ produce the same prefix-count vector $\tau(\bm{v}) = \tau(\bm{v}')$.
Since $\bm{v} \neq \bm{v}'$, there must exist $r \in [L]$ such that $v_r \neq v'_r$.
Thus, on the query $q(k_r)$, the \CRASP program will return the same answer on $X(\bm{v}, r)$ and $X(\bm{v}', r)$.
This contradicts correctness. 

There are $2^L$ possible value assignments $\bm{v}$. The prefix-count vector $\tau(\bm{v})$ takes values in the space $\{0, \ldots, L\}^s$ since each of the $s$ count nodes has at most $L+1$ possible values.
Therefore, injectivity of $\bm{v} \mapsto \tau(\bm{v})$ requires that
\begin{equation}
    (L+1)^s
    \geq
    2^L.
\end{equation}
Taking the logarithm gives
\begin{equation}
    s
    \geq
    \frac{L}{\log_2(L+1)}.
\end{equation}

\textbf{Upper bound.} 
To establish the upper bound, we provide a \CRASP construction for the task. For each key $k_j \in \Km$, define the prefix count
\begin{equation}
    C_j(i) = \#[t \leq i: x_t = x(k_j, +1)],
\end{equation}
where $x_t$ is the token at position $t$ in the input $X(\bm{v}, r)$, so that $x_t = x(k_t, v_t)$ for $t \leq L$ and $x_{L+1} = q(k_r)$.
Given the structure of the task, each key $k_j$ occurs exactly once in the database items, and the query token does not equal $x(k_j, +1)$.
Thus, $C_j(L+1) = 1$ if $v_j = +1$ and $C_j(L+1) = 0$ if $v_j = -1$.
Convert these counts to Boolean predicates by setting $B_j(i) = \mathbf{1}[C_j(i) \geq 1]$. 

At the final query position, we compute the value at the matching key via
\begin{equation}
    \bigvee_{j=1}^L
    \left(
        \mathbf{1}[x_{L+1} = q(k_j)]
        \land
        B_j(L+1)
    \right).
\end{equation}
The condition $\mathbf{1}[x_{L+1} = q(k_j)]$ holds for exactly one $j$, namely $j = r$.
Thus, this gives $\mathtt{True}$ if $v_r = +1$ and $\mathtt{False}$ if $v_r = -1$. 
This construction uses exactly $L$ count nodes and has count depth one.
\end{proof}

This shows that \CRASP requires exponential size to simulate query-dependent information retrieval in this task.
When compiled to a transformer~\citep{yang2024counting,yang2025kneedeep}, this yields a transformer with exponential width, as each counting term requires dedicated residual dimensions.

We next prove a direct width lower bound for uniform-attention transformers.

We consider depth-$D$ causal transformers with bounded-precision activations.
For integers $p \geq 1$ and $f \geq 0$, let
\begin{equation}
\label{eq:fixed-point}
    \mathbb{F}_{p}
    =
    \left\{
        z \cdot 2^{-f}
        \;:\;
        z \in \mathbb{Z},\;
        -2^{p-1} \leq z < 2^{p-1}
    \right\}
\end{equation}
denote the set of $p$-bit fixed-point numbers with $f$ fractional bits.
We keep $f$ fixed throughout and omit it from the notation, since none of the bounds depend on it.
Every activation is assumed to lie in $\mathbb{F}_{p}$, so each coordinate takes one of $|\mathbb{F}_{p}| = 2^p$ values.
A rounding map is any function $\operatorname{round} : \reals \to \mathbb{F}_{p}$ that maps every element of $\mathbb{F}_{p}$ to itself.
Uniform attention means that $W_Q = W_K = 0$ in every attention head,
so every head attends uniformly over its causal prefix.
For head $h$ in layer $\ell$, let $v_i^{(\ell,h)} \in \mathbb{F}_{p}^{d_v^{(\ell,h)}}$ denote the value vector at position $i$, and let
\begin{equation}
    V
    =
    \sum_{\ell=1}^{D}
    \sum_h
    d_v^{(\ell,h)}
\end{equation}
denote the total value dimension across all layers and heads.
Under uniform attention with causal masking, the output of head $h$ in layer $\ell$ at position $i$ is
\begin{equation}
    a_i^{(\ell,h)}
    =
    \operatorname{round}
    \left(
        \frac{1}{i}
        \sum_{j=1}^i
        v_j^{(\ell,h)}
    \right),
\end{equation}
where the sum and the division are computed exactly before the result is rounded back into $\mathbb{F}_{p}$.
The lower bound holds for any rounding map, since the argument below uses only the structure of $\mathbb{F}_{p}$.

\begin{proposition}[Uniform-attention width lower bound]
\label{prop:uniform-attention-retrieval-lower-bound}
Fix $m \geq 1$ and let $L = 2^m$.
Consider a depth-$D$ causal transformer with uniform attention, causal masking, and activations in $\mathbb{F}_{p}$.
If it solves the query-dependent retrieval task, then its total value dimension satisfies
\begin{equation}
    V
    \geq
    \frac{L}{p+m}
    =
    \frac{2^m}{p+m}.
\end{equation}
In particular, if the total value dimension in each layer is at most $d$, then $V \leq Dd$ and
\begin{equation}
    d
    \geq
    \frac{2^m}{D(p+m)}.
\end{equation}
\end{proposition}

\begin{proof}
For each layer $\ell$ and head $h$, define the sum of its value vectors over the $L = 2^m$ database positions:
\begin{equation}
    S_{\ell,h}(\bm{v})
    =
    \sum_{j=1}^{L}
    v_j^{(\ell,h)}.
\end{equation}
By causality, the representation at each database position $j \leq L$ depends only on the first $j$ database items.
Hence $S_{\ell,h}(\bm{v})$ is a function of the value assignment $\bm{v}$ alone and does not depend on the query.
Define the prefix-sum vector $\rho(\bm{v})$ as the concatenation of these sums across all layers and heads:
\begin{equation}
    \rho(\bm{v})
    =
    \left(
        S_{\ell,h}(\bm{v})
    \right)_{\ell,h}.
\end{equation}
This vector has $V$ entries.

We claim that $\rho(\bm{v})$ and the query token determine the output at the final position $L+1$.
The proof is by induction on the layer index $\ell$.
The inductive hypothesis is that the representation at position $L+1$ entering layer $\ell$ is determined by $\rho(\bm{v})$ and the query token.
This holds at $\ell = 1$ because the input embedding at position $L+1$ is a function of the query token.
Assuming it at layer $\ell$, each head attends uniformly over all $L+1$ positions, so
\begin{equation}
    a_{L+1}^{(\ell,h)}
    =
    \operatorname{round}
    \left(
        \frac{
            S_{\ell,h}(\bm{v})
            +
            v_{L+1}^{(\ell,h)}
        }{L + 1}
    \right).
\end{equation}
The first term in the numerator is an entry of $\rho(\bm{v})$. The second is a function of the representation at position $L+1$ entering layer $\ell$, and is therefore determined by $\rho(\bm{v})$ and the query by the inductive hypothesis. Thus, $a_{L+1}^{(\ell,h)}$ is determined by $\rho(\bm{v})$ and the query.
The remaining operations of the layer are positionwise, so the representation at position $L+1$ entering layer $\ell+1$ is also determined by $\rho(\bm{v})$ and the query.

By the same argument as in Proposition~\ref{prop:c-rasp-retrieval-bounds}, the map $\bm{v} \mapsto \rho(\bm{v})$ must be injective for any uniform-attention transformer that solves this task.
Otherwise, suppose two distinct value assignments $\bm{v} \neq \bm{v}'$ have the same prefix-sum vector $\rho(\bm{v}) = \rho(\bm{v}')$.
Choose $r \in [L]$ such that $v_r \neq v'_r$.
The query $q(k_r)$ would produce the same output on $X(\bm{v},r)$ and $X(\bm{v}',r)$, contradicting correctness.

We now bound the number of possible values of $\rho(\bm{v})$.
Each entry is a sum of $L$ elements of $\mathbb{F}_{p}$ having the form $\sum_{j=1}^{L} 2^{-f} z_j$, where $z_j \in \{z \in \mathbb Z : -2^{p-1} \leq z < 2^{p-1}\}$.
Thus, this sum lies in $2^{-f}\{-L2^{p-1},\ldots,L(2^{p-1}-1)\}$, so each entry takes at most
\begin{equation}
    L(2^p-1)+1
    \leq
    2^{m+p}
\end{equation}
possible values.
Since $\rho(\bm{v})$ has $V$ entries, it takes at most
\begin{equation}
    2^{V(p+m)}
\end{equation}
possible values.

There are $2^L = 2^{2^m}$ possible value assignments.
Injectivity therefore requires
\begin{equation}
    2^{2^m}
    \leq
    2^{V(p+m)}.
\end{equation}
Taking logarithms gives $V \geq 2^m / (p+m)$.
If each layer has total value dimension at most $d$, then $V \leq Dd$, which gives the stated width bound.
\end{proof}

%% file: appendices/c-rasp-separation/sections/attention-upper-bound.tex
\subsection{Content-Dependent Attention Upper Bound}
\label{subsec:attention-upper-bound}

The previous results showed that \CRASP and uniform attention transformers require width scaling exponentially with the task parameter $m$.
At the heart of these lower bounds is the fact that \CRASP and uniform attention aggregate over the database in a way that cannot depend on the query, so the aggregate must retain enough information to answer all $L = 2^m$ possible queries.
In contrast, we will show in this section that content-dependent attention can dynamically route information in a query-dependent manner, yielding an exponentially more efficient representation. 

Specifically, we will show that a single attention head
solves the task with residual width $m+2$.

\begin{proposition}[Content-dependent attention upper bound]
\label{prop:attention-retrieval-upper-bound}
Fix $m \geq 1$ and let $L = 2^m$.
There is a one-layer, one-head causal softmax transformer with residual width $m+2$ that solves the query-dependent retrieval task for every value assignment $\bm{v} \in \{-1,+1\}^L$ and every query index $r \in [L]$.
The construction uses no feedforward layer, no positional encoding, and attention scores bounded in magnitude by $\tfrac{m}{2}\log(3m)$.
\end{proposition}

\begin{proof}
\textbf{Construction.} We begin by presenting the construction, before proving its correctness.
The residual stream has width $m+2$.
Write the residual state at position $i$ as $h_i \in \reals^{m+2}$, and initialize it by embedding the input token into the first $m+1$ coordinates:
\begin{equation}
    h_i = (x_i, 0),
    \qquad
    x_i = (k_i, v_i) \in \reals^{m+1}.
\end{equation}
Coordinates $1$ through $m$ hold the key, coordinate $m+1$ holds the value, and coordinate $m+2$ is initially zero and receives the attention output.

Let $\kappa(h_i) \in \reals^m$ denote coordinates $1$ through $m$ of $h_i$ and let $\nu(h_i) \in \reals$ denote coordinate $m+1$.
For a database position $j \leq L$ we have $\kappa(h_j) = k_j$ and $\nu(h_j) = v_j$, and at the query position we have $\kappa(h_{L+1}) = k_r$ and $\nu(h_{L+1}) = 0$.
Define the single attention head by
\begin{equation}
    W_Q h_i = \beta \, \kappa(h_i),
    \qquad
    W_K h_j = \kappa(h_j),
    \qquad
    W_V h_j = \nu(h_j),
    \qquad
    \beta = \tfrac{1}{2}\log(3m).
\end{equation}
Note that the key dimension is $m$ and the value dimension is $1$.
The output matrix $W_O$ is chosen so that the head writes its output into coordinate $m+2$.
The model outputs the sign of coordinate $m+2$ at the final position.

\textbf{Attention weights at the query position.}
The unnormalized score of position $j$ at the final position is
\begin{equation}
    s_j
    =
    \beta \, k_r^\intercal \kappa(h_j).
\end{equation}
Each key is a vector in $\{-1,+1\}^m$, so $k_r^\intercal k_j = m - 2 t_j$, where $t_j$ is the Hamming distance between $k_r$ and $k_j$.
The database contains each key exactly once, so exactly $\binom{m}{t}$ of the database positions are at Hamming distance $t$ from $k_r$.
The query position has $\kappa(h_{L+1}) = k_r$ and therefore $s_{L+1} = \beta m$.
Writing $\epsilon = e^{-2\beta}$, the softmax normalizer is
\begin{equation}
    Z
    =
    \sum_{j=1}^{L+1}
    e^{s_j}
    =
    e^{\beta m}
    +
    \sum_{t=0}^{m}
    \binom{m}{t}
    e^{\beta(m - 2t)}
    =
    e^{\beta m}
    \left(
        1 + A
    \right),
    \qquad
    A = (1+\epsilon)^m .
\end{equation}
Here, we used the binomial theorem with $x = e^{-2 \beta}$: 
$\sum_{t=0}^{m}
    \binom{m}{t}
    e^{\beta(m - 2t)} 
= e^{\beta m} \sum_{t=0}^{m}
    \binom{m}{t}
    e^{- 2 \beta t}
= e^{\beta m} (1 + e^{-2 \beta})^m
$.
The attention weights are therefore
\begin{equation}
    \alpha_r = \frac{1}{1+A},
    \qquad
    \alpha_{L+1} = \frac{1}{1+A},
    \qquad
    \sum_{j \leq L,\; j \neq r} \alpha_j = \frac{A-1}{1+A}.
\end{equation}
They are large at the query position itself and at the matched key position, and small elsewhere.
The choice of $\beta$ controls $A$, and therefore the sharpness of the attention weights.

\textbf{Output.}
The head output at the final position is $o = \sum_{j=1}^{L+1} \alpha_j \nu(h_j)$. We want to lower bound $o v_r$ away from zero to guarantee that $o$ has the right sign. We compute $o \cdot v_r$ directly,
\begin{equation}
    \begin{aligned}
        o \cdot v_r &= \sum_{j=1}^{L+1} \alpha_j \nu(h_j) v_r = \alpha_r v_r^2 + \sum_{j \leq L, j \neq r} \alpha_j v_j v_r \\
        &\geq \alpha_r - \sum_{j \leq L, j \neq r} \alpha_j = \frac{1}{1+A} - \frac{A-1}{1+A}
    \end{aligned}
\end{equation}
Here, the second equality uses the fact that $\nu(h_{L+1}) = 0$, and the inequality uses the fact that $v_j v_r \geq -1$ and that $v_r^2 = (\pm 1)^2 = 1$.
This implies that $o \cdot v_r \geq (2 - A) / (1+A)$.
Next, we choose $\beta$ such that this is lower bounded away from zero.
The choice $\beta = \tfrac{1}{2}\log(3m)$ gives $\epsilon = 1/(3m)$, hence
\begin{equation}
    A
    =
    \left(1 + \frac{1}{3m}\right)^{m}
    \leq
    e^{1/3}
    <
    \frac{3}{2}.
\end{equation}
Therefore $v_r \, o \geq (2 - e^{1/3})/(1 + e^{1/3}) > 1/4$.
In particular $o$ is nonzero and its sign equals $v_r$, so the model outputs $v_r$ as required.
Finally, every score satisfies $|s_j| \leq \beta m = \tfrac{m}{2}\log(3m)$.
\end{proof}

\textbf{Remark (finite precision).}
The construction also works with fixed-point parameters and activations, with \(p=O(\log m)\).
Take \(\beta\) to be the smallest element of \(\mathbb{F}_p\) that is at least \(\tfrac{1}{2}\log(3m)\).
Since $A=(1+e^{-2\beta})^m$ decreases in $\beta$, the bound $v_r\,o>1/4$ still holds.
The parameters and activations $0$, $\pm1$, and $\pm\beta$ are then exactly representable.
As in \Cref{prop:uniform-attention-retrieval-lower-bound}, the head output is computed exactly and then rounded into \(\mathbb{F}_p\); with round-to-nearest, this changes \(o\) by at most \(2^{-f}\leq1/8\) when \(f\geq3\), which does not change its sign.
This shows that \(p=O(\log m)\) bits suffice.

Proposition~\ref{prop:attention-retrieval-upper-bound} completes the proof of Theorem~\ref{thm:c-rasp-retrieval-separation}.
Together with Proposition~\ref{prop:c-rasp-retrieval-bounds} and Proposition~\ref{prop:uniform-attention-retrieval-lower-bound}, it establishes the separation.
One content-dependent attention head with width $m+2$ solves the task, whereas \CRASP requires $\Omega(2^m/m)$ prefix-count nodes and uniform attention requires width $\Omega(2^m / [D(p+m)])$.
The gap is exponential in $m$ and persists for any depth and precision that are polynomial in $m$.

%% file: appendices/c-rasp-separation/sections/experiment.tex
\subsection{Experiments}
\label{subsec:c-rasp-separation-experiment}

In this section, we empirically test the theoretical separation between softmax attention and uniform attention transformers on the query-dependent retrieval task.
\Cref{thm:c-rasp-retrieval-separation} predicts that softmax attention solves the task with small width, while uniform attention needs width that grows exponentially in $m$.

\paragraph{Experimental setup.}
We use key dimensions $m \in \{2, 3, 4, 5, 6\}$, so the database has $L = 2^m \in \{4, \ldots, 64\}$ items and the input has $L + 1$ tokens.
We sample the values $v_1, \ldots, v_L$ independently and uniformly from $\{-1, +1\}$, and the query index $r$ uniformly from $[L]$.
Chance accuracy is therefore $50\%$.
We shuffle the database items with a uniformly random permutation, drawn independently for each input.
A learned linear layer maps each input vector $x(k, v)$ or $q(k)$ into the $d$-dimensional residual stream.

All models use pre-LayerNorm blocks with a GELU MLP of width $4d$, no positional encoding (NoPE), and a learned BOS token.
Softmax attention models learn the query and key projections $\WQ, \WK$, while uniform-attention models fix $\WQ = \WK = 0$, so each head averages uniformly over all earlier positions.
A linear readout at the query position produces two output logits.
We train two model sizes.
The ``minimal'' model has one layer, one head, and width $d = m + 2$ based on the construction in \cref{prop:attention-retrieval-upper-bound}.
Unlike the construction, it also has an MLP, LayerNorm, and a BOS token, to match real transformers more closely.
The ``standard'' model has four layers, four heads, and width $d = 256$.

We train with cross-entropy loss for $2{,}048$ steps of AdamW, with batch size $256$, a constant learning rate of $10^{-3}$, no weight decay, gradient clipping at $1.0$, and bf16 mixed precision.
We sample a fresh batch at every step, for $524{,}288$ training inputs in total.
We report accuracy after the final step on a separate sample of $2{,}048$ inputs.
At small $m$ the input space is small, so these inputs also appear in training.
We run 5 seeds for each combination of $m$, attention type, and model size, for $100$ runs in total.

\begin{figure}[h]
    \centering
    \includegraphics[width=\linewidth]{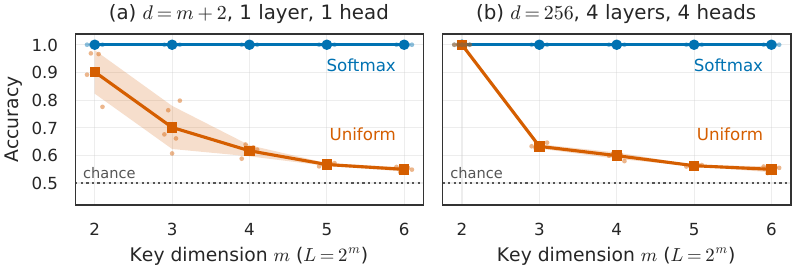}
    \caption{\textbf{Softmax versus uniform attention on query-dependent retrieval.}
    Accuracy against the key dimension $m$, for (a) width $d = m + 2$ with one layer and one head and (b) width 256 with four layers and four heads.
    Small markers are individual seeds.}
    \label{fig:c-rasp-separation-appendix}
\end{figure}

\paragraph{Results.}
We find that softmax attention models solve the task at every $m$ and at both model sizes, while uniform attention models quickly degrade as $m$ grows.
With uniform attention, the standard model solves the task at $m = 2$ and the minimal model reaches $90\%$, but both fall to about $55\%$ at $m = 6$, close to chance.
By contrast, softmax attention models reach $100\%$ accuracy in every seed, at every $m$, and at both model sizes.
These results agree with the predictions of \Cref{thm:c-rasp-retrieval-separation}.
The failure of uniform attention transformers may reflect not only the representational limitations of uniform attention, but also the difficulty of learning to emulate content-dependent retrieval with uniform attention.